\documentclass{article}[12]
\usepackage{mathrsfs}
\usepackage{enumitem}
\usepackage{bbm}
\usepackage{tikz}
\usepackage{float}
\usetikzlibrary{backgrounds}
\usepackage{caption}
\usepackage{color}
\usepackage{url}
\usepackage{graphicx}
\usepackage{latexsym}
\usepackage{amsfonts}
\usepackage{pifont,xspace,fullpage,epsfig, wrapfig}
\usepackage{amsmath, amsthm, amssymb}
\usepackage{multirow}
\usepackage{dsfont}
\DeclareSymbolFont{AMSb}{U}{msb}{m}{n}
\DeclareMathSymbol{\N}{\mathbin}{AMSb}{"4E}
\DeclareMathSymbol{\Z}{\mathbin}{AMSb}{"5A}
\DeclareMathSymbol{\R}{\mathbin}{AMSb}{"52}
\DeclareMathSymbol{\Q}{\mathbin}{AMSb}{"51}
\DeclareMathSymbol{\erert}{\mathbin}{AMSb}{"50}
\DeclareMathSymbol{\I}{\mathbin}{AMSb}{"49}
\DeclareMathSymbol{\C}{\mathbin}{AMSb}{"43}
\newcommand{\sumnl}{\sum\nolimits}

\newcommand{\mynote}[2]{{\textcolor{#1}{ #2}}}

\definecolor{gray}{gray}{0.4}
\newcommand{\gray}[1]{\mynote{gray}{{\footnotesize #1}}}

\newcommand{\remove}[1]{}

\newcommand{\AAA}{\mathcal A}

\newcommand{\DDD}{\mathcal D}
\newcommand{\FFF}{\mathcal F}

\newcommand{\XXX}{\mathcal X}
\newcommand{\eps}{\epsilon}

\newcommand{\error}{{\rm error}}
\newcommand{\db}{S}
\newcommand{\calls}{\operatorname{\rm calls}}
\newcommand{\est}{\operatorname{\rm  Est}}

\newcommand{\Lap}{\operatorname{\rm Lap}}
\newcommand{\VC}{\operatorname{\rm VC}}

\newcommand{\Nlog}[1]{\log^{\lceil #1 \rceil}}
\newcommand{\RepDim}{\operatorname{\rm RepDim}}

\newcommand{\poly}{\mathop{\rm poly}}

\newcommand{\point}{\operatorname{\tt POINT}}

\newcommand{\kpoint}{\operatorname{\tt k-POINT}}
\newcommand{\jpoint}[1]{\operatorname{\tt #1-POINT}}

\newcommand{\rectangle}{\operatorname*{\tt RECTANGLE}}
\newcommand{\thresh}{\operatorname*{\tt THRESH}}

\def\E{\operatorname*{\mathbb{E}}}

\def\Q{\operatorname*{\mathbb{Q}}}
\def\poly{\mathop{\rm{poly}}\nolimits}
\def\Lap{\mathop{\rm{Lap}}\nolimits}
\def\opt{\mathop{\rm{opt}}\nolimits}

\newtheorem{theorem}{Theorem}[section]
\newtheorem{lemma}[theorem]{Lemma}
\newtheorem{definition}[theorem]{Definition}
\newtheorem{remark}[theorem]{Remark}
\newtheorem{proposition}[theorem]{Proposition}
\newtheorem{claim}[theorem]{Claim}

\newtheorem{observation}[theorem]{Observation}

\theoremstyle{remark}

\theoremstyle{remark}
\newtheorem{example}[theorem]{Example}

\newcommand{\defref}[1]{Definition~\ref{def:#1}}

\begin{document}

\begin{titlepage}
\title{Private Learning and Sanitization: Pure vs. Approximate Differential Privacy\footnote{An extended abstract of this paper appeared in the Proceedings of the 17th International Workshop on Randomization and Computation, 2013~\cite{BNS13b}.}}

\author{Amos Beimel\thanks{Supported by a grant from the Israeli Science and Technology ministry and by an Israel Science Foundation grant 938/09. } 
\and 
Kobbi Nissim\thanks{Work done while the second author was a Visiting Scholar at the Harvard Center for Research on Computation and Society, supported by NSF grant CNS-1237235. Partially supported by an Israel Science Foundation grant  2761/12.}
\and 
Uri Stemmer\thanks{Supported by the Ministry of Science and Technology (Israel), by the Check Point Institute for Information Security, and by the IBM PhD Fellowship Awards Program. Work done in part while the third author was visiting Harvard University supported by NSF grant CNS-1237235 and a gift from Google, Inc.}}

\date{\today}

\maketitle

\setcounter{page}{0} \thispagestyle{empty}

\begin{abstract}
We compare the sample complexity of private learning~[Kasiviswanathan et al. 2008] and sanitization~[Blum et al. 2008] under {\em pure} $\epsilon$-differential privacy~[Dwork et al. TCC 2006] and {\em approximate} $(\epsilon,\delta)$-differential privacy~[Dwork et al. Eurocrypt 2006]. We show that the sample complexity of these tasks under approximate differential privacy can be significantly lower than that under pure differential privacy.

We define a family of optimization problems, which we call {\em Quasi-Concave Promise Problems}, that generalizes some of our considered tasks. We observe that a quasi-concave promise problem can be privately approximated using a solution to a smaller instance of a quasi-concave promise problem. This allows us to construct an efficient recursive algorithm solving such problems privately.
Specifically, we construct private learners for point functions, threshold functions, and axis-aligned rectangles in high dimension. Similarly, we construct sanitizers for point functions and threshold functions.

We also examine the sample complexity of {\em label-private} learners, a relaxation of private learning where the learner is required to only protect the privacy of the labels in the sample. We show that the VC dimension completely characterizes the sample complexity of such learners, that is, the sample complexity of learning with label privacy is equal (up to constants) to learning without privacy.
\end{abstract}

\end{titlepage}

\tableofcontents
\setcounter{page}{0} \thispagestyle{empty}
\newpage

\section{Introduction}

Learning is often applied to collections of sensitive data of individuals and it is important to protect the privacy of these individuals.
We examine the sample complexity of private learning~\cite{KLNRS08} and a related task -- sanitization~\cite{BLR08full} -- while preserving differential privacy~\cite{DMNS06}. We show striking differences between the required sample complexity for these tasks under $\epsilon$-differential privacy~\cite{DMNS06} (also called {\em pure} differential privacy) and its variant $(\epsilon,\delta)$-differential privacy~\cite{DKMMN06} (also called {\em approximate} differential privacy).

\paragraph{Differential privacy.} Differential privacy protects the privacy of individuals by requiring that the information of an individual does not significantly affect the output. More formally, an algorithm $A$ satisfies the requirement of {\em Pure Differential Privacy} if for every two databases that differ on exactly one entry, and for every event defined over the output set of $A$, the probability of this event is close up to a multiplicative factor of $e^\epsilon \approx 1+\epsilon$ whether $A$ is applied on one database or on the other. {\em Approximate Differential Privacy} is a relaxation of pure differential privacy where the above guarantee needs to be satisfied only for events whose probability is at least $\approx\delta$. We show that even negligible $\delta>0$ can have a significant effect on sample complexity of private learning and sanitization.

\paragraph{Private Learning.} 

Private learning was introduced in~\cite{KLNRS08} as a combination of Valiant's PAC learning model~\cite{Valiant84} and differential privacy.
For now, we can think of a private learner as a differentially private algorithm that operates on a set of classified random examples, and outputs a hypothesis that misclassifies fresh examples with probability at most (say) $\frac{1}{10}$.
The work on private learning has mainly focused on pure privacy. On the one hand, Blum et al.~\cite{BDMN05} and Kasiviswanathan et al.~\cite{KLNRS08} have showed, via generic constructions, that every finite concept class $C$ can be learned privately, using sample complexity proportional to $\poly(\log|C|)$ (often efficiently). On the other hand, a significant difference was shown between the sample complexity of {\em traditional} (non-private) learners (crystallized in terms of $\VC(C)$ and smaller than $\log|C|$ in many interesting cases) and private learners.
As an example, let $\point_d$ be the class of point functions over the domain $\{0,1\}^d$ (these are the functions that evaluate to one on exactly one point of the domain and to zero elsewhere). Consider the task of {\em properly} learning $\point_d$ where, after consulting its sample, the learner outputs a hypothesis that is by itself in $\point_d$. Non-privately, learning $\point_d$ requires merely a constant number of examples (as $\VC(\point_d)=1$). Privately, $\Omega(d)$ examples are required~\cite{BBKN12}. Curiously, the picture changes when the private learner is allowed to output a hypothesis not in $\point_d$ (such learners are called {\em improper}), as the sample complexity can be reduced to $O(1)$~\cite{BBKN12}. This, however, comes with a price, as it was shown in~\cite{BBKN12} that such learners must return hypotheses that evaluate to one on exponentially many points in $\{0,1\}^d$ and, hence, are very far from all functions in $\point_d$.

A complete characterization for the sample complexity of pure-private learners was recently given in~\cite{BNS13}, in terms of a new dimension -- the {\em Representation Dimension}, that is, given a class $C$, the number of samples needed and sufficient for privately learning $C$ is $\Theta(\RepDim(C))$. Following that, Feldman and Xiao~\cite{FX14} showed an equivalence between the representation dimension of a concept $C$ and the randomized one-way communication complexity of the evaluation problem for concepts from $C$. Using this equivalence they separated the sample complexity of pure-private learners from that of non-private ones. For example, they showed a lower bound of $\Omega(d)$ on the sample complexity of every pure-private (proper or improper) learner for the class $\thresh_d$ of threshold functions over the interval $[0,2^d-1]$. This is a strong separation from the non-private sample complexity, which is $O(1)$ (as the VC dimension of this class is constant).

We show that the sample complexity of proper learning with {\em approximate} differential privacy can be significantly lower than that satisfying {\em pure} differential privacy. Our starting point for this work is an observation that with {\em approximate} $(\epsilon,\delta)$-differential privacy, sample complexity of $O(\log (1/\delta))$ suffices for learning points {\em properly}. This gives a separation between pure and approximate proper private learning for $\delta=2^{-o(d)}$. 
\remove{Anecdotally, combining this with a result from~\cite{BBKN12} gives a learning task that is not computationally feasible under pure differential privacy and becomes polynomial time computable under approximate differential privacy.}

\paragraph{Sanitization.} 

The notion of differentially private sanitization was introduced in the work of Blum et al.~\cite{BLR08full}. A sanitizer for a class of predicates $C$ is a differentially private mechanism translating an input database $S$ to an output database $\hat{S}$ such that $\hat{S}$ (approximately) agrees with $S$ on the fraction of the entries satisfying $\varphi$ for all $\varphi\in C$, where every predicate $\varphi\in C$ is a function from $X$ to $\{0,1\}$. Blum et al.\ gave a generic construction of pure differentially private sanitizers exhibiting sample complexity $O(\VC(C)\log|X|)$. Lower bounds partially supporting this sample complexity were given by~\cite{Roth10,BBKN12,Moritz}. As with private learning, we show significant differences between the sample complexity required for sanitization of simple predicate classes under pure and approximate differential privacy. We note that the construction of sanitizers is not generally computationally feasible~\cite{DNRRV09,UV11,Ullman12}.

\subsection{Our Contributions}

To simplify the exposition, we omit in this section dependency on all variables except for $d$, corresponding to the representation length of domain elements.

\paragraph{Tools.}

A recent instantiation of the Propose-Test-Release (PTR) framework~\cite{DworkLei} by Smith and Thakurta~\cite{ADist} results, almost immediately, with a proper learner for points, exhibiting $O(1)$ sample complexity while preserving approximate differential privacy.
This simple technique does not suffice for our other constructions of learners and sanitizers, and we, hence, introduce new tools for coping with proper private learning of thresholds and axis-aligned rectangles, and sanitization for point functions and thresholds:

\begin{itemize}
\item {\bf Choosing mechanism:} Given a {\em low-sensitivity} quality function, one can use the exponential mechanism~\cite{MT07} to choose an approximately maximizing solution. This requires, in general, a database of size logarithmic in the number of possible solutions. We identify a sub family of low-sensitivity functions, called {\em bounded-growth} functions, for which it is possible to significantly reduce the necessary database size when using the exponential mechanism.

\item {\bf Recursive algorithm for quasi-concave promise problems:} We define a family of optimization problems, which we call {\em Quasi-Concave Promise Problems}.  The possible solutions are ordered, and {\em quasi-concavity} means that if two solutions $f\leq h$ have quality of at least $\XXX$, then any solution $f\leq g\leq h$ also has quality of at least $\XXX$.
The optimization goal is, when there exists a solution with a promised quality of (at least) $r$, to find a solution with quality $\approx r$.
We observe that a quasi-concave promise problem can be privately approximated using a solution to a smaller instance of a quasi-concave promise problem. This allows us to construct an efficient recursive algorithm solving such problems privately. We show that the task of learning $\thresh_d$ is, in fact, a quasi-concave promise problem, and it can be privately solved using our algorithm with sample size roughly $2^{O(\log^* d)}$. Sanitization for $\thresh_d$ does not exactly fit the model of quasi-concave promise problems but can still be solved by iteratively defining and solving a small number of quasi-concave promise problems.
\end{itemize}

\paragraph{Implications for Private Learning and Sanitization.}

We give new private {\em proper}-learning algorithms for the classes $\point_d$ and $\thresh_d$. We also construct a new private proper-learner for (a discrete version of) the class of all axis-aligned rectangles over $n$ dimensions.
Our algorithms exhibit sample complexity that is significantly lower than bounds given in prior work, separating pure and approximate private learning. Similarly, we construct sanitizers for $\point_d$ and $\thresh_d$, again with sample complexity that is significantly lower than bounds given in prior work, separating sanitization in the pure and approximate privacy cases.
Our algorithms are time-efficient.

\paragraph{Sanitization vs.\ Private Learning.}

Gupta et al.~\cite{GHRU11} have given reductions in both directions between agnostic learning of a concept class $C$,
and the sanitization task for the same class $C$. The learners and sanitizers they consider are limited to access their data via statistical queries~\cite{Kearns98} (such algorithm can be easily transformed to satisfy differential privacy~\cite{BDMN05}).
In Section~\ref{sec:reduction} we show a similar reduction from the task of privately learning a concept class $C$ to the sanitization task of $C$, where the sanitizer's access to the database is unrestricted. This allows us to exploit lower bounds on the sample complexity of private learners and show an explicit class of predicates $C$ over a domain $X$ for which every private sanitizer requires databases of size $\Omega(\VC(C)\log|X|)$. A similar lower bound was shown by Hardt and Rothblum~\cite{Moritz}, achieving tighter results in terms of the approximation parameter. Their work proves the existence of such a concept class, but does not give an explicit one.

\paragraph{Label Privacy.}

In Section~\ref{sec:semiPrivate} we examine private learning under a relaxation of differential privacy called {\em label privacy} (see~\cite{CH11} and references therein), where the learner is required to only protect the privacy of the labels in the sample. Chaudhuri et al.~\cite{CH11} have proved lower bounds for label-private learners in terms of the doubling dimension of the target concept class. We show that the VC dimension completely characterizes the sample complexity of such learners, that is, the sample complexity of learning with label privacy is equal (up to constants) to learning without privacy.

\subsection{Open Questions} 
This work raises two kinds of research directions.
First, this work presents (time and sample efficient) private learners and sanitizers for relatively simple concept classes. It would be natural to try and construct private learners and sanitizers for more complex concept classes. In particular, constructing a (time and sample efficient) private learner for hyperplanes would be very interesting to the community.

Another very interesting research direction is to try and understand the sample complexity of approximate-private learners. Currently, no lower bounds are known on the sample complexity of such learners. On the other hand, 
no generic construction for such learners is known to improve the sample complexity achieved by the generic construction of Kasiviswanathan et al.~\cite{KLNRS08} for pure-private learners.
Characterizing the sample complexity of approximate-private learners is a very interesting open question.

\subsection{Other Related Work}

Most related to our work is the work on private learning and its sample complexity~\cite{KLNRS08,BDMN05,BBKN12,CH11} and the early work on sanitization~\cite{BLR08full} mentioned above. Another related work is the work of De~\cite{De12}, who proved a separation between {\em pure} $\epsilon$-differential privacy and {\em approximate} $(\epsilon,\delta)$-differential privacy. Specifically, he demonstrated that there exists a query where it is sufficient to add noise $O(\sqrt{n\log(1/\delta)})$ when $\delta>0$ and $\Omega(n)$ noise is required when $\delta=0$. Earlier work by Hardt and Talwar~\cite{HT10} separated pure from approximate differential privacy for $\delta=n^{-O(1)}$. 

Another interesting gap between pure and approximate differential privacy is the following.
Blum et al.~\cite{BLR08full} have given a generic construction of pure-private sanitizers, in which the sample complexity grows as $\frac{1}{\alpha^3}$ (where $\alpha$ is the approximation parameter).
Following that, Hardt and Rothblum~\cite{PMW_HR10} showed that with approximate-privacy, the sample complexity can be reduce to grow as $\frac{1}{\alpha^2}$. Currently, it is unknown whether this gap is essential.

\remove{ 
De~\cite{De12} showed that pure and approximate differential privacy are very different in terms of the mutual information between the outcome of the mechanism and database even when $\delta=2^{-o(n)}$ is negligible in the size $n$ of the database. 
}
\section{Preliminaries}

\paragraph{Notations.} We use $O_{\gamma}(f(t))$ as a shorthand for $O(h(\gamma) \cdot f(t))$ for some non-negative function $h$. In informal discussions, we sometimes use $\widetilde{O}(f(t))$ instead of $O(f(t)\cdot {\rm polylog}(f(t)))$. For example, $2^{\log^*(d)}\cdot\log^*(d)=\widetilde{O}\left(2^{\log^*(d)}\right)$.

We use $X$ to denote an arbitrary domain, and $X_d$ for the domain $\{0,1\}^d$.
We use $X^m$ (and respectively $X_d^m$) for the cartesian $m^{\text th}$ power of $X$, i.e., $X^m=(X)^m$, and use $X^*=\bigcup_{m=0}^\infty{X^m}$.

Given a distribution $\DDD$ over a domain $X$, we denote $\DDD(j) \triangleq \Pr_{x\sim\DDD}[x=j]$ for $j\in X$, and $\DDD(J) \triangleq \Pr_{x\sim\DDD}[x\in J]$ for $J \subseteq X$.

\subsection{Differential Privacy} 

Differential privacy aims at protecting information of individuals. We consider a database, where each entry contains information pertaining to an individual. An algorithm operating on databases is said to preserve differential privacy if a change of a single record of the database does not significantly change the output distribution of the algorithm. Intuitively, this means that whatever is learned about an individual could also be learned with her data arbitrarily modified (or without her data). Formally:

\begin{definition}
Databases $S_1\in X^m$ and $S_2\in X^m$ over a domain $X$ are called {\em neighboring} if they differ in exactly one entry.
\end{definition}

\begin{definition}[Differential Privacy~\cite{DMNS06,DKMMN06}] \label{def:dp} 
A randomized algorithm $A$ is $(\epsilon,\delta)$-differentially private if for all neighboring databases $\db_1,\db_2\in X^m$, and for all sets $\mathcal{F}$ of outputs,
\begin{eqnarray}
\label{eqn:diffPrivDef}
  & \Pr[A(\db_1) \in \mathcal{F}] \leq \exp(\epsilon) \cdot \Pr[A(\db_2) \in \mathcal{F}] + \delta.  &
\end{eqnarray}
The probability is taken over the random coins of $A$. 
When $\delta=0$ we omit it and say that $A$ preserves $\epsilon$-differential privacy.
\end{definition}
We use the term {\em pure} differential privacy when $\delta=0$ and the term {\em approximate} differential privacy when $\delta>0$, in which case $\delta$ is typically a negligible function of the database size $m$.

We will later present algorithms that access their input database using (several) differentially private mechanisms. We will use the following composition theorems. 

\begin{theorem}[\cite{DKMMN06}]\label{thm:composition1}
If $A_1$ and $A_2$ satisfy $(\epsilon_1,\delta_1)$ and $(\epsilon_2,\delta_2)$ differential privacy, respectively, then their concatenation $A(S)=\langle A_1(S),A_2(S) \rangle$ satisfies $(\epsilon_1+\epsilon_2,\delta_1+\delta_2)$-differential privacy.
\end{theorem}

Moreover, a similar theorem holds for the adaptive case, where a mechanism interacts with $k$ {\em adaptively chosen} differentially private mechanisms.

\begin{theorem}[\cite{DKMMN06, DworkLei}]\label{thm:composition3}
A mechanism that permits $k$ adaptive interactions with mechanisms that preserves $(\epsilon,\delta)$-differential privacy (and does not access the database otherwise) ensures $(k\epsilon, k\delta)$-differential privacy.
\end{theorem}

Note that the privacy guaranties of the above bound deteriorates linearly with the number of interactions. By bounding the {\em expected} privacy loss in each interaction (as opposed to worst-case), Dwork et al.~\cite{DRV10} showed the following stronger composition theorem, where privacy deteriorates (roughly) as $\sqrt{k}\epsilon+k\epsilon^2$ (rather than $k\epsilon$).

\begin{theorem}[\cite{DRV10}, restated]\label{thm:composition2}
Let $0<\epsilon,\delta'\leq1$, and let $\delta\in[0,1]$. A mechanism that permits $k$ adaptive interactions with mechanisms that preserves $(\epsilon,\delta)$-differential privacy (and does not access the database otherwise) ensures $(\epsilon', k\delta+\delta')$-differential privacy, for $\epsilon'=\sqrt{2k\ln(1/\delta')}\cdot\epsilon+2k\epsilon^2$.
\end{theorem}

\subsection{Preliminaries from Learning Theory}  

\subsubsection{The PAC Model}

A concept $c:X\rightarrow \{0,1\}$ is a predicate that labels {\em examples} taken from the domain $X$ by either 0 or 1.  A \emph{concept class} $C$ over $X$ is a set of concepts (predicates) mapping $X$ to $\{0,1\}$. A learning algorithm is given examples sampled according to an unknown probability distribution $\DDD$ over $X$, and labeled according to an unknown {\em target} concept $c\in C$. The learning algorithm is successful when it outputs a hypothesis $h$ that approximates the target concept over samples from $\DDD$. More formally:

\begin{definition}
The {\em generalization error} of a hypothesis $h:X\rightarrow\{0,1\}$ is defined as 
$$\error_{\DDD}(c,h)=\Pr_{x \sim \DDD}[h(x)\neq c(x)].$$ 
If $\error_{\DDD}(c,h)\leq\alpha$ we say that $h$ is {\em $\alpha$-good} for $c$ and $\DDD$.
\end{definition}

\begin{definition}[PAC Learning~\cite{Valiant84}]\label{def:PAC}
Algorithm $A$ is an {\em $(\alpha,\beta,m)$-PAC learner} for a concept
class $C$ over $X$ using hypothesis class $H$ if for all 
concepts $c \in C$, all distributions $\DDD$ on $X$,
given an input of $m$ samples $\db =(z_1,\ldots,z_m)$, where $z_i=(x_i,c(x_i))$ and each $x_i$
is drawn i.i.d.\ from $\DDD$, algorithm $A$ outputs a
hypothesis $h\in H$ satisfying
$$\Pr[\error_{\DDD}(c,h)  \leq \alpha] \geq 1-\beta.$$
The probability is taken over the random choice of
the examples in $\db$ according to $\DDD$ and the coin tosses of the learner $A$.
If $H\subseteq C$ then $A$ is called a {\em proper} PAC learner; otherwise, it is called an {\em improper} PAC learner.
\end{definition}

\begin{definition}
For a labeled sample $\db=(x_i,y_i)_{i=1}^m$, the {\em empirical error} of $h$ is
$$\error_S(h) = \frac{1}{m} |\{i : h(x_i) \neq y_i\}|.$$
\end{definition}

\subsubsection{The Vapnik-Chervonenkis Dimension}\label{sec:VC}
The Vapnik-Chervonenkis (VC) Dimension is a combinatorial measure of concept classes, which characterizes the sample size of PAC learners.

\begin{definition}[\cite{VC}] Let $C$ be a concept class over a domain $X$, and let $B=\{b_1,\ldots,b_\ell\}\subseteq X$. 
The set of all dichotomies (behaviors) on $B$ that are realized by $C$ is 
$$\Pi_C(B)=\Big\{(c(b_1),\ldots,c(b_\ell)):c\in C\Big\}.$$
\end{definition}

Observe that $\Pi_C(B)$ is a subset of $\{0,1\}^\ell$ (as $c\in C$ maps into $\{0,1\}$). The set of dichotomies $\Pi_C(B)$ can be viewed as the ``projection'' of $C$ on $B$.

\begin{definition}[\cite{VC}]
A set $B\subseteq X$ is {\em shattered} by $C$ if $\;\Pi_C(B)=\{0,1\}^\ell$ (where $\ell=|B|$). 
\end{definition}

That is, $B$ is shattered by $C$ if $C$ realizes all possible dichotomies over $B$.

\begin{definition}[VC-Dimension~\cite{VC}]
The {\em VC-Dimension} of a concept class $C$ (over a domain $X$), denoted as $\VC(C)$, is the cardinality of the largest set $B\subseteq X$ shattered by $C$. If arbitrarily large finite sets can be shattered by $C$, then $\VC(C)=\infty$.
\end{definition}

Observe that as $\Pi_C(B) \leq |C|$ a set $B$ can be shattered only if $|B|\leq \log |C|$ and hence $\VC(C)\leq \log |C|$.

\subsubsection{VC Bounds}
Classical results in computational learning theory state that a sample of size $\theta(\VC(C))$ is both necessary and sufficient for the PAC learning of a concept class $C$.
The following two theorems give upper and lower bounds on the sample complexity.

\begin{theorem}[\cite{EHKV}]
Any algorithm for PAC learning a concept class $C$ must have sample complexity $\Omega(\frac{\VC(C)}{\alpha})$, where $\alpha$ is the approximation parameter.
\end{theorem}

\begin{theorem}[VC-Dimension Generalization Bound~\cite{VC,BEHW}]\label{thm:VCconsistant}
Let $C$ and $\DDD$ be a concept class and a distribution over a domain $X$.
Let $\alpha,\beta>0$, and $m\geq\frac{8}{\alpha}(\VC(C)\ln(\frac{16}{\alpha})+\ln(\frac{2}{\beta}))$.
Fix a concept $c\in C$, and suppose that we draw a sample $S=(x_i,y_i)_{i=1}^m$, where $x_i$ are drawn i.i.d.\ from $\DDD$ and $y_i=c(x_i)$. Then,
$$
\Pr[\exists h\in C \text{ s.t. } \error_{\DDD}(h,c)>\alpha \; \wedge \; \error_S(h)=0]\leq\beta.
$$
\end{theorem}

So, for any concept class $C$, any algorithm that takes a sample of $m=\Omega_{\alpha,\beta}(\VC(C))$ labeled examples and produces as output a concept $h\in C$ that agrees with the sample is a PAC learner for $C$.
Such an algorithm is a PAC learner for $C$ using $C$ (that is, both the target concept and the returned hypotheses are taken from the same concept class $C$), and, therefore, there always exist a hypothesis $h\in C$ with $\error_S(h)=0$ (e.g., the target concept itself).

The next theorem handles (in particular) the agnostic case, in which a learning algorithm for a concept class $C$ is using a hypotheses class $H\neq C$, and given a sample $S$ (labeled by some $c\in C$), a hypothesis $h$ with $\error_S(h)=0$ might not exist in $H$.

\begin{theorem}[VC-Dimension Agnostic Generalization Bound \cite{Anthony2009,Anthony93}]\label{thm:generalization}
Let $\DDD$ and $H$ be a distribution and a concept class over a domain $X$,
and let $f:X\rightarrow\{0,1\}$ be some concept, not necessarily in $H$.
For a sample $S=(x_i,f(x_i))_{i=1}^m$ where $m\geq\frac{50 \VC(H)}{\alpha^2}\ln(\frac{1}{\alpha\beta})$
and $\{x_i\}$ are drawn i.i.d. from $\DDD$, it holds that
$$\Pr\Big[\forall \; h\in H:\;\; \big|\error_\DDD(h,f)-\error_S(h)\big|\leq\alpha\Big]\geq1-\beta.$$  
\end{theorem}

Notice that in the agnostic case the sample complexity is proportional to $\frac{1}{\alpha^2}$, as opposed to $\frac{1}{\alpha}$ when learning a class $C$ using $C$.

\subsection{Private Learning}\label{sec:PPAC}
In private learning, we would like to accomplish the same goal as in non-private learning, while protecting the privacy of the input database.
\begin{definition}[Private PAC Learning~\cite{KLNRS08}]
Let $A$ be an algorithm that gets an input $\db =\{z_1,\ldots,z_m\}$. Algorithm $A$ is an {\em $(\alpha,\beta,\epsilon,\delta,m)$-PPAC learner} for a concept
class $C$ over $X$ using hypothesis class $H$ if
\begin{description}
\item{\sc Privacy.} Algorithm $A$ is $(\epsilon,\delta)$-differentially private (as in \defref{dp});
\item{\sc Utility.} Algorithm $A$ is an {\em $(\alpha,\beta,m)$-PAC learner} for $C$ using $H$ (as in \defref{PAC}).
\end{description}
When $\delta=0$ (pure privacy) we omit it from the list of parameters.
\end{definition}

Note that the utility requirement in the above definition is an average-case requirement, as the learner is only required to do well on typical samples (i.e., samples drawn i.i.d. from a distribution $\DDD$ and correctly labeled by a target concept $c\in C$). In contrast, the privacy requirement is a worst-case requirement, and Inequality~(\ref{eqn:diffPrivDef}) must hold for every pair of neighboring databases (no matter how they were generated, even if they are not consistent with any concept in $C$).

\subsection{Sanitization}
Given a database $S=(x_1,\ldots,x_m)$ containing elements from some domain $X$, the goal of {\em sanitization mechanisms} is to output (while preserving differential privacy) another database $\hat{S}$ that is in some sense similar to $S$. This returned database $\hat{S}$ is called a {\em sanitized} database.

Let $c:X\rightarrow\{0,1\}$ be a concept. The counting query
$Q_c:X^*\rightarrow[0,1]$ is 
$$Q_c(\db) = \frac{1}{|\db|}\cdot \Big|\{i \,:\,  c(x_i) =1\} \Big|.$$
That is, $Q_c(\db)$ is the fraction of the entries in $\db$ that satisfy the concept $c$. Given a database $S$, a sanitizer for a concept class $C$ is required to output a sanitized database $\hat{\db}$ s.t. $Q_c(\db) \approx Q_c(\hat{\db})$ for every $c\in C$.
For computational reasons, sanitizers are sometimes allowed not to return an actual database, but rather a data structure capable of approximating $Q_c(\db)$ for every $c\in C$.

\begin{definition} 
Let $C$ be a concept class and let $\db$ be a database. A function $\est:C\rightarrow [0,1]$ is called {\em $\alpha$-close} to $\db$ if $|Q_c(\db)-\est(c)|\leq\alpha$ for every $c\in C$.
If, furthermore, $\est$ is defined in terms of a database $\hat\db$, i.e., $\est(c)=Q_c(\hat\db)$, we say that $\hat\db$ is $\alpha$-close to $\db$.
\end{definition}

\begin{definition}[Sanitization \cite{BLR08full}]
Let $C$ be a class of concepts mapping $X$ to $\{0,1\}$. Let $A$ be an algorithm that on an input database $S\in X^*$ outputs a description of a function $\est:C\rightarrow [0,1]$. Algorithm $A$ is an $(\alpha,\beta,\epsilon,\delta,m)$-improper-sanitizer for predicates in the class $C$, if
\begin{enumerate}
\item $A$ is $(\epsilon,\delta)$-differentially private;
\item For every input $\db\in X^m$, it holds that
$ \Pr\limits_A\left[\mbox{\rm $\est$ is $\alpha$-close to $\db$}\right]\geq 1-\beta.$
\end{enumerate}
The probability is over the coin tosses of algorithm $A$.
If on an input database $\db$ algorithm $A$ outputs another database $\hat{S}\in X^*$, and $\est(\cdot)$ is defined as $\est(c)=Q_c(\hat{S})$, then algorithm $A$ is called a {\em proper-sanitizer} (or simply a {\em sanitizer}).
As before, when $\delta=0$ (pure privacy) we omit it from the set of parameters.
\end{definition}

\begin{remark}\label{improperProperSanitization}
Note that without the privacy requirements sanitization is a trivial task as it is possible to simply output the input database $\db$. Furthermore, ignoring computational complexity, an $(\alpha,\beta,\epsilon,\delta,m)$-improper-sanitizer can always be transformed into a $(2\alpha,\beta,\epsilon,\delta,m)$-sanitizer, by finding a database $\hat{\db}$ of $m$ entries that is $\alpha$-close to $\est$. Such a database must exist except with probability $\beta$ (as in particular $S$ is $\alpha$-close to $\est$), and is $2\alpha$-close to $S$ (by the triangle inequality).
\end{remark}

The following theorems state some of the known results on the sample complexity of pure-privacy sanitizers. We start with an upper bound on the necessary sample complexity.

\begin{theorem}[Blum et al.~\cite{BLR08full}]\label{thm:BlumUp}
There exists a constant $\Gamma$ such that
for any class of predicates $C$ over a domain $X$, and any parameters $\alpha,\beta,\epsilon$,
there exists an $(\alpha,\beta,\epsilon,m)$-sanitizer for $C$, provided that the size of the database, denoted $m$, is at least 
$$m\geq \Gamma\left(\frac{\log|X|\cdot \VC(C)\cdot\log(1/\alpha)}{\alpha^3\epsilon}+\frac{\log(1/\beta)}{\epsilon\alpha}\right).$$
The algorithm might not be efficient.
\end{theorem}


The above theorem states that, in principle, data sanitization is possible. The input database may be required to be as big as the representation size of elements in $X$.
The next theorem states a general lower bound (far from the above upper bound) on the sample complexity of any concept class $C$. Better bounds are known for specific concept classes~\cite{Moritz}.

\begin{theorem}[Blum et al.~\cite{BLR08full}]\label{thm:BlumLow}
Let $C$ be a class of predicates, and let $m\leq\frac{\VC(C)}{2}$.
For any $0<\beta<1$ bounded away from 1 by a constant,
for any $\epsilon\leq1$, if $A$ is an $(\alpha,\beta,\epsilon,m)$-sanitizer for $C$,
then $\alpha\geq\frac{1}{4+16\epsilon}$.
\end{theorem}

Recall that a proper sanitizer operates on an input database $S\in X^m$, and outputs a sanitized database $\hat{S}\in X^*$.
The following is a simple corollary of Theorem~\ref{thm:generalization}, stating that the size of $\hat{S}$ does not necessarily depend on the size of the input database $S$.

\begin{theorem}\label{thm:BlumShrink}
Let $C$ be a concept class. For any database $\db$ there exists a database $\hat{\db}$ of size $n=O(\frac{\VC(C)}{\alpha^2}\log(\frac{1}{\alpha}))$ such that $\max_{h\in C}|Q_h(\db)-Q_h(\hat{\db})|\leq\alpha$.
\end{theorem}

In particular, the above theorem implies that an $(\alpha,\beta,\epsilon,\delta,m)$-sanitizer  $A$ can always be transformed into a $(2\alpha,\beta,\epsilon,\delta,m)$-sanitizer $A'$ s.t. the sanitized databases returned by $A'$ are always of fixed size $n=O(\frac{\VC(C)}{\alpha^2}\log(\frac{1}{\alpha}))$.
This can be done by finding a database $\hat{S}$ of $n$ entries that is $\alpha$-close to the sanitized database returned by $A$.
Using the triangle inequality, $\hat{S}$ is (w.h.p.) $2\alpha$-close to the input database.

\subsection{Basic Differentially-Private Mechanisms}\label{sec:PreTools}

\subsubsection{The Laplace Mechanism}
The most basic constructions of differentially private algorithms are via the Laplace mechanism as follows.

\begin{definition}[The Laplace Distribution]
A random variable has probability distribution $\Lap(b)$ if its probability density function is $f(x)=\frac{1}{2b}\exp(-\frac{|x|}{b})$, where $x\in\R$.
\end{definition}

\begin{definition}[Sensitivity]
A function $f:X^m \rightarrow \R^n$ has {\em sensitivity $k$} if for every neighboring $D,D'\in X^m$, it holds that $||f(D)-f(D')||_1\leq k$.
\end{definition}

\begin{theorem}[The Laplacian Mechanism \cite{DMNS06}]\label{thm:lap}
Let $f:X^m \rightarrow \R^n$ be a sensitivity $k$ function. The mechanism $A$ that on input $D\in X^m$ 
adds independently generated noise with distribution $\Lap(\frac{k}{\epsilon})$ to each of the $n$ output terms of $f(D)$ preserves $\epsilon$-differential privacy. Moreover,
$$\Pr\Big[\exists i \; s.t. \; |A_i(D)-f_i(D)|>\Delta\Big]\leq n\cdot\exp\left(-\frac{\epsilon \Delta}{k}\right),$$
where $A_i(D)$ and $f_i(D)$ are the $i^{\text th}$ coordinates of $A(D)$ and $f(D)$.
\end{theorem}

\subsubsection{The Exponential Mechanism}
We next describe the exponential mechanism of McSherry and Talwar~\cite{MT07}.
Let $X$ be a domain and $H$ a set of solutions.
Given a quality function $q:X^*\times H \rightarrow\N$, and a database $S\in X^*$, the goal is to chooses a solution $h\in H$ approximately maximizing $q(S,h)$. The mechanism chooses a solution probabilistically, where the probability mass that is assigned to each solution $h$ increases exponentially with its quality $q(S,h)$:

\begin{center}
\noindent\fbox{
\parbox{.95\columnwidth}{
{\bf Input:} parameter $\epsilon$, finite solution set $H$, database $S\in X^m$, and a sensitivity 1 quality function $q$.
\begin{enumerate}
	\item Randomly choose $h \in H$ with probability
	$\frac{\exp\left(\epsilon \cdot q(S,h) /2 \right)}{\sum_{f\in H}\exp\left(\epsilon \cdot q(S,f) /2 \right)}.$
	\item Output $h$.
\end{enumerate}
}}
\end{center}

\begin{proposition}[Properties of the Exponential Mechanism]\label{prop:expMech}
(i) The exponential mechanism is $\epsilon$- differentially private. (ii)
Let $\hat{e}\triangleq\max_{f\in H}\{q(S,f)\}$ and $\Delta>0$. The exponential mechanism outputs a solution $h$ such that $q(S,h)\leq(\hat{e} - \Delta m)$ with probability at most $|H| \cdot \exp(-\epsilon \Delta m /2)$.
\end{proposition}

Kasiviswanathan et al.~\cite{KLNRS08} showed in 2008 that the exponential mechanism can be used as a generic private learner -- when used with the quality function $q(S,h)=|\{i:h(x_i)=y_i\}|$, the probability that the exponential mechanism outputs a hypothesis $h$ such that $\error_S(h)>\min_{f\in H}\{\error_S(f)\} + \Delta$ is at most $|H| \cdot \exp(-\epsilon \Delta m /2)$. This results in a generic private proper-learner for every finite concept class $C$, with sample complexity $O_{\alpha,\beta,\epsilon}(\log|C|)$.

\subsubsection{Stability and Privacy -- $\AAA_{\rm dist}$}
We restate a simplified variant of algorithm $\AAA_{\rm dist}$ by Smith and Thakurta~\cite{ADist}, which is an instantiation of the Propose-Test-Release framework~\cite{DworkLei}. 
Let $q:X^*\times H\rightarrow\N$ be a sensitivity-1 quality function
over a domain $X$ and a set of solutions $H$.
Given a database $S\in X^*$, the goal is to choose a solution $h\in H$ maximizing $q(S,h)$, under the assumption that the optimal solution $h$ scores much better than any other solution in $H$.
\begin{center}
\noindent\fbox{
\parbox{.95\columnwidth}{
{\bf Algorithm $\AAA_{\rm dist}$} \\
{\bf Input:} parameters $\epsilon,\delta$, database $S\in X^*$, sensitivity-1 quality function $q$.
\begin{enumerate}[topsep=-3pt, rightmargin=5pt]
\item Let $h_1 \neq h_2$ be two highest score solutions in $H$, where $q(S,h_1)\geq q(S,h_2)$.
\item Let ${\rm gap}=q(S,h_1) - q(S,h_2)$ and ${\rm gap}^*={\rm gap}+\Lap(\frac{1}{\epsilon})$.
\item If ${\rm gap}^*<\frac{1}{\epsilon}\log(\frac{1}{\delta})$ then output $\bot$ and halt.
\item Output $h_1$.\\
\end{enumerate}
}}
\end{center}

\begin{proposition}[Properties of $\AAA_{\rm dist}$~\cite{ADist}]\label{prop:aDist}
(i) Algorithm $\AAA_{\rm dist}$ is $(\epsilon,\delta)$- differentially private. (ii) 
When given an input database $S$ for which ${\rm gap}\geq\frac{1}{\epsilon}\log(\frac{1}{\beta\delta})$, algorithm $\AAA_{\rm dist}$ outputs $h_1$ maximizing $q(h,S)$ with probability at least $(1-\beta)$.
\end{proposition}

\subsection{Concentration Bounds}

Let $X_1,\dots,X_n$ be independent random variables where $\Pr[X_i=1]=p$ and $\Pr[X_i=0]=1-p$ for some $0<p<1$. Clearly, $\E[\sumnl_i{X_i}]=pn$. The Chernoff bounds show that the sum is concentrated around this expected value:
\begin{align*}
&\Pr\left[\sumnl_i{X_i}>(1+\delta)pn\right]\leq \exp\left(-pn\delta^2/3\right) \;\;\text{ for } \delta>0,\nonumber\\
&\Pr\left[\sumnl_i{X_i}<(1-\delta)pn\right]\leq \exp\left(-pn\delta^2/2\right) \;\;\text{ for } 0<\delta<1.\nonumber\\
\end{align*}

\section{Learning with Approximate Privacy}

We present proper $(\epsilon,\delta)$-private learners for two simple concept classes, $\point_d$ and $\thresh_d$, demonstrating separations between pure and approximate private proper learning.

\subsection{$(\epsilon,\delta)$-PPAC Learner for $\point_d$}\label{sec:pointLearner}

\begin{definition}
For $j\in X_d$ let $c_j:X_d \rightarrow\{0,1\}$ be defined as $c_j(x)=1$ if $x=j$ and $c_j(x)=0$ otherwise. Define the concept class $\point_d = \{c_j\}_{j\in X_d}$.
\end{definition}

Note that the VC dimension of $\point_d$ is 1, and, therefore, there exists a {\em proper} non-private learner for $\point_d$ with sample complexity $O_{\alpha,\beta}(1)$. 
Beimel et al.~\cite{BBKN12} proved that every {\em proper} $\epsilon$-private learner for $\point_d$ must have sample complexity $\Omega(d)=\Omega(\log|\point_d|)$. They also showed that there exists an {\em improper} $\epsilon$-private learner for this class, with sample complexity $O_{\alpha,\beta,\epsilon}(1)$. 
An alternative private learner for this class was presented in~\cite{BNS13}.

As we will now see, algorithm $\AAA_{\rm dist}$ (defined in Section~\ref{sec:PreTools}) can be used as a {\em proper} $(\epsilon,\delta)$-private learner for $\point_d$ with sample complexity $O_{\alpha,\beta,\epsilon,\delta}(1)$. This is our first (and simplest) example separating the sample complexity of pure and approximate private proper-learners. Consider the following algorithm.
\begin{center}
\noindent\fbox{
\parbox{.95\columnwidth}{
{\bf Input:} parameters $\alpha,\beta,\epsilon,\delta$, and a database $S\in (X_{d+1})^m$.
\begin{enumerate}[topsep=-3pt, rightmargin=5pt]
\item For every $x\in X_d$, define $q(S,x)$ as the number of appearances of $(x,1)$ in $S$.
\item Execute $\AAA_{\rm dist}$ on $S$ with the quality function $q$ and parameters $\frac{\alpha}{2},\frac{\beta}{2},\epsilon,\delta$.
\item If the output was $j$ then return $c_j$.
\item Else, if the output was $\bot$ then return a random $c_i\in\point_d$.\\
\end{enumerate}
}}
\end{center}

\begin{lemma}
Let $\alpha,\beta,\epsilon,\delta$ be s.t. $\frac{1}{\alpha\beta}\leq2^d$. The above algorithm is an efficient $(\alpha,\beta,\epsilon,\delta)$-PPAC proper learner for $\point_d$ using a sample of $m=O\left(\frac{1}{\alpha\epsilon}\ln(\frac{1}{\beta\delta})\right)$ labeled examples.
\end{lemma}

For intuition, consider a target concept $c_j$ and an underlying distribution $\DDD$. Whenever $\DDD(j)$ is noticeable, a typical sample $S$ contains many copies of the point $j$ labeled as $1$. As every other point $i\neq j$ will be labeled as $0$, we expect $q(S,j)$ to be significantly higher than any other $q(S,i)$, and we can use algorithm $\AAA_{\rm dist}$ to identify $j$.

\begin{proof}
As algorithm $\AAA_{\rm dist}$ is $(\epsilon,\delta)$-differentially private, all that remains is the utility analysis.
Fix a target concept $c_\ell\in\point_d$ and a distribution $\DDD$ on $X_d$. In a typical sample $S$, the only point that can appear with the label 1 is $\ell$, and algorithm $\AAA_{\rm dist}$ has two possible outputs: $\ell,\bot$. 

If $\DDD(\ell)>\alpha$ then (using the Chernoff bound), 
with probability at least $\left(1-\exp(-\frac{\alpha m}{8})\right)$, the labeled example $(\ell,1)$ appears in $S$ at least $r=\frac{\alpha m}{2}$ times.
Note that $q(S,\ell)\geq r$, and every $i\neq\ell$ has quality $q(S,i)=0$. For $m\geq\frac{8}{\alpha\epsilon}\ln(\frac{4}{\beta\delta})$, by Proposition~\ref{prop:aDist}, this gap is big enough s.t. algorithm $\AAA_{\rm dist}$ outputs $\ell$ with probability at least $(1-\frac{\beta}{2})$.
Therefore, when $\DDD(\ell)>\alpha$, the probability of $A$ outputting an $\alpha$-good solution is at least $(1-\exp(-\frac{\alpha m}{8}))(1-\frac{\beta}{2})$, which is at least $(1-\beta)$ for $m\geq\frac{8}{\alpha}\ln(\frac{2}{\beta})$.

If, on the other hand, $\DDD(\ell)\leq\alpha$, then algorithm $A$ will fail to output an $\alpha$-good solution only if $\AAA_{\rm dist}$ outputs $\bot$, and algorithm $A$ chooses a hypothesis $c_i$ s.t. $i\neq\ell$ and $\DDD(i)>\alpha$. But there could be at most $\frac{1}{\alpha}$ such points, and the probability of $A$ failing is at most $\frac{1}{\alpha 2^d}$. Assuming $2^d\geq\frac{1}{\alpha\beta}$, this probability is at most $\beta$.
\end{proof}

\begin{remark}
Recall that the above algorithm outputs a random $c_i\in\point_d$ whenever $\AAA_{\rm dist}$ outputs $\bot$. In order for this random $c_i$ to be good (w.h.p.) we needed $2^d$ (i.e., the number of possible concepts) to be at least $\frac{1}{\alpha\beta}$. This requirement could be avoided by outputting the all zero hypothesis $c_0\equiv0$ whenever $\AAA_{\rm dist}$ outputs $\bot$. However, this approach results in a {\em proper} learner only if we add the all zero concept to $\point_d$.
\end{remark}

\subsection{Towards a Proper $(\epsilon,\delta)$-PPAC Learner for $\thresh_d$}
\label{sec:introInterval}

\begin{definition}
For $0\leq j\leq 2^d$ let $c_j:X_d \rightarrow\{0,1\}$ be defined as $c_j(x)=1$ if $x<j$ and $c_j(x)=0$ otherwise. Define the concept class $\thresh_d = \{c_j\}_{0\leq j\leq 2^d}$.
\end{definition}

Note that $\VC(\thresh_d)=1$, and, therefore, there exists a {\em proper} non-private learner for $\thresh_d$ with sample complexity $O_{\alpha,\beta}(1)$. As $|\thresh_d|=2^d+1$, one can use the generic construction of Kasiviswanathan et al.~\cite{KLNRS08} and get a proper $\epsilon$-private learner for this class with sample complexity $O_{\alpha,\beta,\epsilon}(d)$.
Feldman and Xiao~\cite{FX14} showed that this is in fact optimal, and every $\epsilon$-private learner for this class (proper or improper) must have sample complexity $\Omega(d)$.

Our learner for $\point_d$ relied on a strong ``stability'' property of the problem: 
Given a labeled sample, either a random concept is (w.h.p.) a good output, or,
there is exactly one consistent concept in the class, and every other concept has large empirical error.
This, however, is not the case when dealing with $\thresh_d$. In particular, many hypotheses can have low empirical error, and changing a single entry of a sample $S$ can significantly affect the set of hypotheses consistent with it.

In Section~\ref{sec:Concave}, we present a proper $(\epsilon,\delta)$-private learner for $\thresh_d$ with sample complexity (roughly) $2^{O(\log^*(d))}$. We use this section for motivating the construction. We start with two simplifying assumptions. First, when given a labeled sample $S$, we aim at choosing a hypothesis $h\in\thresh_d$ approximately minimizing the empirical error (rather than the generalization error).
Second, we assume that we are given a ``diverse'' sample $S$ that contains many points labeled as $1$ and many points labeled as $0$. Those two assumptions (and any other informalities made hereafter) will be removed in Section~\ref{sec:Concave}.

Assume we are given as input a sample $S=(x_i,y_i)_{i=1}^m$ labeled by some unknown $c_{\ell}\in\thresh_d$. We would now like to choose a hypothesis $h\in\thresh_d$ with small empirical error on $S$, and we would like to do so while accessing the sample $S$ only through differentially private tools.

We refer to points labeled as $1$ in $S$ as {\em ones}, and to points labeled as $0$ as {\em zeros}.
Imagine for a moment that we already have a differentially private algorithm that given $S$ outputs an interval $G\subseteq X_d$ with the following two properties:
\begin{enumerate}
	\item The interval $G$ contains ``a lot'' of ones, {\em and} ``a lot'' of zeros in $S$.
	\item Every interval $I\subseteq X_d$ of length $\leq \frac{|G|}{k}$ does not contain, simultaneously, ``too many'' ones and ``too many'' zeros in $S$, where $k$ is some constant.
\end{enumerate}

Such an interval will be referred to as a {\em $k$-good} interval.
Note that a $k$-good interval is, in particular, an $\ell$-good interval for every $\ell\geq k$.
Figure~\ref{fig:k-good-interval} illustrates such an interval $G$, where the dotted line represents the (unknown) target concept, and the bold dots correspond to sample points.

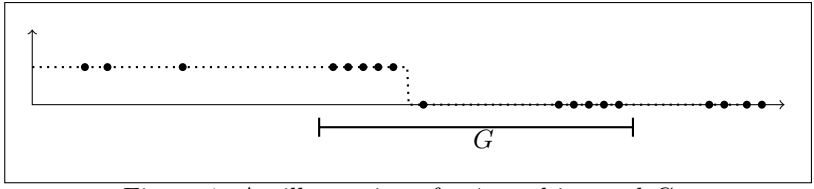
\begin{figure}[H]
\begin{center}
\begin{tikzpicture}[xscale=1,yscale=0.5,show background rectangle,inner frame sep=10pt]
\draw[black,thick,dotted,domain=0:9.5, samples=500] plot(\x, {floor((1.999-\x/5))});
\draw [<->] (0,2) -- (0,0) -- (10,0);
\draw [|-|,thick] (3.8,-0.6) -- (8,-0.6) ;
\node at (6,-0.9) {$G$};

\node[fill,circle,inner sep=0pt,minimum size=3pt] (a1) at (0.7,1) {};
\node[fill,circle,inner sep=0pt,minimum size=3pt] (a1) at (1,1) {};
\node[fill,circle,inner sep=0pt,minimum size=3pt] (a1) at (2,1) {};
\node[fill,circle,inner sep=0pt,minimum size=3pt] (a1) at (4,1) {};
\node[fill,circle,inner sep=0pt,minimum size=3pt] (a1) at (4.2,1) {};
\node[fill,circle,inner sep=0pt,minimum size=3pt] (a1) at (4.4,1) {};
\node[fill,circle,inner sep=0pt,minimum size=3pt] (a1) at (4.6,1) {};
\node[fill,circle,inner sep=0pt,minimum size=3pt] (a1) at (4.8,1) {};

\node[fill,circle,inner sep=0pt,minimum size=3pt] (a1) at (5.2,0) {};
\node[fill,circle,inner sep=0pt,minimum size=3pt] (a1) at (7,0) {};
\node[fill,circle,inner sep=0pt,minimum size=3pt] (a1) at (7.2,0) {};
\node[fill,circle,inner sep=0pt,minimum size=3pt] (a1) at (7.4,0) {};
\node[fill,circle,inner sep=0pt,minimum size=3pt] (a1) at (7.6,0) {};
\node[fill,circle,inner sep=0pt,minimum size=3pt] (a1) at (7.8,0) {};
\node[fill,circle,inner sep=0pt,minimum size=3pt] (a1) at (9,0) {};
\node[fill,circle,inner sep=0pt,minimum size=3pt] (a1) at (9.2,0) {};
\node[fill,circle,inner sep=0pt,minimum size=3pt] (a1) at (9.5,0) {};
\node[fill,circle,inner sep=0pt,minimum size=3pt] (a1) at (9.7,0) {};
\end{tikzpicture}
\end{center}
\caption{An illustration of a 4-good interval $G$. \label{fig:k-good-interval}}
\end{figure}

Given such a $4$-good interval $G$, we can (without using the sample $S$) define a set $H$ of five hypotheses s.t.\ at least one of them has small empirical error. To see this, consider Figure~\ref{fig:divide-4-good}, where $G$ is divided into four equal intervals $g_1,g_2,g_3,g_4$, and five hypotheses $h_1,\ldots,h_5$ are defined 
s.t.\ the points where they switch from one to zero are located at the edges of $g_1,g_2,g_3,g_4$.

Now, as the interval $G$ contains both ones and zeros, it must be that the target concept $c_{\ell}$ switches from $1$ to $0$ inside $G$. Assume without loss of generality that this switch occurs inside $g_2$. Note that $g_2$ is of length $\frac{|G|}{4}$ and, therefore, either does not contain too many ones, and $h_2$ is ``close'' to the target concept, or does not contain too many zeros, and $h_3$ is ``close'' to the target concept. For this argument to go through we need ``not too many'' to be smaller than $\alpha m$ (say $\frac{3}{4}\alpha m$), where $\alpha$ is our approximation parameter and $m$ is the sample size.

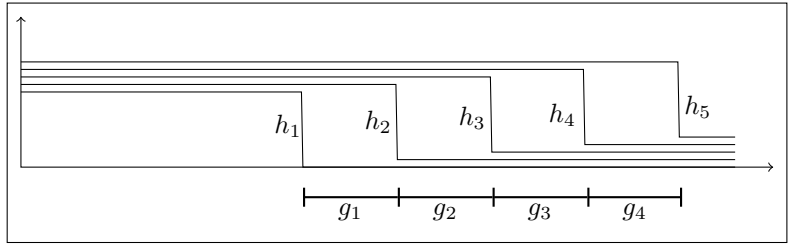
\begin{figure}[H]
\begin{center}
\begin{tikzpicture}[xscale=1,yscale=1,show background rectangle,inner frame sep=5pt]
\draw[domain=0:9.5, samples=500] plot(\x, {0.4+min(1,floor((1.999-(\x-3.75)/5)))});
\draw[domain=0:9.5, samples=500] plot(\x, {0.3+min(1,floor((1.999-(\x-2.5)/5)))});
\draw[domain=0:9.5, samples=500] plot(\x, {0.2+min(1,floor((1.999-(\x-1.25)/5)))});
\draw[domain=0:9.5, samples=500] plot(\x, {0.1+min(1,floor((1.999-(\x)/5)))});
\draw[domain=0:9.5, samples=500] plot(\x, {max(min(1,floor((1.999-(\x+1.25)/5))),0)});
\draw [<->] (0,2) -- (0,0) -- (10,0);
\draw [|-,thick] (3.75,-0.4) -- (5.01,-0.4) ;
\draw [|-,thick] (5.01,-0.4) -- (6.27,-0.4) ;
\draw [|-,thick] (6.27,-0.4) -- (7.53,-0.4) ;
\draw [|-|,thick] (7.53,-0.4) -- (8.79,-0.4) ;
\node at (4.38,-0.6) {$g_1$};
\node at (5.64,-0.6) {$g_2$};
\node at (6.9,-0.6) {$g_3$};
\node at (8.16,-0.6) {$g_4$};

\node at (3.55,0.55) {$h_1$};
\node at (4.75,0.6) {$h_2$};
\node at (6,0.65) {$h_3$};
\node at (7.2,0.7) {$h_4$};
\node at (9,0.8) {$h_5$};
\end{tikzpicture}
\end{center}
\caption{Extracting a small set of hypotheses from a good interval. \label{fig:divide-4-good}}
\end{figure}

After defining such a set $H$, we could use the exponential mechanism to choose a hypothesis $h\in H$ with small empirical error on $S$. As the size of $H$ is constant, this requires only a constant number of samples. To conclude, finding a $4$-good interval $G$ (while preserving privacy) is sufficient for choosing a good hypothesis. We next explain how to find such an interval.

Assume, for now, that we have a differentially private algorithm that given a sample $S$, returns an {\em interval length} $J$ s.t. there exists a {\em 2-good} interval $G\subseteq X_d$ of length $|G|=J$. This length $J$ is used to find an explicit {\em 4-good} interval as follows. Divide $X_d$ into intervals $\{A_i\}$ of length $2J$, and into intervals $\{B_i\}$ of length $2J$ right shifted by $J$ as in Figure~\ref{fig:axisDivision}.

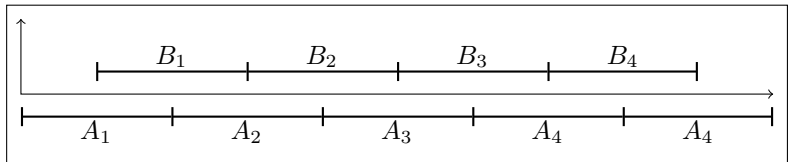
\begin{figure}[H]
\begin{center}
\begin{tikzpicture}[xscale=1,yscale=1,show background rectangle,inner frame sep=5pt]
\draw [<->] (0,1) -- (0,0) -- (10,0);
\draw [|-,thick] (0,-0.3) -- (2,-0.3) ;
\draw [|-,thick] (2,-0.3) -- (4,-0.3) ;
\draw [|-,thick] (4,-0.3) -- (6,-0.3) ;
\draw [|-,thick] (6,-0.3) -- (8,-0.3) ;
\draw [|-|,thick] (8,-0.3) -- (10,-0.3) ;
\node at (1,-0.5) {$A_1$};
\node at (3,-0.5) {$A_2$};
\node at (5,-0.5) {$A_3$};
\node at (7,-0.5) {$A_4$};
\node at (9,-0.5) {$A_4$};

\draw [|-,thick] (1,0.3) -- (3,0.3) ;
\draw [|-,thick] (3,0.3) -- (5,0.3) ;
\draw [|-,thick] (5,0.3) -- (7,0.3) ;
\draw [|-|,thick] (7,0.3) -- (9,0.3) ;
\node at (2,0.5) {$B_1$};
\node at (4,0.5) {$B_2$};
\node at (6,0.5) {$B_3$};
\node at (8,0.5) {$B_4$};
\end{tikzpicture}
\end{center}
\caption{Dividing the axis $X_d$ into intervals of length $2J$. \label{fig:axisDivision}}
\end{figure}

As the promised $2$-good interval $G$ is of length $J$, at least one of the above intervals contains $G$. We next explain how to privately choose such interval. If, e.g., $G\subseteq A_2$  then $A_2$ contains both a lot of zeros and a lot of ones. The target concept must switch inside $A_2$, and, therefore, every other $A_i\neq A_2$ cannot contain both zeros and ones. For every interval $A_i$, define its quality $q(A_i)$ to be the minimum between the number of zeros in $A_i$ and the number of ones in $A_i$. Therefore, $q(A_2)$ is large, while $q(A_i)=0$ for every $A_i\neq A_2$. That is, $A_2$ scores much better than any other $A_i$ under this quality function $q$. The sensitivity of $q()$ is one and we can use algorithm $\AAA_{\rm dist}$ to privately identify $A_2$. It suffices, e.g., that $q(A_2)\geq\frac{1}{4}\alpha m$, and we can, therefore, set our ``a lot'' bound to be $\frac{1}{4}\alpha m$.
Recall that $G\subseteq A_2$ is a $2$-good interval, and that $|A_2|=2|G|$. The identified $A_2$ is, therefore, a $4$-good interval.

To conclude, if we could indeed find (while preserving privacy) a length $J$ s.t. there exists a $2$-good interval $G$ of that length, then our task would be completed. 

\paragraph{Computing the interval length $\mathbf{J}$.} 
At first attempt, one might consider preforming a binary search for such a length $0\leq J\leq 2^d$, in which every comparison will be made using the Laplace mechanism.
More specifically, for every length $0\leq J\leq 2^d$, define
$$Q(J)=\max\limits_{\substack{[a,b]\subseteq X_d\\b-a=J}} \Bigg\{ \min\left\{ \, 	
					\begin{array}{c}
					\text{number of}\\
					\text{zeros in } [a,b]
					\end{array} \; , \;			
					\begin{array}{c}
					\text{number of}\\
					\text{ones in } [a,b]
					\end{array}	 \, \right\}  \Bigg\}.$$
If, e.g., $Q(J)=100$ for some $J$, then {\em there exists} an interval $[a,b]\subseteq X_d$ of length $J$ that contains at least $100$ ones and at least $100$ zeros. Moreover, {\em every} interval of length $\leq J$ either contains at most $100$ ones, or, contains at most $100$ zeros.

Note that $Q(\cdot)$ is a monotonically non-decreasing function, and that $Q(0)=0$ (as in a correctly labeled sample a point cannot appear both with the label 1 and with the label 0). Recall our assumption that the sample $S$ is ``diverse'' (contains many points labeled as $1$ and many points labeled as $0$), and, therefore, $Q(2^d)$ is large. Hence, there exists a $J$ s.t. $Q(J)$ is ``big enough'' (say at least $\frac{1}{4}\alpha m$) while $Q(J-1)$ is ``small enough'' (say at most $\frac{3}{4}\alpha m$). That is, a $J$ s.t. (1) there exists an interval of length $J$ containing lots of ones and lots of zeros; and (2), every interval of length $<J$ cannot contain too many ones and too many zeros simultaneously. Such a $J$ can easily be (privately) obtained using a (noisy) binary search.
However, as there are $d$ noisy comparisons, this solution requires a sample of size $d^{O(1)}$ in order to achieve reasonable utility guarantees.\\

As a second attempt, one might consider preforming a binary search, not on $0\leq J\leq 2^d$, but rather on the power $j$ of an interval of length $2^j$. That is, preforming a search for a power $0\leq j\leq d$ for which there exists a $2$-good interval of length $2^j$. Here there are only $\log(d)$ noisy comparisons, and the sample size is reduced to $\log^{\Omega(1)}(d)$.
Again, a (noisy) binary search on $0\leq j\leq d$ can (privately) yield an appropriate length $J=2^j$ s.t. $Q(2^j)$ is ``big enough'', while $Q(2^{j-1})$ is ``small enough''. Such a $J=2^j$ is, indeed, a length of a $2$-good interval. Too see this, note that as $Q(2^j)$ is ``big enough'', there exists an interval of length $2^j$ containing lots of ones and lots of zeros. Moreover, as $Q(2^{j-1})$ is ``small enough'', 
every interval of length $2^{j-1}=\frac{1}{2}2^j$ cannot contain too many ones and too many zeros simultaneously.

\begin{remark}
A binary search as above would have to operate on noisy values of $Q(\cdot)$ (as otherwise differential privacy cannot be obtained). For this reason, we set the bounds for ``big enough'' and ``small enough'' to overlap. Namely, we search for a value $j$ such that $Q(2^j)\geq\frac{\alpha}{4}m$ and  $Q(2^{j-1})\leq\frac{3\alpha}{4}m$, where $\alpha$ is our approximation parameter, and $m$ is the sample size.
\end{remark}

To summarize, using a binary search we find a length $J=2^j$ such that there exists a 2-good interval of length $J$. Then, using $\AAA_{\rm dist}$, we find a 4-good interval. Finally, we partition this interval to 4 intervals, and using the exponential mechanism we choose a starting point or end point of one of these intervals as our the threshold.

We will apply recursion to reduce the costs of computing $J=2^j$ to $2^{O(\log^*(d))}$. The tool performing the recursion would be formalized and analyzed in the next section. This tool will later be used in our construction of a proper $(\epsilon,\delta)$-private learner for $\thresh_d$.

\subsection{Privately Approximating Quasi-Concave Promise Problems}
\label{sec:Concave}
We next define the notions that enable our recursive algorithm.
\begin{definition}
A function $Q(\cdot)$ is {\em quasi-concave} if $Q(\ell)\geq \min\{Q(i),Q(j)\}$ for every $i\leq \ell \leq j$.
\end{definition}

\begin{definition}[Quasi-Concave Promise Problem]
A {\em Quasi-Concave Promise Problem} consists of an ordered set of possible solutions $[0,T]=\{0,1,\ldots,T\}$, a database $S\in X^m$, a sensitivity-1 quality function $Q:X^*\times[0,T] \rightarrow\R$, an approximation parameter $\alpha$, and another parameter $r$ (called a {\em quality promise}). 

If $Q(S,\cdot)$ is quasi-concave and if there exists a solution $p\in[0,T]$ for which $Q(S,p)\geq r$ then a good output for the problem is a solution $k\in[0,T]$ satisfying $Q(S,k)\geq(1-\alpha)r$. The outcome is not restricted otherwise.
\end{definition}

\begin{example}\label{example:intervalConcave}
Consider a sample $S=(x_i,y_i)_{i=1}^m$, labeled by some target function $c_j\in\thresh_d$.
The goal of choosing a hypothesis with small empirical error can be viewed as a quasi-concave promise problem as follows. Set the range of possible solutions to $[0,2^d]$, the approximation parameter to $\alpha$ and the quality promise to $m$. Define $Q(S,k)=|\{ i : c_k(x_i)=y_i \}|$; i.e., $Q(S,k)$ is the number of points in $S$ correctly classified by $c_k\in\thresh_d$. 
Note that the target concept $c_j$ satisfies $Q(S,j)=m$.
Our task is to find a hypothesis $h_k\in\thresh_d$ s.t.\ $\error_S(h_k)\leq\alpha$, which is equivalent to finding $k\in[0,2^d]$ s.t.\ $Q(S,k)\geq(1-\alpha)m$.

To see that $Q(S,\cdot)$ is quasi-concave, let $u\leq v\leq w$ be s.t. $Q(S,u),Q(S,w)\geq\lambda$.
Consider $j$, the index of the target concept, and assume w.l.o.g. that $j\leq v$ (the other case is symmetric). That is, $j\leq\ v\leq w$.
Note that $c_v$ errs only on points in between $j$ and $v$, and $c_w$ errs on all these points. That is, $\error_S(c_v)\leq\error_S(c_w)$, and, therefore, $Q(S,v)\geq\lambda$. See Figure~\ref{fig:Qconcave} for an illustration.
\end{example}

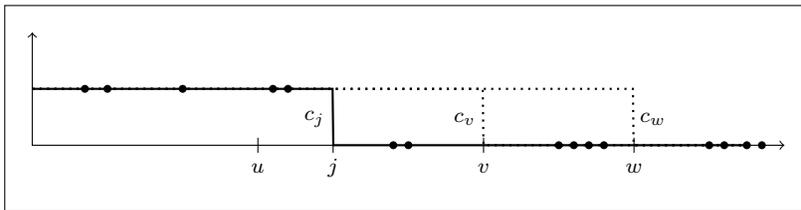
\begin{figure}[H]
\begin{center}
\begin{tikzpicture}[xscale=1,yscale=0.5,show background rectangle,inner frame sep=10pt]
\draw[black,thick,domain=0:9.5, samples=1000] plot(\x, {1.5*max(0,floor((1.999-\x/4)))});
\draw[black,thick,dotted,domain=0:9.5, samples=1000] plot(\x, {1.5*max(0,floor((1.999-\x/6)))});
\draw[black,thick,dotted,domain=0:9.5, samples=1000] plot(\x, {1.5*max(0,floor((1.999-\x/8)))});
\draw [<->] (0,3) -- (0,0) -- (10,0);

\node[fill,circle,inner sep=0pt,minimum size=3pt] at (0.7,1.5) {};
\node[fill,circle,inner sep=0pt,minimum size=3pt] at (1,1.5) {};
\node[fill,circle,inner sep=0pt,minimum size=3pt] at (2,1.5) {};
\node[fill,circle,inner sep=0pt,minimum size=3pt] at (3.2,1.5) {};
\node[fill,circle,inner sep=0pt,minimum size=3pt] at (3.4,1.5) {};
\node[fill,circle,inner sep=0pt,minimum size=3pt] at (4.8,0) {};
\node[fill,circle,inner sep=0pt,minimum size=3pt] at (5,0) {};
\node[fill,circle,inner sep=0pt,minimum size=3pt] at (7,0) {};
\node[fill,circle,inner sep=0pt,minimum size=3pt] at (7.2,0) {};
\node[fill,circle,inner sep=0pt,minimum size=3pt] at (7.4,0) {};
\node[fill,circle,inner sep=0pt,minimum size=3pt] at (7.6,0) {};
\node[fill,circle,inner sep=0pt,minimum size=3pt] at (9,0) {};
\node[fill,circle,inner sep=0pt,minimum size=3pt] at (9.2,0) {};
\node[fill,circle,inner sep=0pt,minimum size=3pt] at (9.5,0) {};
\node[fill,circle,inner sep=0pt,minimum size=3pt] at (9.7,0) {};

\draw [-] (3,-0.2) -- (3,0.2);
\draw [-] (4,-0.2) -- (4,0.2);
\draw [-] (6,-0.2) -- (6,0.2);
\draw [-] (8,-0.2) -- (8,0.2);

\node at (3,-0.6) {\footnotesize $u$};
\node at (4,-0.6) {\footnotesize $j$};
\node at (6,-0.6) {\footnotesize $v$};
\node at (8,-0.6) {\footnotesize $w$};

\node at (3.75,0.7) {\footnotesize $c_j$};
\node at (5.75,0.7) {\footnotesize $c_v$};
\node at (8.25,0.7) {\footnotesize $c_w$};

\end{tikzpicture}
\end{center}
\caption{An illustration for Example~\ref{example:intervalConcave}.
Here $c_j$ is the target concept and the bold dots correspond to sample points.
Note that $c_w$ errs on every point on which $c_v$ errs.
\label{fig:Qconcave}}
\end{figure}

\begin{remark}\label{remark:intervalConcave}
Note that if the sample $S$ in the above example is {\em not} consistent with any $c\in\thresh_d$, then there is no $j$ s.t. $Q(S,j)=m$, and the quality promise is void. Moreover, in such a case $Q(S,\cdot)$ might not be quasi-concave.
\end{remark}

We are interested in solving quasi-concave promise problems while preserving differential privacy. As motivated by Remark~\ref{remark:intervalConcave}, privacy must be preserved even when $Q(S,\cdot)$ is not quasi-concave or $Q(S,p) < r $ for all $p\in[0,T]$.
Our algorithm $RecConcave$ is presented in Figure~\ref{fig:RecConcave} (see inline comments for some of the underlying intuition).

\begin{figure}
\begin{center}
\noindent\fbox{
\parbox{\textwidth}{
{\bf Algorithm $RecConcave$}\\
{\bf Inputs:} Range $[0,T]$, quality function $Q$, quality promise $r$, parameters $\alpha,\epsilon,\delta$, and a database $S$.\\
{\bf Optional Input:} a bound $N\geq1$ on the number of recursive calls (set $N=\infty$ otherwise).
\begin{enumerate}[itemsep=5pt,rightmargin=7pt]

\item If [$(T\leq32)$ or $(N=1)$], then use the exponential mechanism with the quality function $Q$ and the parameter $\epsilon$ to choose and return an index $j\in [0,\ldots,T]$. Otherwise set $N=N-1$.

\item Let $T'$ be the smallest power of $2$ s.t. $T'\geq T$, and define $Q(S,i)=\min\{0, Q(S,T)\}$ for $T<i\leq T'$.

\item For $0\leq j\leq\log(T')$ let 
$L(S,j)=\max\limits_{\substack{[a,b]\subseteq [0,T']\\b-a+1=2^j}}\Bigg( \min\limits_{i\in[a,b]} \Big( Q(S,i) \Big) \Bigg)$. For $j=\log(T')+1$ let $L(S,j)=\min\{0,L(S,\log(T')\}$.
\begin{enumerate}[label=\gray{\%},topsep=-10pt]
\item \gray{If $L(S,j)=x$ then (1) there exists an interval $I\subseteq[0,T']$ of length $2^j$ s.t. $Q(S,i)\geq x$ for all $i\in I$;
and (2) in every interval $I\subseteq[0,T']$ of length $2^j$ there exists a point $i\in I$ s.t. $Q(S,i)\leq x$. 
Note that $L(S,j+1)\leq L(S,j)$. See Figure~\ref{fig:5} for an illustration.}
\end{enumerate}

\item Define the function $q(S,j)=\min\Big(  L(S,j)-(1-\alpha)r , r-L(S,j+1) \Big)$ where $0\leq j\leq\log(T')$.
\begin{enumerate}[label=\gray{\%},topsep=-10pt]
\item \gray{If $q(S,j)$ is high for some $j$, then there exists an interval $I=[a,a+2^j-1]$ s.t. every $i\in I$ has a quality $Q(S,i) >> (1-\alpha)r$, and for every interval $I'=[a',a'+2^{j+1}-1]$ there exists $i'\in I'$ with quality $Q(S,i) << r$. See Figure~\ref{fig:5}.}
\end{enumerate}

\item Let $R=\frac{\alpha}{2}r$.
\begin{enumerate}[label=\gray{\%},topsep=-10pt]
\item \gray{$R$ is the promise parameter for the recursive call. Note that for the maximal $j$ with $L(S,j)\geq(1-\frac{\alpha}{2})r$ we get  $q(S,j)\geq\frac{\alpha}{2}r$ = R.}
\end{enumerate}

\item Execute $RecConcave$ recursively on the range $\{0,\ldots,\log(T')\}$, the quality function $q(\cdot,\cdot)$, the promise $R$, an approximation parameter $\frac{1}{4}$, and $\epsilon,\delta,N$. Denote the returned value by $k$, and let $K=2^k$.
\begin{enumerate}[label=\gray{\%},topsep=-10pt]
\item \gray{If the call to $RecConcave$ was successful, then $k$ is s.t.\ $q(S,k)\geq(1-\frac{1}{4})R=\frac{3\alpha}{8} r$. That is, $L(S,k)\geq(1-\frac{5\alpha}{8})r$ and $L(S,k+1)\leq(1-\frac{3\alpha}{8})r$. Note in the top level call the approximation parameter $\alpha$ is arbitrary (given as input), and that in all of the lower level calls the approximation parameter is fixed at $\frac{1}{4}$.}
\end{enumerate}

\item Divide $[0,T']$ into the following intervals of length $8K$ (the last ones might be trimmed): \\
$ A_1=[0,8K-1],\; A_2=[8K,16K-1],\; A_3=[16K,24K-1], \ldots$ \\
$ B_1=[4K,12K-1],\; B_2=[12K,20K-1],\; B_3=[20K,28K-1], \ldots$
\begin{enumerate}[label=\gray{\%},topsep=-10pt]
\item \gray{We show that in at least one of those two partitions (say the $\{A_i\}$'s), there exists a good interval $A_g$ s.t.\ $Q(S,i)=r$ for some $i\in A_g$, and $Q(S,i) \leq (1-\frac{3\alpha}{8})r$ for all $i\in \{0,\ldots,T\}\setminus A_g$.}
\end{enumerate}
   
\item For every such interval $I\in\{A_i\}\cup\{B_i\}$ let $u(S,I)=\max\limits_{i\in I} \Big(Q(S,i) \Big)$.

\item Use algorithm $\AAA_{\rm dist}$ with parameters $\epsilon,\delta$ and the quality function $u(\cdot,\cdot)$, once to choose an interval $A\in\{A_i\}$, and once more to choose an interval $B\in\{B_i\}$.
\begin{enumerate}[label=\gray{\%},topsep=-10pt]
\item \gray{By the properties of $\AAA_{\rm dist}$, w.h.p.\ at least one of the returned $A$ and $B$ is good.}
\end{enumerate}

\item Use the exponential mechanism with the quality function $Q(\cdot,\cdot)$ and parameter $\epsilon$ to choose and return an index $j\in (A\cup B)$.
\begin{enumerate}[label=\gray{\%},topsep=-10pt]
\item \gray{We show that a constant fraction of the solutions in $(A\cup B)$ have high qualities, and, hence, the exponential mechanism needs only a constant sample complexity in order to achieve good utility guarantees.}
\end{enumerate}

\end{enumerate}
}}
\end{center}
\caption{Algorithm $RecConcave$. \label{fig:RecConcave}}
\end{figure}

\captionsetup[figure]{skip=-7pt}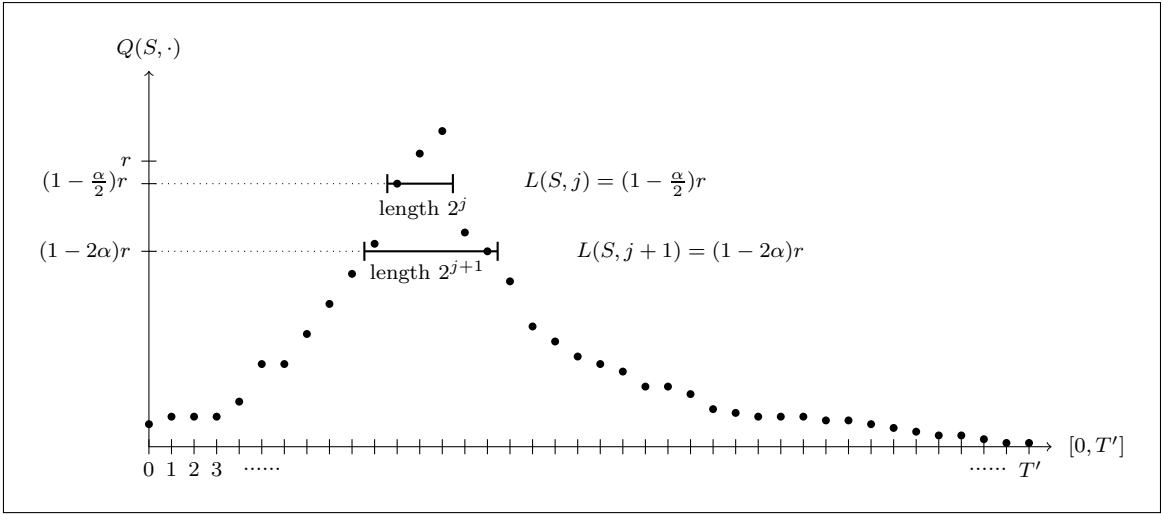
\begin{figure}
\begin{center}
\begin{tikzpicture}[xscale=1,yscale=1,show background rectangle,inner frame sep=10pt]
\draw [<->] (0,5) -- (0,0) -- (12,0);
\node at (0,5.3) {\footnotesize $Q(S,\cdot)$};
\node at (12.6,0) {\footnotesize $[0,T']$};

\foreach \x in {0,0.3,...,12} 
    \draw [-] (\x,-0.1) -- (\x,0.1);

\node at (0,-0.3) {\footnotesize $0$};
\node at (0.3,-0.3) {\footnotesize $1$};
\node at (0.6,-0.3) {\footnotesize $2$};
\node at (0.9,-0.3) {\footnotesize $3$};
\node at (1.5,-0.3) {\footnotesize ......};
\node at (11.15,-0.3) {\footnotesize ......};
\node at (11.72,-0.3) {\footnotesize $T'$};

\node[fill,circle,inner sep=0pt,minimum size=3pt] at (0,0.3) {};
\node[fill,circle,inner sep=0pt,minimum size=3pt] at (0.3,0.4) {};
\node[fill,circle,inner sep=0pt,minimum size=3pt] at (0.6,0.4) {};
\node[fill,circle,inner sep=0pt,minimum size=3pt] at (0.9,0.4) {};
\node[fill,circle,inner sep=0pt,minimum size=3pt] at (1.2,0.6) {};
\node[fill,circle,inner sep=0pt,minimum size=3pt] at (1.5,1.1) {};
\node[fill,circle,inner sep=0pt,minimum size=3pt] at (1.8,1.1) {};
\node[fill,circle,inner sep=0pt,minimum size=3pt] at (2.1,1.5) {};
\node[fill,circle,inner sep=0pt,minimum size=3pt] at (2.4,1.9) {};
\node[fill,circle,inner sep=0pt,minimum size=3pt] at (2.7,2.3) {};
\node[fill,circle,inner sep=0pt,minimum size=3pt] at (3,2.7) {};
\node[fill,circle,inner sep=0pt,minimum size=3pt] at (3.3,3.5) {};
\node[fill,circle,inner sep=0pt,minimum size=3pt] at (3.6,3.9) {};
\node[fill,circle,inner sep=0pt,minimum size=3pt] at (3.9,4.2) {};
\node[fill,circle,inner sep=0pt,minimum size=3pt] at (4.2,2.85) {};
\node[fill,circle,inner sep=0pt,minimum size=3pt] at (4.5,2.6) {};
\node[fill,circle,inner sep=0pt,minimum size=3pt] at (4.8,2.2) {};
\node[fill,circle,inner sep=0pt,minimum size=3pt] at (5.1,1.6) {};
\node[fill,circle,inner sep=0pt,minimum size=3pt] at (5.4,1.4) {};
\node[fill,circle,inner sep=0pt,minimum size=3pt] at (5.7,1.2) {};
\node[fill,circle,inner sep=0pt,minimum size=3pt] at (6,1.1) {};
\node[fill,circle,inner sep=0pt,minimum size=3pt] at (6.3,1) {};
\node[fill,circle,inner sep=0pt,minimum size=3pt] at (6.6,0.8) {};
\node[fill,circle,inner sep=0pt,minimum size=3pt] at (6.9,0.8) {};
\node[fill,circle,inner sep=0pt,minimum size=3pt] at (7.2,0.7) {};
\node[fill,circle,inner sep=0pt,minimum size=3pt] at (7.5,0.5) {};
\node[fill,circle,inner sep=0pt,minimum size=3pt] at (7.8,0.45) {};
\node[fill,circle,inner sep=0pt,minimum size=3pt] at (8.1,0.4) {};
\node[fill,circle,inner sep=0pt,minimum size=3pt] at (8.4,0.4) {};
\node[fill,circle,inner sep=0pt,minimum size=3pt] at (8.7,0.4) {};
\node[fill,circle,inner sep=0pt,minimum size=3pt] at (9,0.35) {};
\node[fill,circle,inner sep=0pt,minimum size=3pt] at (9.3,0.35) {};
\node[fill,circle,inner sep=0pt,minimum size=3pt] at (9.6,0.3) {};
\node[fill,circle,inner sep=0pt,minimum size=3pt] at (9.9,0.25) {};
\node[fill,circle,inner sep=0pt,minimum size=3pt] at (10.2,0.2) {};
\node[fill,circle,inner sep=0pt,minimum size=3pt] at (10.5,0.15) {};
\node[fill,circle,inner sep=0pt,minimum size=3pt] at (10.8,0.15) {};
\node[fill,circle,inner sep=0pt,minimum size=3pt] at (11.1,0.1) {};
\node[fill,circle,inner sep=0pt,minimum size=3pt] at (11.4,0.05) {};
\node[fill,circle,inner sep=0pt,minimum size=3pt] at (11.7,0.05) {};

\draw [-,dotted] (0,3.5) -- (3.155,3.5) ;
\draw [|-|,thick] (3.155,3.5) -- (4.055,3.5) ;
\node at (3.655,3.17) {\footnotesize length $2^j$};
\node at (6.2,3.5) {\footnotesize $L(S,j)=(1-\frac{\alpha}{2})r$};

\draw [-,dotted] (0,2.6) -- (2.85,2.6) ;
\draw [|-|,thick] (2.85,2.6) -- (4.65,2.6) ;
\node at (3.7,2.33) {\footnotesize length $2^{j+1}$};
\node at (7.2,2.6) {\footnotesize $L(S,j+1)= (1-2\alpha)r$};

\draw [-] (-0.1,3.8) -- (0.1,3.8) ;
\node at (-0.3,3.8) {\footnotesize $r$};
\draw [-] (-0.1,3.5) -- (0.1,3.5) ;
\node at (-0.85,3.5) {\footnotesize $(1-\frac{\alpha}{2})r$};
\draw [-] (-0.1,2.6) -- (0.1,2.6) ;
\node at (-0.85,2.6) {\footnotesize $(1-2\alpha)r$};

\end{tikzpicture}
\end{center}
\caption{A demonstration for the functions $L$ and $q$ from Steps~3,4 of $RecConcave$.
In the above illustration, every interval of length $2^j$ contains at least one point with quality at most $(1-\frac{\alpha}{2})r$, and there exists an interval of length $2^j$ containing only points with quality at least $(1-\frac{\alpha}{2})r$. Hence, $L(S,j)=(1-\frac{\alpha}{2})r$. Similarly, $L(S,j+1)=(1-2\alpha)r$. Therefore, for this $j$ we have that
$q(S,j)=\min\{ L(S,j)-(1-\alpha)r , r-L(S,j+1) \}=\frac{\alpha}{2}r$.
The reason for defining $q(\cdot,\cdot)$ is the following. We were interested in identifying a $j$ with an appropriate lower bound on $L(S,j)$ and with an appropriate upper bound on $L(S,j+1)$. That is, in order to decide whether a given $j$ is a good, we need to check both $L(S,j)$ and $L(S,j+1)$. After defining $q(S,\cdot)$, we can simply look for a $j$ with a high $q(S,j)$. A high $q(S,j)$ implies upper and lower bounds (respectively) on $L(S,j),L(S,j+1)$.
}\label{fig:5}
\end{figure}

We start the analysis of Algorithm $RecConcave$ by bounding the number of recursive calls.

\paragraph{Notation.} Given an integer $n$, let $\log^{\lceil * \rceil}(n)$ denote the number of times that the function $\lceil \log(x) \rceil$ must be iteratively applied before the result is less or equal to $1$, i.e., $\log^{\lceil * \rceil}(n)= 1+\log^{\lceil * \rceil}\lceil\log(n)\rceil$ if $n>1$ and zero otherwise.
Observe that $\log^{\lceil * \rceil}(n)=\log^*(n)$.\footnote{ Clearly $\log^{\lceil * \rceil}(n) \geq \log^*(n)$. Let $\ell$ be the smallest number of the form $2^{2^{\cdot^{\cdot^{\cdot^2}}}}$ s.t.\ $\ell\geq n$. We have that $\log^*(\ell)=\log^*(n)$, and that $\log^{\lceil * \rceil}(\ell)=\log^*(\ell)$ (as all of the numbers in the iterative process of $\log^{\lceil * \rceil}(\ell)$ will be integers). As $\log^{\lceil * \rceil}(\cdot)$ is monotonically non-decreasing we get $\log^{\lceil * \rceil}(n)\leq \log^{\lceil * \rceil}(\ell)=\log^*(\ell)=\log^*(n)$.}

\begin{observation}
On a range $[0,T]$ there could be at most $\log^{\lceil * \rceil}(T)=\log^*(T)$ recursive calls throughout the execution of $RecConcave$.
\end{observation}

Before proceeding to the privacy analysis, we make the following simple observation.

\begin{observation}\label{obs:minMaxSens1}
Let $\{ f_1,f_2,\ldots, f_N  \}$ be a set of sensitivity-1 functions mapping $X^*$ to $\R$. Then $f_{max}(S)=\max_i\{ f_i(S) \}$ and $f_{min}(S)=\min_i\{ f_i(S) \}$ are sensitivity-1 functions.
\end{observation}

We now proceed with the privacy analysis of algorithm $RecConcave$.

\begin{lemma}\label{lem:recPrivacy}
When executed on a sensitivity-1 quality function $Q$, parameters $\epsilon, \delta$, and a bound on the recursion depth $N$, algorithm $RecConcave$ preserves $(3N\epsilon,3N\delta)$-differential privacy.
\end{lemma}

\begin{proof}
Note that since $Q$ is a sensitivity-1 function, all of the quality functions defined throughout the execution of $RecConcave$ are of sensitivity 1 (see Observation~\ref{obs:minMaxSens1}). In each recursive call algorithm $RecConcave$ invokes at most three differentially private mechanisms -- once with the Exponential Mechanism (on Step~1 or on Step~11), and at most twice with algorithm $\AAA_{\rm dist}$ (on Step~9). As there are at most $N$ recursive calls, we conclude that throughout the entire execution algorithm $RecConcave$ invokes most $3N$ mechanisms, each $(\epsilon,\delta)$-differentially private. Hence, using Theorem~\ref{thm:composition3}, algorithm $RecConcave$ is $(3N\epsilon,3N\delta)$-differentially private.
\end{proof}

We now turn to proving the correctness of algorithm $RecConcave$. 
As the proof is by induction (on the number of recursive calls), we need to show that each of the recursive calls to $RecConcave$ is made with appropriate inputs.
We first claim that the function $q(S,\cdot)$ constructed in Step~4 is quasi-concave. Note that for this claim we do not need to assume that $Q(S,\cdot)$ is quasi-concave.

\begin{claim}\label{claim:qConcave}
Let $Q:X^*\times[0,T]\rightarrow\R$ be a quality function, and let the functions $L(\cdot,\cdot)$ and $q(\cdot,\cdot)$ be as in steps~3,~4 of algorithm $RecConcave$. Then, for every $S\in X^*$, it holds that $q(S,\cdot)$ is quasi-concave.
\end{claim}

\begin{proof}
Fix $S\in X^*$.
First observe that the function 
$$L(S,j)=\max\limits_{\substack{[a,b]\subseteq [0,T']\\b-a+1=2^j}}\Bigg( \min\limits_{i\in[a,b]} \Big( Q(S,i) \Big) \Bigg)$$
is monotonically non-increasing (as a function of $j$).
To see this, note that if $L(S,j)=\XXX$, then there exists an interval of length $2^j$ in which every point has quality at least $\XXX$. In particular, there exists such an interval of length $\frac{1}{2}2^j$, and $L(S,j-1)\geq\XXX$.

Now, let $i\leq\ell\leq j$ be s.t.\ $q(S,i),q(S,j)\geq x$. We get that $L(S,\ell)-(1-\alpha)r\geq L(S,j)-(1-\alpha)r\geq x$, and that $r-L(S,\ell+1)\geq r-L(S,i+1)\geq x$. Therefore, $q(S,\ell)\geq x$, and $q(S,\cdot)$ is quasi-concave.
\end{proof}

\remove{In some cases, it might be useful to bound the number of recursive calls in the execution of algorithm $RecConcave$ (using the input parameter $N$).
For example, this will enable us to construct, for every constant $\lambda$, a private learner for $\thresh_d$ with sample complexity $O_{\alpha,\beta,\epsilon}(\log(\frac{1}{\delta})+\underbrace{\log\log\cdots\log}_{\lambda \text{ times}}(2^d))$.}

\paragraph{Notation.} We use $\Nlog{N}(\cdot)$ to denote the outcome of $N$ iterative applications of the function $\lceil \log(\cdot) \rceil$, i.e., $\Nlog{N}(n)=\underbrace{\lceil\log\lceil\log\lceil\cdots\lceil\log}_{N \text{ times}}(n)\rceil\cdots\rceil\rceil\rceil$.
Observe that $\Nlog{N}(n)\leq 2+ \underbrace{\log\log\cdots\log}_{N \text{ times}}(n)$
for every $N\leq\log^*(n)$.\footnote{
For example
$\lceil \log \lceil \log \lceil \log(n) \rceil \rceil \rceil
\leq \lceil \log \lceil \log (2 + \log(n) ) \rceil \rceil
\leq \lceil \log \lceil \log (2 \log(n) ) \rceil \rceil
= \lceil \log \lceil 1 + \log\log(n) \rceil \rceil
\leq \lceil \log (2 + \log\log(n)) \rceil
\leq \lceil \log ( 2 \log\log(n) ) \rceil
= \lceil 1 + \log\log\log(n) \rceil
\leq 2 + \log\log\log(n)$.}

\begin{lemma}\label{lem:recUtility}
Let $Q:X^*\times[0,T]\rightarrow\R$ be a sensitivity-1 quality function, and let $S\in X^*$ be a database s.t. $Q(S,\cdot)$ is quasi-concave.
Let $\alpha\leq\frac{1}{2}$ and let $\beta,\epsilon,\delta,r,N$ be s.t. 
$$\max_{i\in[0,T]}\{Q(S,i)\}\geq r \geq
8^N \cdot \frac{4}{\alpha\epsilon} \left\{ \log\Big(\frac{32}{\beta\delta}\Big) + \Nlog{N}(T) \right\}.$$
When executed on $S,[0,T],r,\alpha,\epsilon,\delta,N$, algorithm $RecConcave$ fails to outputs an index $j$ s.t.
$Q(S,j)\geq(1-\alpha)r$ with probability at most $2\beta N$.
\end{lemma}

\begin{proof}
The proof is by induction on the number of recursive calls, denoted as $t$.
For $t=1$ (i.e., $T\leq32$ or $N=1$), the exponential mechanism ensures that for $r\geq\frac{2}{\alpha\epsilon}\log(\frac{T}{\beta})$, the probability of algorithm $RecConcave$ failing to output a $j$ s.t. $Q(S,j)\geq(1-\alpha)r$ is at most $\beta$.

Assume that the stated lemma holds whenever algorithm $RecConcave$ performs at most $t-1$ recursive calls, and let $S,[0,T],r,\alpha,\epsilon,\delta,N$ be inputs (satisfying the conditions of Lemma~\ref{lem:recUtility}) on which algorithm $RecConcave$ preforms $t$ recursive calls.
Consider the first call in the execution of $RecConcave$ on those inputs, and denote by $T'$ the smallest power of 2 s.t. $T'\geq T$.
In order to apply the inductive assumption, we need to show that for the recursive call in step~6, all the conditions of Lemma~\ref{lem:recUtility} hold.

We first note that by Claim~\ref{claim:qConcave}, the quality function $q(S,\cdot)$ defined of step~4 is quasi-concave.
We next show that the recursive call is preformed with an appropriate quality promise $R=\frac{\alpha}{2}r$. 
The conditions of the lemma ensure that $L(S,0)\geq r$, and, by definition, we have that $L(S,\log(T')+1)\leq0$. There exists therefore a $j\in[0,\log(T')]$ for which $L(S,j)\geq(1-\frac{\alpha}{2})r$, and $L(S,j+1)<(1-\frac{\alpha}{2})r$. Plugging these inequalities in the definition of $q(S,j)$ we get that $q(S,j)\geq\frac{\alpha}{2}r$. Therefore, there exists an index $j\in[0,\log(T')]$ with quality $q(S,j)\geq R$.
Moreover, the recursive call of step~6 executes $RecConcave$ on the range $[0,\log(T')]=[0,\lceil\log(T)\rceil]$ with $(N-1)$ as the bound on the recursion depth, with $\widetilde{\alpha}\triangleq\frac{1}{4}$ as the approximation parameter, and with a quality promise $R$ satisfying
\begin{eqnarray*}
R &=& \frac{\alpha}{2}r\\
&\geq& \frac{\alpha}{2} \cdot 8^N \cdot \frac{4}{\alpha\epsilon} \left\{ \log\Big(\frac{32}{\beta\delta}\Big) + \Nlog{N}(T) \right\}\\
&=& 8^{N-1} \cdot \frac{4}{\widetilde{\alpha} \epsilon} \left\{ \log\Big(\frac{32}{\beta\delta}\Big) + \Nlog{N-1}\lceil\log(T)\rceil \right\}\\
\end{eqnarray*}

We next show that w.h.p.\ at least one of the two intervals $A,B$ chosen on Step~9, contains a lot of points with high score.
Denote the index returned by the recursive call of step~6 as $k$.
By the inductive assumption, with probability at least $(1-2\beta(N-1))$, the index $k$ is s.t. $q(S,k)\geq(1-\frac{1}{4})R=\frac{3\alpha}{8}r$; we proceed with the analysis assuming that this event happened.
By the definition of $q(S,k)$, this means that $L(S,k)\geq q(S,k)+(1-\alpha)r\geq(1-\frac{5\alpha}{8})r$ and that  $L(S,k+1)\leq r-q(S,k)\leq(1-\frac{3\alpha}{8})r$.
That is, there exists an interval $G$ of length $2^k$ s.t. $\forall i\in G$ it holds that $Q(S,i)\geq(1-\frac{5\alpha}{8})r$, and every interval of length $2\cdot2^k$ contains at least one point $i$ s.t. $Q(S,i)\leq(1-\frac{3\alpha}{8})r$.

As promised by the conditions of the lemma, there exists a point $p\in[0,T]$ with quality $Q(S,p)\geq r$. Consider the following two intervals: $P_1=[p-2\cdot2^k+1,p]$ and $P_2=[p,p+2\cdot2^k-1]$, and denote $P=P_1\cup P_2$ (these two intervals might be trimmed if $p$ is close to the edges of $[0,T]$). Assuming $P_1,P_2$ are not trimmed, they both are intervals of length $2\cdot2^k$, and, therefore, each of them contains a point $i_1,i_2$ respectively with quality $Q(S,i_1),Q(S,i_2)\leq(1-\frac{3\alpha}{8})r$. Therefore, by the quasi-concavity of $Q(S,\cdot)$, every point $\ell\geq i_2$ and every point $\ell\leq i_1$ must have quality at most $Q(S,\ell)\leq(1-\frac{3\alpha}{8})r$ (otherwise, by the quasi-concavity of $Q(S,\cdot)$, every point between $\ell$ and $p$ must have quality strictly greater than $(1-\frac{3\alpha}{8})r$, contradicting the quality bound on $i_1,i_2$). See Figure~\ref{fig:4}.

\captionsetup[figure]{skip=-7pt}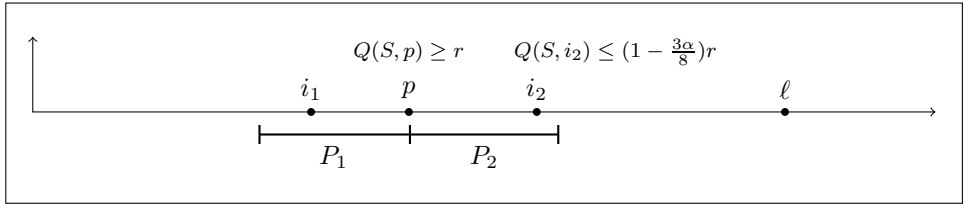
\begin{figure}[H]
\begin{center}
\begin{tikzpicture}[xscale=1,yscale=1,show background rectangle,inner frame sep=10pt]
\draw [<->] (0,1) -- (0,0) -- (12,0);
\draw [|-,thick] (3,-0.3) -- (5,-0.3) ;
\draw [|-|,thick] (5,-0.3) -- (7,-0.3) ;
\node at (4,-0.6) {$P_1$};
\node at (6,-0.6) {$P_2$};
\node[fill,circle,inner sep=0pt,minimum size=3pt] (a1) at (6.7,0) {};
\node[fill,circle,inner sep=0pt,minimum size=3pt] (a1) at (5,0) {};
\node[fill,circle,inner sep=0pt,minimum size=3pt] (a1) at (3.7,0) {};
\node[fill,circle,inner sep=0pt,minimum size=3pt] (a1) at (10,0) {};

\node at (5,0.3) {$p$};
\node at (5,0.8) {\footnotesize $Q(S,p)\geq r$};
\node at (6.7,0.3) {$i_2$};
\node at (7.75,0.8) {\footnotesize $Q(S,i_2)\leq(1-\frac{3\alpha}{8})r$};
\node at (3.7,0.3) {$i_1$};
\node at (10,0.3) {$\ell$};
\end{tikzpicture}
\end{center}
\caption{A point $\ell\notin P$ cannot have quality greater than $(1-\frac{3\alpha}{8})r$.}\label{fig:4}
\end{figure}

Note that if $P_1$ (or $P_2$) is trimmed, then there are no points on the left of (or on the right of) $P$.
So, the interval $P$ contains the point $p$ with quality $Q(S,p)\geq r$ and every point $i\in[0,T]\setminus P$ has quality of at most $(1-\frac{3\alpha}{8})r$. Moreover, $P$ is of length $4\cdot2^k-1$. As the intervals of the partitions $\{A_i\}$ and $\{B_i\}$ are of length $8\cdot 2^k$, and the $\{B_i\}$'s are shifted by $4\cdot 2^k$, there must exist an interval $C\in\{A_i\}\cup\{B_i\}$ s.t. $P\subseteq C$. Assume without loss of generality that $C\in\{A_i\}$.

Recall that the quality $u(S,\cdot)$ of an interval $I$ is defined as the maximal quality $Q(S,i)$ of a point $i\in I$. Therefore, as $p\in C$, the quality of $C$ is at least $r$. On the other hand, the quality of every $A_i\neq C$ is at most $(1-\frac{3\alpha}{8})r$. That is, the interval $C$ scores better (under $u$) than any other interval in $\{A_i\}$ by at least an additive factor of $\frac{3\alpha}{8}r\geq\frac{1}{\epsilon}\log(\frac{1}{\beta\delta})$. By the properties of $\AAA_{\rm dist}$, with probability at least $(1-\beta)$, the chosen interval $A$ in step~9 is s.t. $P\subseteq A$. We proceed with the analysis assuming that this is the case.

Consider again the interval $P$ containing the point $p$, and recall that there exists an interval $G$ of length $2^k$ containing only points with quality $Q(S,\cdot)$ of at least $(1-\frac{5\alpha}{8})r$. Such an interval must be contained in $P$. Otherwise, by the quasi-concavity of $Q(S,\cdot)$, all the points between $G$ and the point $p$ must also have quality at least $(1-\frac{5\alpha}{8})r$, and, in particular, $P$ must indeed contain such an interval.

So, the chosen interval $A$ in step~9 is of length $8\cdot2^k$, and it contains a sub interval of length $2^k$ in which every point has quality at least $(1-\frac{5\alpha}{8})r$. 
That is, at least $\frac{1}{16}$ out of the points in $(A\cup B)$ has quality at least $(1-\frac{5\alpha}{8})r$. Therefore, as $r\geq\frac{4}{\alpha\epsilon}\log(\frac{16}{\beta})$, the exponential mechanism ensures that the probability of step~10 failing to return a point $h\in (A\cup B)$ with $Q(S,h)\geq(1-\alpha)r$ is at most $\beta$.\footnote{
As there are at least $\frac{1}{16}|A\cup B|$ solutions with quality at least $(1-\frac{5\alpha}{8})r$, the probability that the exponential mechanism outputs a specific solution $h\in (A\cup B)$ with $Q(S,h)\geq(1-\alpha)r$ is at most
$\frac{\exp(\frac{\epsilon}{2}(1-\alpha)r)}{\frac{1}{16}|A\cup B|\exp(\frac{\epsilon}{2}(1-\frac{5\alpha}{8})r)}$. Hence, the probability that the exponential mechanism outputs {\em any} solution $h\in (A\cup B)$ with $Q(S,h)\geq(1-\alpha)r$ is at most
$16\frac{\exp(\frac{\epsilon}{2}(1-\alpha)r)}{\exp(\frac{\epsilon}{2}(1-\frac{5\alpha}{8})r)}$, which is at most $\beta$ for our choice of $r$.}

All in all, with probability at least $(1-2\beta(N-1)-2\beta)=(1-2\beta N)$, algorithm $RecConcave$ returns an index $j\in[0,T]$ s.t. $Q(S,j)\geq(1-\alpha)r$.
\end{proof}

Combining Lemma~\ref{lem:recPrivacy} and Lemma~\ref{lem:recUtility} we get the following theorem.

\begin{theorem}\label{thm:rec}

Let algorithm $RecConcave$ be executed on a range $[0,T]$, a sensitivity-1 quality function $Q$, a database $S$, a bound on the recursion depth $N$, privacy parameters $\frac{\epsilon}{3N},\frac{\delta}{3N}$, approximation parameter $\alpha$, and a quality promise $r$. The following two statements hold:
\begin{enumerate}
	\item Algorithm $RecConcave$ preserves $(\epsilon,\delta)$-differential privacy.
	\item If $S$ is s.t. $Q(S,\cdot)$ is quasi-concave, and if 
	\begin{eqnarray}\label{eq:recFinal}
\max_{i\in[0,T]}\{Q(S,i)\}\geq r \geq
8^N \cdot \frac{36 N}{\alpha\epsilon} \left\{ \log\Big(\frac{6N}{\beta\delta}\Big) + \underbrace{\log\log\cdots\log}_{N \text{ times}}(T) \right\}
\end{eqnarray}
then algorithm $RecConcave$ fails to outputs an index $j$ s.t.
$Q(S,j)\geq(1-\alpha)r$ with probability at most $\beta$.
\end{enumerate}
\end{theorem}

\begin{remark}
Recall that the number of recursive calls on a range $[0,T]$ is always bounded by $\log^*(T)$, and note that for $N=\log^*(T)$ we have that $\Nlog{N}(T)\leq1$. Therefore, the promise requirement in Inequality~(\ref{eq:recFinal}) can be replaced with
$8^{\log^*(T)} \cdot \frac{36 \log^*(T)}{\alpha\epsilon} \log\Big(\frac{12\log^*(T)}{\beta\delta}\Big)$.
\end{remark}

\begin{remark}
The computational efficiency of algorithm $RecConcave$ depends on the quality function $Q(\cdot,\cdot)$. Note, however, that it suffices to efficiently implement the top level call (i.e., without the recursion). This is true because an iteration of algorithm $RecConcave$, operating on a range $[0,T]$, can easily be implemented in time $\poly(T)$, and the range given as input to recursive calls is logarithmic in the size of the initial range.
\end{remark}

\subsection{A Proper $(\epsilon,\delta)$-Private Learner For $\thresh_d$}\label{sec:interval}
As we will now see, algorithm $RecConcave$ can be used as a proper $(\alpha,\beta,\epsilon,\delta,m)$-private learner for $\thresh_d$. 
Recall Example~\ref{example:intervalConcave} (showing that the goal of choosing a hypothesis with small empirical error can be viewed as a quasi-concave promise problem), and consider the following algorithm.
\begin{center}
\noindent\fbox{
\parbox{.95\columnwidth}{
{\bf Algorithm $LearnThresholds$}\\
{\bf Input:} A labeled sample $S=(x_i,y_i)_{i=1}^m$ and parameters $\alpha,\epsilon,\delta,N$.
\begin{enumerate}[rightmargin=10pt]
  \item Denote $\hat{\alpha}=\frac{\alpha}{2},\;\;
  \hat{\epsilon}=\frac{\epsilon}{3N}$,
  and $\hat{\delta}=\frac{\delta}{3N}$.
	\item For every $0\leq j\leq 2^d$, define $Q(S,j)=|\{ i : c_j(x_i)=y_i \}|$.
	\item Execute algorithm $RecConcave$ on the sample $S$, the range $[0,2^d]$, the quality function $Q(\cdot,\cdot)$, the promise $m$, and parameters $\hat{\alpha},\hat{\epsilon},\hat{\delta},N$. Denote the returned value as $k$.
	\item Return $c_k$.
\end{enumerate}
}}
\end{center}

\begin{theorem}\label{thm:intervalLearner}
For every $1\leq N\leq\log^*(2^d)$, Algorithm $LearnThresholds$ is an efficient proper $(\alpha,\beta,\epsilon,\delta,m)$-PPAC learner for $\thresh_d$,
where the sample size is
$$m=O\left(
 \frac{8^N \cdot N}{\alpha\cdot\min\{\alpha,\epsilon\}} \left\{ \log\Big(\frac{N}{\alpha\beta\delta}\Big) + \underbrace{\log\log\cdots\log}_{N \text{ times}}(2^d) \right\}
\right).$$
\end{theorem}

\begin{proof}
By Theorem~\ref{thm:rec}, algorithm $LearnThresholds$ is $(\epsilon,\delta)$-differentially private. For the utility analysis, fix a target concept $c_j\in\thresh_d$, and a distribution $\DDD$ on $X_d$, and let $S$ be a sample drawn i.i.d. from $\DDD$ and labeled by $c_j$.
Define the following two good events:
\begin{enumerate}[label=$E_{\arabic*}:$]
\item $\forall \; h\in\thresh_d,\;\; |\error_\DDD(h,c_j)-\error_S(h)|\leq\frac{\alpha}{2}$.
\item Algorithm $RecConcave$ returns $k$ s.t. $\error_S(c_k)\leq\frac{\alpha}{2}$
\end{enumerate}

Clearly, when both $E_1,E_2$ occur, algorithm $LearnThresholds$ succeeds in outputting an $\alpha$-good hypothesis for $c_j$ and $\DDD$.
Note that as $\VC(\thresh_d)=1$, Theorem~\ref{thm:generalization} ensures that for $m\geq\frac{200}{\alpha^2}\ln(\frac{4}{\alpha\beta})$, event $E_1$ happens with probability at least $(1-\frac{\beta}{2})$.

Next, note that for the target concept $c_j$ it holds that $Q(S,j)=m$, and algorithm $RecConcave$ is executed on step~3 with a valid quality promise.
Moreover, as shown in Example~\ref{example:intervalConcave}, algorithm $RecConcave$ is executed with a quasi-concave quality function.

So, algorithm $RecConcave$ is executed on step~3 with a valid quality promise and with a quasi-concave quality function. For
$$m\geq 8^N \cdot \frac{72 N}{\alpha\epsilon} \left\{ \log\Big(\frac{12N}{\beta\delta}\Big) + \underbrace{\log\log\cdots\log}_{N \text{ times}}(T) \right\},$$
algorithm $RecConcave$ ensures that with probability at least $(1-\frac{\beta}{2})$, the index $k$ at step~2 is s.t. $Q(k)\geq(1-\frac{\alpha}{2})m$. The empirical error of $c_k$ is at most $\frac{\alpha}{2}$ in such a case. Therefore, Event $E_2$ happens with probability at least $(1-\frac{\beta}{2})$. Overall, we conclude that $LearnThresholds$ is a proper $(\alpha,\beta,\epsilon,\delta,m)$-PPAC learner for $C$, where
$$m\geq\max\left\{ \; \frac{200}{\alpha^2}\ln(\frac{4}{\alpha\beta}) \; , \; \frac{8^N 72 N}{\alpha\epsilon} \left( \log\Big(\frac{12N}{\beta\delta}\Big) + \underbrace{\log\log\cdots\log}_{N \text{ times}}(2^d) \right) \; \right\}.$$

\end{proof}

\begin{remark}
By using $N=\log^*(2^d)$ in the above theorem, we can bound the sample complexity of $LearnThresholds$ by 
$$m=O\left(
 \frac{8^{\log^*(d)} \cdot \log^*(d)}{\alpha\cdot\min\{\alpha,\epsilon\}} \log\Big(\frac{\log^*(d)}{\alpha\beta\delta}\Big)
\right).$$
\end{remark}

\subsection{Axis-Aligned Rectangles in High Dimension}

Consider the class of all axis-aligned rectangles (or hyperrectangles) in the Euclidean space $\R^n$.
A concept in this class could be thought of as the product of $n$ intervals, one on each axis. We briefly describe an efficient approximate-private proper-learner for a discrete version of this class.

Formally,
\begin{definition}
Let $X_d^n=(\{0,1\}^d)^n$ denote a discrete $n$-dimensional domain, in which every axis consists of $2^d$ points $\{0,1,\ldots,2^d-1\}$.
For every $\vec{a}=(a_1,\ldots,a_n),\vec{b}=(b_1,\ldots,b_n)\in X_d^n$ define the concept $c_{[\vec{a},\vec{b}]}:X_d^n\rightarrow\{0,1\}$ where $c_{[\vec{a},\vec{b}]}(\vec{x})=1$ if and only if for every $1\leq i\leq n$ it holds that $a_i\leq x_i\leq b_i$. Define the concept class of all axis-aligned rectangles over $X^n_d$ as $\rectangle_d^n=\{ c_{[\vec{a},\vec{b}]} \}_{\vec{a},\vec{b}\in X_d^n}$.
\end{definition}
The VC dimension of this class is $2n$, and, thus, it can be learned non-privately with sample complexity $O_{\alpha,\beta}(n)$. Note that $|\rectangle_d^n|=2^{O(nd)}$, and, therefore, the generic construction of Kasiviswanathan et al.~\cite{KLNRS08} yields an inefficient proper $\epsilon$-private learner for this class with sample complexity $O_{\alpha,\beta,\epsilon}(nd)$.

In~\cite{Kearns98}, Kearns gave an efficient (noise resistant) non-private learner for this class. The learning model there was a variant of the statistical queries model~\cite{Kearns98}, in which the learner is also being given access to the underling distribution $\DDD$. Every learning algorithm in the statistical queries model can be transformed to satisfy differential privacy while preserving efficiency~\cite{BDMN05,KLNRS08}. However, as Kearns' algorithm assumes direct access to $\DDD$, this transformation cannot be applied directly.

Kearns' algorithm begins by sampling $\DDD$ and using the drawn samples to divide each axis $i\in[n]$ into $O(n/\alpha)$ intervals ${\cal I}_i = \{I\}$ with the property that the $x_i$ component of a random point from $\DDD$ is approximately equally likely to fall into each of the intervals in ${\cal I}_i$. 
The algorithm proceeds by estimating the boundary of the target rectangle separately for every dimension $i$: For every interval $I\in{\cal I}_i$, the algorithm uses statistical queries to estimate the probability that a positively labeled input has its $x_i$ component in $I$, i.e., 
$$p_I = \Pr_{x\sim\DDD}\big[ \; (\mbox{$x$ is labeled $1$}) \; \wedge \; (x_i\in I) \; \big].$$
The algorithm places the left boundary of the hypothesis rectangle in the $i$-th dimension at the left-most interval $I\in{\cal I}_i$ such that $p_I$ is significant, and analogously on the right.

Note that once the interval sets ${\cal I}_i$ are defined for each axis $i\in [n]$, estimating every single $p_I$ can be done via statistical queries, and can, therefore, be made private using the transformation of~\cite{BDMN05,KLNRS08}.
Alternatively, estimating (simultaneously) all of the $p_I$'s (on the $i^{\text th}$ axis) could be done privately using the laplacian mechanism. This use of the laplacian mechanism is known as a histogram (see Theorem~\ref{thm:lap}).

Thus, our task is to privately partition each axis. The straight forward approach for privately finding ${\cal I}_i$ is by a noisy binary search for the boundary of each of the $n/\alpha$ intervals (in each axis). This would result in $\Omega(d)$ noisy comparisons, which, in turn, results in a private learner with a high sample complexity.\\

We now overcome this issue using a sanitizer for $\thresh_d$. Such a sanitizer will be constructed in Section~\ref{sec:sanIntervals}; here we use it for privately finding ${\cal I}_i$.

\begin{theorem}[Restatement of Theorem~\ref{thm:sanIntervals}]
Fix $\alpha,\beta,\epsilon,\delta$. There exists an efficient $(\alpha,\beta,\epsilon,\delta,m)$-sanitizer for $\thresh_d$, where
$ m= \widetilde{O}_{\beta,\epsilon,\delta}\left(\frac{1}{\alpha^{2.5}}\cdot 8^{\log^*(d)} \right)$.
\end{theorem}

As we next explain, such a sanitizer can be used to (privately) divide the axes.
Given an interval $[a,b]\subseteq X_d$ and a sample $S$, we denote the probability mass of $[a,b]$ under $\DDD$ as $\DDD[a,b]$, and the number of sample points in this interval as $\#_S[a,b]$. Standard arguments in learning theory (specifically, Theorem~\ref{thm:generalization}) state that for a large enough sample (whose size is bigger than the VC dimensions of the intervals class) w.h.p. $\frac{1}{|S|}\#_S[a,b]\approx\DDD[a,b]$ for {\em every} interval $[a,b]\subseteq X_d$.

On an input database $S\in(X_d)^*$, such a sanitizer for $\thresh_d$ outputs an alternative database $\hat{S}\in(X_d)^*$ s.t. 
$\frac{1}{|\hat{S}|}\#_{\hat{S}}[0,b]\approx\frac{1}{|S|}\#_S[0,b]$ for every interval $[0,b]\subseteq X_d$.
Hence, for every interval $[a,b]\subseteq X_d$ we have that
\begin{eqnarray*}
\frac{1}{|\hat{S}|}\#_{\hat{S}}[a,b] &=& \frac{1}{|\hat{S}|}\#_{\hat{S}}[0,b] - \frac{1}{|\hat{S}|}\#_{\hat{S}}[0,a-1]\\
&\approx& \frac{1}{|S|}\#_S[0,b] - \frac{1}{|S|}\#_S[0,a-1]\\
&=& \frac{1}{|S|}\#_S[a,b]\\
&\approx& \DDD[a,b].
\end{eqnarray*}

So, in order to divide the $i^{\text th}$ axis we apply the above mentioned sanitizer, and divide the axis using the returned sanitized database. In order to accumulate error of up to $\alpha/n$ on each axis (as required by Kearns' algorithm), we need to execute the above mentioned sanitizer with an approximation parameter of (roughly) $\alpha/n$. Every such execution requires, therefore, a sample of $\widetilde{O}_{\alpha,\beta,\epsilon,\delta}\left(n^{2.5}\cdot 8^{\log^*(d)} \right)$ elements. As there are $n$ such executions (one for each axis), using Theorem~\ref{thm:composition2} (composition theorem), the described learner is of sample complexity $\widetilde{O}_{\alpha,\beta,\epsilon,\delta}\left(n^{3}\cdot 8^{\log^*(d)} \right)$.

\begin{theorem}
There exists an efficient $(\alpha,\beta,\epsilon,\delta,m)$-PPAC proper-learner for $\rectangle_d^n$, where
$$ m= O\left(\frac{n^3}{\alpha^{2.5} \epsilon}\cdot 8^{\log^*(d)} \cdot \log^*(d) \cdot \log\left(\frac{n}{\alpha\delta}\right) \cdot \log\left( \frac{n\cdot\log^*(d)}{\alpha\beta\epsilon\delta}\right)\right).$$
\end{theorem}

This should be contrasted with $\theta_{\alpha,\beta}(n)$, which is the non-private sample complexity for this class (as the $\VC$-dimension of $\rectangle_d^n$ is $2n$), and with $\theta_{\alpha,\beta,\epsilon}(nd)$ which is the pure-private sample complexity for this class.\footnote{The general construction of Kasiviswanathan et al.~\cite{KLNRS08} yields an (inefficient) pure-private proper-learner for this class with sample complexity $O_{\alpha,\beta,\epsilon}(nd)$.
Feldman and Xiao~\cite{FX14} showed that this is in fact optimal, and every $\epsilon$-private (proper or improper) learner for this class must have sample complexity $\Omega(nd)$.}

\section{Sanitization with Approximate Privacy}

In this section we present $(\epsilon,\delta)$-private sanitizers for several concept classes, and separate the database size necessary for $(\epsilon,0)$-private sanitizers from the database size sufficient for $(\epsilon,\delta)$-private sanitizers.

\subsection{The Choosing Mechanism}
Recall that in our private PAC learner for $\point_d$, given a typical labeled sample, there exists a unique concept in the class that stands out (we used algorithm $\AAA_{\rm dist}$ to identify it). This is not the case in the context of sanitization, as a given database $S$ can have many $\alpha$-close sanitized databases $\hat{S}$. We will overcome this issue by using the following private tool for approximating a restricted class of choosing problems. 

A function  $q:X^*\times \FFF \rightarrow \N$ defines an {\em optimization problem} over the domain $X$ and solution set $\FFF$:  Given a dataset $\db$ over domain $X$ choose $f\in\FFF$ that (approximately) maximizes $q(\db,f)$. We are interested in a subset of these optimization problems, which we call {\em bounded-growth choice problems}. In this section we consider a database $S\subseteq X^*$ as a multiset.

\begin{definition}
Given $q$ and $\db$ define $\opt_q(\db) = \max_{f\in\FFF}\{q(\db,f)\}$. A solution $f\in\FFF$ is called {\em $\alpha$-good} for a database $\db$ if $q(\db,f)\geq \opt_q(\db) -\alpha |\db|$. 
\end{definition}

\begin{definition}
A quality function $q:X^*\times \FFF \rightarrow \N$ is {\em $k$-bounded-growth} if:
\begin{enumerate}[topsep=-5pt,itemsep=-2pt]
\item $q(\emptyset,f)=0$ for all $f\in\FFF$.
\item If $S_2 = S_1 \cup \{x\}$, then (i)  $q(S_1,f)+1\geq q(S_2,f) \geq q(S_1,f)$ for all $f\in\FFF$; and (ii) there are at most $k$ solutions $f\in\FFF$ s.t. $q(S_2,f) > q(S_1,f)$. 
\end{enumerate}
\end{definition}

In words, the second requirement means that (i) Adding an element to the database could either have no effect on the score of a solution $f$, or can increase the score by exactly $1$; and (ii) There could be at most $k$ solutions whose scores are increased (by $1$).
Note that a $k$-bounded-growth quality function is, in particular, a sensitivity-1 function as two neighboring $S,S'$ must be of the form $D\cup\{x_1\}$ and $D\cup\{x_2\}$ respectively. Hence, $q(S,f)-q(S',f)\leq q(D,f)+1-q(D,f)=1$ for every solution $f$.

\begin{example}\label{example:pointBounded}
As an example of a 1-bounded growth quality function, consider the following $q:X^*\times X \rightarrow\N$. Given a database $S=(x_1,\ldots,x_m)$ containing elements from some domain $X$, define $q(S,a)=\big|\{i:x_i=a\}\big|$. That is, $q(S,a)$ is the number of appearances of $a$ in $S$.
Clearly, $q(\emptyset,f)=0$ for all $f\in X$. Moreover, adding an element $a\in X$ to a database $S$ increases by 1 the quality of $q(S,a)$, and does not effect the quality of every other $b\neq a$.
\end{example}

The choosing mechanism (in Figure~\ref{fig:choosing}) is a private algorithm for approximately solving bounded-growth choice problems. Step~1 of the algorithm checks whether a good solutions exist, as otherwise any solution is approximately optimal (and the mechanism returns $\bot$). Step~2 invokes the exponential mechanism, but with the {\em small} set $G(S)$ instead of $\FFF$.

\begin{figure}[H]
\begin{center}
\noindent\fbox{
\parbox{.95\columnwidth}{
{\bf Choosing Mechanism} \\
{\bf Input:} a database $S$, a quality function $q$, and parameters $\alpha,\beta,\epsilon,\delta$.
\begin{enumerate}[topsep=-1pt, rightmargin=10pt]
\item Set ${\rm best}(S)=\max_{f\in\FFF} \left\{ q(S,f) \right\}+\Lap(\frac{4}{\epsilon})$. If ${\rm best}(S)<\frac{\alpha m}{2}$ then halt and return $\bot$.
\item Let $G(S)=\{f\in\FFF : q(S,f)\geq1\}$. Choose and return $f\in G(S)$ using the exponential mechanism with parameter $\frac{\epsilon}{2}$.\\ 
\end{enumerate}
}}
\end{center}
\caption{The choosing mechanism. \label{fig:choosing}}
\end{figure}

\begin{lemma} 
When $q$ is a $k$-bounded-growth quality function, the choosing mechanism preserves $(\epsilon,\delta)$-differential privacy for databases of
$m\geq \frac{16}{\alpha\epsilon}\ln(\frac{16 k}{\alpha\beta\epsilon\delta})$ elements.
\end{lemma}
\begin{proof}
Let $S,S'$ be neighboring databases of $m$ elements. We need to show that $\Pr[A(S)\in R]\leq \exp(\epsilon)\cdot\Pr[A(S')\in R]+\delta$ for any set of outputs $R$. Note first that by the properties of the Laplace Mechanism,
\begin{eqnarray}\label{eq:CMbot}
\Pr[A(S)=\bot] &=& \Pr\left[{\rm best}(S)<\frac{\alpha m}{2}\right] \nonumber\\
&\leq& \exp(\frac{\epsilon}{4})\cdot\Pr\left[{\rm best}(S')<\frac{\alpha m}{2}\right] \nonumber\\
&=& \exp(\frac{\epsilon}{4})\cdot\Pr[A(S')=\bot].
\end{eqnarray}

\paragraph{Case~(a): $q(S,f)<\frac{\alpha m}{4}$ for all $f$.} Using $m\geq\frac{16}{\alpha\epsilon}\ln(\frac{1}{2\delta})$, we get that
$$\Pr[A(S)\neq\bot]\leq
\Pr\left[\Lap(\frac{4}{\epsilon})>\frac{\alpha m}{4}\right]=
\frac{1}{2}\exp(-\frac{\epsilon}{4}\frac{\alpha m}{4}) \leq \delta.$$
Hence, for every set of outputs $R$
\begin{eqnarray*}
\Pr[A(S)\in R] & \leq & \mathds{1}_{(\bot\in R)}\cdot\Pr[A(S)=\bot]+\Pr[A(S)\neq\bot] \\
& \leq & \mathds{1}_{(\bot\in R)}\cdot \exp(\frac{\epsilon}{4})\cdot\Pr[A(S')=\bot]+\delta \\
& \leq & \exp(\frac{\epsilon}{4}) \cdot \Pr[A(S')\in R]+\delta.
\end{eqnarray*}

\paragraph{Case~(b): There exists $\hat f$ s.t. $q(S,\hat f)\geq\frac{\alpha m}{4}$.} Let $G(S)$ and $G(S')$ be the sets used in step~2 in the execution $S$ and on $S'$ respectively.
We will show that the following two facts hold:\\

\noindent
$Fact \; 1:$ For every $f\in G(S)\setminus G(S')$, it holds that $\Pr[A(S)=f]\leq \frac{\delta}{k}$.\\

\noindent
$Fact \; 2:$ For every possible output $ f \notin G(S)\setminus G(S')$, it holds that $\Pr[A(S)=f]\leq e^{\epsilon}\Pr[A(S')=f]$.\\

We first show that the two facts imply that the lemma holds for Case~(b).
Let $B \triangleq G(S)\setminus G(S')$, and note that as $q$ is $k$-growth-bounded, $|B|\leq k$.
Denote $B=\{ b_1,\ldots,b_\ell \}$, where $\ell\leq k$.
Using the above two facts, for every set of outputs $R$ we have
\begin{eqnarray*}
\Pr[A(S)\in R] &=&  \Pr[A(S)\in R\setminus B] + \sum_{i:\;b_i\in R}\Pr[A(S)=b_i]\\
&\leq& e^\epsilon \Pr[A(S')\in R\setminus B] + \sum_{i:\;b_i\in R}\frac{\delta}{k}\\
&\leq& e^\epsilon \Pr[A(S')\in R] + \delta.\\
\end{eqnarray*}

For proving Fact~1, let $f\in G(S)\setminus G(S')$.
That is, $q(S,f)\geq1$ and $q(S',f)=0$. As $q$ is (in particular) a sensitivity-1 function, 
it must be, therefore, that $q(S,f)=1$. As there exists $\hat f\in S$ with $q(S,\hat f)\geq\frac{\alpha m}{4}$, we have that
$$
\Pr[A(S)=f]\leq
\Pr\left[  \begin{array}{c}
	\text{The exponential}\\
	\text{mechanism chooses $f$}
\end{array} \right]
\leq \frac{\exp(\frac{\epsilon}{4})}{\exp(\frac{\epsilon}{4}\frac{\alpha m}{4})},
$$
which is at most $\frac{\delta}{k}$ for $m\geq\frac{16}{\alpha\epsilon}(\frac{\epsilon}{4}+\ln(\frac{k}{\delta}))$.

For proving Fact~2, let $f \notin G(S)\setminus G(S')$ be a possible output of $A(S)$.
If $f\notin(G(S)\cup\{\bot\})$ then trivially $\Pr[A(S)=f]=0\leq e^\epsilon \Pr[A(S')=f]$.
We have already established (in Inequality~(\ref{eq:CMbot})) that for $f=\bot$ it holds that $\Pr[A(S)=\bot]\leq e^{\epsilon/4} \Pr[A(S')=\bot]$.
It remains, hence, to deal with the case where $f\in G(S)\cap G(S')$. For this case, we use the following Fact~3, proved below.\\

\noindent
$Fact \; 3:$ $\sum\limits_{h\in G(S')}\exp(\frac{\epsilon}{4}q(S',h))\leq e^{\epsilon/2}\cdot\sum\limits_{h\in G(S)}\exp(\frac{\epsilon}{4}q(S,h))$.\\

\noindent
Using Fact~3, for every possible output $f\in G(S)\cap G(S')$ we have that
\begin{eqnarray*}
&&\frac{\Pr[A(S)=f]}{\Pr[A(S')=f]}\\
&& \;\;\;\;\;\; =  \;\;\;
\left( 
\Pr[A(S)\neq\bot]
\frac{\exp(\frac{\epsilon}{4}q(f,S))}
{\sum_{h\in G(S)}\exp( \frac{\epsilon}{4} q(h,S) )}\right) / \left( 
\Pr[A(S')\neq\bot]
\frac{\exp(\frac{\epsilon}{4}q(f,S'))}
{\sum_{h\in G(S')}\exp( \frac{\epsilon}{4} q(h,S') )}\right) \\
&& \;\;\;\;\;\; = \;\;\;
\frac{\Pr[A(S)\neq\bot]}{\Pr[A(S')\neq\bot]}\cdot
\frac{\exp(\frac{\epsilon}{4}q(f,S))\cdot\sum_{h\in G(S')}{\exp( \frac{\epsilon}{4} q(h,S') )}}
{\exp(\frac{\epsilon}{4}q(f,S'))\cdot\sum_{h\in G(S)}{\exp( \frac{\epsilon}{4} q(h,S) )}}\leq
e^\frac{\epsilon}{4}\cdot e^\frac{\epsilon}{4}\cdot e^\frac{\epsilon}{2}=e^\epsilon.
\end{eqnarray*}

We now prove Fact~3. Denote $\XXX\triangleq \sum\limits_{h\in G(S)}\exp(\frac{\epsilon}{4}q(S,h))$. We first show that 
\begin{eqnarray}\label{eqn:goal}
k\cdot e^{\epsilon/4}+e^{\epsilon/4}\cdot\XXX\leq e^{\epsilon/2}\XXX.
\end{eqnarray}
That is, we need to show that $\XXX\geq \frac{k}{e^{\epsilon/4}-1}$. As $1+\frac{\epsilon}{4}\leq e^{\epsilon/4}$, it suffices to show that $\XXX\geq\frac{4k}{\epsilon}$. Recall that there exists a solution $\hat{f}$ s.t. $q(S,\hat{f})\geq\frac{\alpha m}{4}$. Therefore,
$\XXX\geq\exp(\frac{\epsilon}{4}\frac{\alpha m}{4})$, which is at least $\frac{4k}{\epsilon}$ for $m\geq\frac{16}{\alpha \epsilon}\ln(\frac{4k}{\epsilon})$. This proves~(\ref{eqn:goal}).\\

Now, recall that as $q$ is $k$-growth-bounded, for every $h\in\FFF$ it holds that $|q(S,h)-q(S',h)|\leq1$.
Moreover, $|G(S')\setminus G(S)|\leq k$, and every $h\in(G(S')\setminus G(S))$ obeys $q(S',h)=1$. Hence,
\begin{eqnarray*}
\sum_{h\in G(S')}\exp(\frac{\epsilon}{4}q(S',h)) &\leq&  k\cdot\exp(\frac{\epsilon}{4})+ \sum_{h\in G(S')\cap G(S)}\exp(\frac{\epsilon}{4}q(S',h))\\
&\leq&  k\cdot\exp(\frac{\epsilon}{4})+ \exp(\frac{\epsilon}{4})\cdot\sum_{h\in G(S')\cap G(S)}\exp(\frac{\epsilon}{4}q(S,h))\\
&\leq&  k\cdot\exp(\frac{\epsilon}{4})+ \exp(\frac{\epsilon}{4})\cdot\sum_{h\in G(S)}\exp(\frac{\epsilon}{4}q(S,h))\\
&=&  k\cdot e^{\epsilon/4}+e^{\epsilon/4}\cdot\XXX\leq e^{\epsilon/2}\XXX.\\
\end{eqnarray*}
This concludes the proof of Fact~3, and completes the proof of the lemma.
\end{proof}

The utility analysis for the choosing mechanism is rather straight forward:
\begin{lemma}\label{lem:CMutility} 
When $q$ is a $k$-bounded-growth quality function, given a database $S$ of $m\geq \frac{16}{\alpha\epsilon}\ln(\frac{16 k}{\alpha\beta\epsilon\delta})$ elements,
the choosing mechanism outputs an $\alpha$-good solution for $S$ with probability at least $1-\beta$.
\end{lemma}

\begin{proof}
Note that if $q(S,f)<\alpha m$ for every solution $f$, then every solution is an $\alpha$-good solution, and the mechanism cannot fail. Assume, therefore, that there exists a solution $f$ s.t. $q(f,S)\geq\alpha m$, and
recall that the mechanism defines ${\rm best}(S)$ as $\max_{f\in\FFF} \left\{ q(f,S) \right\}+\Lap(\frac{4}{\epsilon})$. Now consider the following two good events:
\begin{enumerate}[label=$E_{\arabic*}$:]
\item ${\rm best}(S)\geq\frac{\alpha m}{2}$.
\item The exponential mechanism chooses a solution $f$ s.t. $q(\db,f)\geq \opt(S)-\alpha m$.
\end{enumerate}
If $E_2$ occurs then the mechanism outputs an $\alpha$-good solution.
Note that the event $E_2$ is contained inside the event $E_1$, and, therefore,
$\Pr[E_2]=\Pr[E_1 \wedge E_2]=\Pr[E_1]\cdot\Pr[E_2 | E_1]$.
By the properties of the Laplace Mechanism, 
$\Pr[E_1]\geq\left(1-\frac{1}{2}\exp(-\frac{\epsilon}{4}\frac{\alpha m}{2})\right)$, which is at least $(1-\frac{\beta}{2})$ for $m\geq\frac{8}{\alpha\epsilon}\ln(\frac{1}{\beta})$.

By the growth-boundedness of $q$, and as $S$ is of size $m$, there are at most $km$ possible solutions $f$ with $q(f,S)>0$. That is, $|G(S)|\leq km$. By the properties of the Exponential Mechanism, we have that
$\Pr[E_2|E_1]\geq \left(1- km\cdot\exp(-\frac{\alpha\epsilon m}{4})\right)$, which is at least $(1-\frac{\beta}{2})$ for
$m\geq\frac{8}{\alpha\epsilon}\ln(\frac{16k}{\alpha\beta\epsilon})$. For our choice of $m$ we have, therefore, that $\Pr[E_2]\geq(1-\frac{\beta}{2})(1-\frac{\beta}{2})\geq(1-\beta)$.

All in all, for $m\geq \frac{16}{\alpha\epsilon}\ln(\frac{16 k}{\alpha\beta\epsilon\delta})$ we get that with probability at least $(1-\beta)$ it outputs an $\alpha$-good solution for its input database.
\end{proof}

\subsection{$(\epsilon,\delta)$-Private Sanitizer for $\point_d$}\label{sec:sanIntervals}

Beimel et al.~\cite{BBKN12} showed that every pure $\epsilon$-private sanitizer for $\point_d$, must operate on databases of $\Omega(d)$ elements. In this section we present an $(\epsilon,\delta)$-private sanitizer for $\point_d$ with sample complexity $O_{\alpha,\beta,\epsilon,\delta}(1)$.
This separates the database size necessary for $(\epsilon,0)$-private sanitizers from the database size sufficient for $(\epsilon,\delta)$-private sanitizers.

Let $S=(x_1,x_2,\ldots,x_m)\in X_d^m$ be a database of $d$-bit strings. For every $c_j\in\point_d$, the query $Q_{c_j}:X_d^*\rightarrow[0,1]$ is defined to be the fraction of the strings in the database that equal $j$
$$Q_{c_j}(S)=\frac{1}{m}|\{ i : c_j(x_i)=1  \}| = \frac{1}{m}|\{ i : x_i=j  \}|.$$

Our sanitizing algorithm invokes the Choosing Mechanism to choose points $x\in X_d$. Consider the following $q:X_d^* \times X_d \rightarrow \N$. Given a database $S\in X_d^m$ and a point $x\in X_d$, define $q(S,x)$ to be the number of appearances of $x$ in $S$. By Example~\ref{example:pointBounded}, $q$ defines a 1-bounded-growth choosing problem. Moreover, given a subset $R\subseteq X_d$ consider the restriction of $q$ to the subset $R$ defined as $q_R(S,x)=q(S,x)$ for $x\in R$ and zero otherwise. The function $q_R$ is a 1-bounded-growth quality function. Our sanitizer $SanPoints$ appears in Figure~\ref{fig:sanpoint}.

\begin{figure}[H]
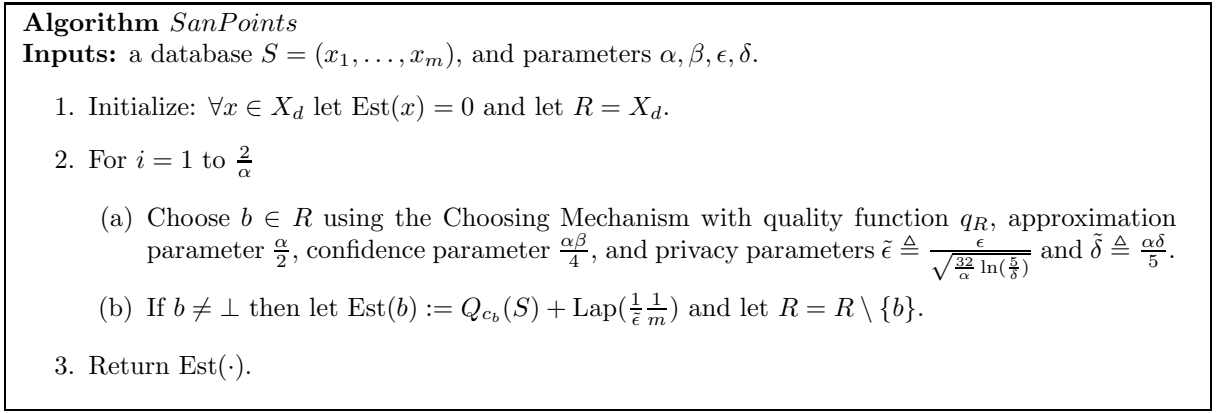

\begin{center}
\noindent\fbox{
\parbox{.95\columnwidth}{
{\bf Algorithm $SanPoints$}\\
{\bf Inputs:} a database $S=(x_1,\ldots,x_m)$, and parameters $\alpha,\beta,\epsilon,\delta$.
\begin{enumerate}[rightmargin=10pt]

\item Initialize: $\forall x\in X_d$ let $\est(x) = 0$ and let $R=X_d$.
\item For $i=1$ to $\frac{2}{\alpha}$

\begin{enumerate}
\item Choose $b\in R$ using the Choosing Mechanism with quality function $q_R$, approximation parameter $\frac{\alpha}{2}$, confidence parameter $\frac{\alpha\beta}{4}$, and privacy parameters $\tilde{\epsilon}\triangleq\frac{\epsilon}{\sqrt{\frac{32}{\alpha}\ln(\frac{5}{\delta})}}$ and $\tilde{\delta}\triangleq\frac{\alpha\delta}{5}$.
\item If $b\neq\bot$ then let $\est(b):=Q_{c_b}(S)+\Lap(\frac{1}{\tilde{\epsilon}}\frac{1}{m})$ and let $R = R \setminus \{ b \}$.
\end{enumerate}
     
\item Return $\est(\cdot)$.

\end{enumerate}}}
\end{center}
\caption{Algorithm $SanPoints$. \label{fig:sanpoint}}
\end{figure}

\begin{theorem}\label{thm:sanPoint}
Fix $\alpha,\beta,\epsilon,\delta$. For
$m\geq O\left(  \frac{1}{\alpha^{1.5} \epsilon}
\sqrt{\ln(\frac{1}{\delta})}
\ln(\frac{1}{\alpha\beta\epsilon\delta}) \right)$,
 algorithm $SanPoints$ is an efficient $(\alpha,\beta,\epsilon,\delta,m)$-improper-sanitizer for $\point_d$.
\end{theorem}

\begin{proof}
We start with the utility analysis. Fix a database $S=(x_1,\ldots,x_m)$, and consider the execution of algorithm $SanPoints$ on $S$. Denote the element chosen by the Choosing Mechanism on the $i^{\rm th}$ iteration of step~2 by $b_i$, and denote the set of all such elements as $B=\{ b_1,\ldots,b_{2/\alpha} \}\setminus\{\bot\}$.
Moreover, let $R_i$ denote the set $R$ as is was at the beginning of the $i^{\text{th}}$ iteration.
Consider the following two bad events:
\begin{enumerate}[label=$E_{\arabic*}.$]
\item $\exists b\in B$ s.t. $|Q_{c_b}(S)-\est(b)|>\alpha$.
\item $\exists a\notin B$ s.t. $Q_{c_a}(S)>\alpha$.
\end{enumerate}
If none of these two events happen, then algorithm $SanPoints$ succeeds in outputting an estimation $\est$ s.t. $\forall c_j\in \point_d \;\; \big| Q_{c_j}(S)-\est(j) \big|\leq\alpha$. We now bound the probability of both events.

Consider an iteration $i$ in which an element $b_i\neq\bot$ is chosen in step~2a.
It holds that $\Pr[|Q_{c_{b_i}}(S)-\est(b_i)|>\alpha]=\Pr[|\Lap(\frac{1}{\tilde{\epsilon}}\frac{1}{m})|>\alpha]=\exp(-\tilde{\epsilon}\alpha m)$.
Using the union bound on the number of iterations, we get that
$\Pr[E_1]\leq\frac{2}{\alpha} \exp(-\tilde{\epsilon}\alpha m)$.
For $m\geq \frac{6}{\alpha^{1.5}\epsilon}\ln(\frac{4}{\alpha\beta})\sqrt{\ln(\frac{5}{\delta})}$ we get that $\Pr[E_1]\leq \frac{\beta}{2}$.

We now bound $\Pr[E_2]$. By the properties of the Choosing Mechanism (Lemma~\ref{lem:CMutility}), with probability at least $(1-\frac{\alpha\beta}{4})$, an execution of the 
Choosing Mechanism on step~2a returns an $\frac{\alpha}{2}$-good solution $b_i$ s.t.
\begin{equation}\label{eq:iSucc}
q_{R_i}(S,b_i)\geq \max_{x\in X_d} \{ q_{R_i}(S,x) \}-\frac{\alpha}{2}m.
\end{equation}

Using the union bound on the number of iterations, we get that with probability at least $(1-\frac{\beta}{2})$, Inequality (\ref{eq:iSucc}) holds for every iteration $1\leq i\leq \frac{\alpha}{2}$. We will now see that in such a case, event $E_2$ does not occur. Assume to the contrary that there exists an $a\notin B$ s.t. $Q_{c_a}(S)>\alpha$. Therefore, for every iteration $i$ it holds that $\max_{x\in X_d} \{ q_{R_i}(S,x) \}>\alpha m$ and thus $q_{R_i}(S,b_i)>\frac{\alpha}{2}m$. This means that there exist (at least) $\frac{2}{\alpha}$ different points $b_i\in X_d$ that appear in $S$ more than $\frac{\alpha}{2}m$ times, which contradicts the fact that the size of $S$ is $m$.

All in all, $\Pr[E_2]\leq\frac{\beta}{2}$, and the probability of algorithm $SanPoints$ failing to output an estimation $\est$ s.t. $\forall c_j\in \point_d \;\; \big| Q_{c_j}(S)-\est(j) \big|\leq\alpha$ is at most $\beta$.

We now proceed with the simple privacy analysis. Note that algorithm $SanPoints$ accesses its input database only using the Choosing Mechanism on step~2a and using the laplacian mechanism on step 2b. Every interaction with the laplacian mechanism preserves $(\tilde{\epsilon},0)$-differential privacy, and there exactly $\frac{2}{\alpha}$ such interactions. For our choice of $m$, every interaction with the Choosing Mechanism preserves $(\tilde{\epsilon},\tilde{\delta})$-differential privacy, and there are exactly $\frac{2}{\alpha}$ such interactions. Applying Theorem~\ref{thm:composition2} (the composition theorem) with our choice of $\tilde{\epsilon},\tilde{\delta}$, we get that algorithm $SanPoints$ preserves $(\epsilon,\delta)$-differential privacy.
\end{proof}

The above algorithm $SanPoints$ can also be used as a sanitizer for the concept class $\kpoint_d$, defined as follows.
For every $A \subseteq X_d$ s.t. $|A|=k$, the concept class $\kpoint_d$ contains the concept $c_A:X_d \rightarrow\{0,1\}$, defined as
$c_A(x)=1$ if $x\in A$ and $c_A(x)=0$ otherwise.

Let $S=(x_1,x_2,\ldots,x_m)\in X_d^m$ be a database.
For every $c_I\in\kpoint_d$, the query $Q_{c_I}:X_d^*\rightarrow[0,1]$ is defined as
$$Q_{c_I}(S)=\frac{1}{m}|\{ i : c_I(x_i)=1 \}|=\frac{1}{m}|\{ i : x_i\in I \}|.$$
Consider the following algorithm.

\begin{figure}[H]
\begin{center}
\noindent\fbox{
\parbox{.95\columnwidth}{
{\bf A Sanitizer for $\kpoint_d$.}\\
{\bf Input:} parameters $\alpha,\beta,\epsilon,\delta$ and a dataset $S$ of $m$ elements.
\begin{enumerate}
\item Execute $SanPoints$ on $S,\frac{\alpha}{k},\beta,\epsilon,\delta$, and denote the returned function as $\est(\cdot)$.
\item For every $I\subseteq X_d$ s.t. $|I|=k$, define $e(I)=\sum_{i\in I}{\est(i)}$.
\item Return $e(\cdot)$.
\end{enumerate}
}}
\end{center}
\caption{A Sanitizer for $\kpoint_d$.}
\end{figure}

\begin{theorem}\label{thm:sanKPoint}
Fix $k,\alpha,\beta,\epsilon,\delta$.
For $m\geq O\left(  \frac{k^{1.5}}{\alpha^{1.5} \epsilon}
\sqrt{\ln(\frac{1}{\delta})}
\ln(\frac{k}{\alpha\beta\epsilon\delta}) \right)$,
the above algorithm is an efficient $(\alpha,\beta,\epsilon,\delta,m)$-improper-sanitizer for $\kpoint_d$.
\end{theorem}

\begin{proof}
The privacy of the above algorithm is immediate. 
Fix a database $S=(x_1,x_2,\ldots,x_m)\in X_d^m$.
By Theorem~\ref{thm:sanPoint}, with probability at least $(1-\beta)$, the estimation $\est$ on step~1 is s.t.
$\forall j\in X_d \;\; \big| \frac{1}{m}\sum_{i=1}^m{\mathds{1}_{\{x_i=j\}}}-\est(j) \big|\leq\frac{\alpha}{k}$. Now fix a set $I\subseteq X_d$ of cardinality $k$. As
$Q_{c_I}(S)=\frac{1}{m}|\{i\; : \; x_i\in I\}|$,
we have that
$|Q_{c_I}(S) - \sum_{i\in I}{\est(i)}  |\leq k\frac{\alpha}{k}=\alpha$.
\end{proof}

\subsection{$(\epsilon,\delta)$-Private Sanitizer for $\thresh_d$}

Recall that $\thresh_d=\{c_0,\ldots,c_{2^d} \}$, where $c_j(x)=1$ if and only if $x<j$. Let $S=(x_1,\ldots,x_m)\in X_d^m$ be a database. For every $c_j\in\thresh_d$, the query $Q_{c_j}:X_d^*\rightarrow[0,1]$ is defined as
$$Q_{c_j}(S)=\frac{1}{m}|\{ i : c_j(x_i)=1  \}|=\frac{1}{m}|\{ i : x_i<j  \}|.$$

As $|\thresh_d|=2^d+1$, one can use the generic construction of Blum et al.~\cite{BLR08full}, and get an $\epsilon$-private sanitizer for this class with sample complexity $O(d)$. By~\cite{BBKN12}, this is the best possible when guaranteeing pure privacy (ignoring the dependency on $\alpha,\beta$ and $\epsilon$).
We next present a recursive sanitizer for $\thresh_d$, guaranteeing approximated privacy and exhibiting sample complexity $\widetilde{O}_{\alpha,\beta,\epsilon,\delta}(8^{\log^*(d)})$.

The algorithm maintains its sanitized database $\hat{S}$ as a global variable, which is initialized as the empty set. In addition, for the privacy analysis, we would need a bound on the number of recursive calls. It will be convenient to maintain another global variable, $calls$, initialized at the desired bound and decreased in every recursive call.

Given a database $S=(x_1,\ldots,x_m)\in X_d^m$, and a subset $R\subseteq X_d$, we denote by $\#_S[R]$ the number of examples $x\in S$ s.t. $x\in R$. That is, $\#_S[R]=|\{ i : x_i\in R  \}|$.
At every recursive call, the algorithm, which is executed on a range $[k,\ell]\subseteq X_d$, identifies an interval $[a,b]\subseteq[k,\ell]$ s.t. $\#_S[a,b]$ is ``not too big'', but ``not too small'' either. Then the algorithm estimates $\hat{\#}[a,b]=\#_S[a,b]+\Lap(\frac{1}{\epsilon})$, and adds $\hat{\#}[a,b]$ copies of the point $b$ to the constructed sanitized database. Afterwards, the algorithm proceeds recursively on the range $[k,a-1]$ and on the range $[b+1,\ell]$. As the number of points in [a,b] is ``not too small'', the depth of this recursion is bounded. Our sanitizer $SanThresholds$ appears in Figure~\ref{fig:san3}. We next prove its properties, starting with the privacy analysis.

\begin{figure}
\begin{center}
\noindent\fbox{
\parbox{\columnwidth}{
{\bf Input:} a range $[k,\ell]$, parameters $\alpha,\beta,\epsilon,\delta$, and a dataset $S$ of $m$ elements. \\
{\bf Global variables:} the sanitized database $\hat{S}$ (initially empty) and $\calls$ (initialized to $\frac{77}{\alpha}$).
\begin{enumerate}[itemsep=0pt,rightmargin=10pt]

\item If $\calls=0$ then halt. Otherwise, set $\calls:=\calls-1$.

\item Compute $\hat{\#}[k,\ell]=\#_S[k,\ell]+\Lap(\frac{1}{\epsilon})$.

\item If $\hat{\#}[k,\ell]<\frac{\alpha m}{8}$ then define $[a,b]:=[k,\ell]$, 
add $\hat{\#}[a,b]$ copies of the point $b$ to $\hat{S}$, and halt.

\item Let $T$ be the smallest power of 2 s.t. $T\geq(\ell-k+1)$.

\item For every $0\leq j \leq \log(T)$, define $I(S,j)=\max\limits_{\substack{[x,y]\subseteq [k,\ell]\\y-x+1\leq2^j}} \Bigg\{  \#_S[x,y] \Bigg\}$.
\begin{enumerate}[label=\gray{\%},topsep=-10pt]
\item \gray{Every interval $[x,y]\subseteq [k,\ell]$ of length $2^j$ contains at most $I(S,j)$ points in $S$, and there exists at least one interval of length (at most) $2^j$ containing exactly $I(S,j)$ points.}
\end{enumerate}
					
\item For every $0\leq j \leq \log(T)$, define its quality $Q(S,j)$ as\\
\centerline{$Q(S,j)=\min\left\{ I(S,j)-\frac{\alpha m}{32} , \frac{3\alpha m}{32} - I(S,j-1)\right\} \text{, where } I(S,-1)\triangleq 0.$}
\begin{enumerate}[label=\gray{\%},topsep=-10pt]
\item \gray{If $Q(S,j)$ is high (for some $j$), then there exists an interval $[a,b]$ of length $2^j$ containing significantly more points than $\frac{\alpha m}{32}$, and every interval of length $\frac{1}{2}2^j$ contains significantly less points than $\frac{3\alpha m}{32}$.}
\end{enumerate}

\item Define $r=\frac{\alpha m}{32}$.

\item Execute algorithm $RecConcave$ on the range $[0,\log(T)]$, the quality function $Q(\cdot,\cdot)$, the quality promise $r$, accuracy parameter $\frac{1}{4}$, and privacy parameters $\hat{\epsilon}=\frac{\epsilon}{3\log^*(d)},\hat{\delta}=\frac{\delta}{3\log^*(d)}$.
Denote the returned value as $z$, and let $Z=2^z$.
\begin{enumerate}[label=\gray{\%},topsep=-10pt]
\item \gray{Assuming the recursive call was successful, the returned $z$ is s.t. $Q(S,z)\geq(1-\frac{1}{4})r=\frac{3\alpha m}{128}$. That is, $I(S,z)\geq\frac{7\alpha m}{128}$ and $I(S,z-1)\leq\frac{9\alpha m}{128}$.  }
\end{enumerate}

\item If $z=0$ then Choose $b\in[k,\ell]$ using the Choosing Mechanism with parameters $\frac{\alpha}{64},\beta,\epsilon,\delta$ and the quality function $\#_S[\cdot]$. Denote $[a,b]=[b,b]$.
\begin{enumerate}[label=\gray{\%},topsep=-10pt]
\item \gray{That is, the quality of $b\in[k,\ell]$ is the number of appearances of the point $b$ is $S$, and the choosing mechanism chooses a frequent point.}
\end{enumerate}

\item Otherwise (if $z\geq1)$ then
\begin{enumerate}[topsep=-10pt]
\item Divide $[k,\ell]$ into the following intervals of length $2Z$ (the last ones might be trimmed): \\
		   $ A_1=[k,k+2Z-1], \; A_2=[k+2Z,k+4Z-1], \; A_3=[k+4Z,k+6Z-1], \ldots $\\
		   $ B_1=[k+Z,k+3Z-1], \; B_2=[k+3Z,k+5Z-1], \; B_3=[k+5Z,k+7Z-1], \ldots $

\item Choose $[a,b]\in\{A_i\}\cup\{B_i\}$ using the Choosing Mechanism with parameters $\frac{\alpha}{64},\beta,\epsilon,\delta$ and the quality function $\#_S[\cdot]$.
\begin{enumerate}[label=\gray{\%},topsep=-10pt]
\item \gray{The union of the $A_i$'s and the $B_i$'s causes the quality function $\#_S[\cdot]$ to be 2-growth-bounded.}
\end{enumerate}
\end{enumerate}

\item Compute $\hat{\#}[a,b]=\#_S[a,b]+\Lap(\frac{1}{\epsilon})$, and add $\hat{\#}[a,b]$ copies of the point $b$ to $\hat{S}$.

\item If $a>k$, then execute $SanThresholds$ recursively on the range $[k,a-1]$, the parameters $\alpha,\beta,\epsilon,\delta$, the database $S$, and the references to $\hat{S}$ and to $\calls$.

\item If $b<\ell$, then execute $SanThresholds$ recursively on the range $[b+1,\ell]$, the parameters $\alpha,\beta,\epsilon,\delta$, the database $S$, and the references to $\hat{S}$ and to $\calls$.

\end{enumerate}
}}
\end{center}
\caption{Sanitizer $SanThresholds$ for $\thresh_d$. \label{fig:san3}}
\end{figure}

\begin{lemma}\label{lem:sanIntervalPrivacy}
When permitted $c$ recursive calls on a sample $S$ of
$m\geq\frac{1024}{\alpha\epsilon}\ln(\frac{2048}{\alpha\beta\epsilon\delta})$
elements, algorithm $SanThresholds$ preserves $(\tilde{\epsilon},5c\delta)$-differential privacy, where $\tilde{\epsilon}=\sqrt{8c\ln(\frac{1}{c\delta})}\epsilon+8c\epsilon^2$.
\end{lemma}

\begin{proof}
Every iteration of algorithm $SanThresholds$ can access its input database at most twice using the laplacian mechanism (on steps 2,11), at most once using the Choosing Mechanism (on step 9 or on step 10b), and at most once using algorithm $RecConcave$ (on step 8).
By the properties of the laplacian mechanism, every interaction with it preserves $(\epsilon,0)$-differential privacy.
Note that the quality function with which we call the Choosing Mechanism is at most 2-growth-bounded. Therefore, as
$m\geq\frac{1024}{\alpha\epsilon}\ln(\frac{2048}{\alpha\beta\epsilon\delta})$,
every such interaction with the Choosing Mechanism preserves $(\epsilon,\delta)$-differential privacy.
Last, for our choice of $\hat{\epsilon},\hat{\delta}$, every interaction with algorithm $RecConcave$ preserves $(\epsilon,\delta)$-differential privacy.

That is, throughout its entire execution, algorithm $SanThresholds$ invokes at most $4c$ mechanisms, each $(\epsilon,\delta)$-differentially private.
By Theorem~\ref{thm:composition2}, algorithm $SanThresholds$ is $(\tilde{\epsilon},5c\delta)$-differential private, where $\tilde{\epsilon}=\sqrt{8c\ln(\frac{1}{c\delta})}\epsilon+8c\epsilon^2$.
\end{proof}

We start the utility analysis of $SanThresholds$ with the following simple claim.

\begin{claim}\label{claim:san3concQ}
The function $Q(S,\cdot)$, defined on step 6, is quasi-concave.
\end{claim}

\begin{proof}
First note that the function $I(S,\cdot)$ defined on step 5 is non-decreasing. Now, let $u\leq v \leq w$ be s.t. $Q(S,u),Q(S,w)\geq x$. That is,
$$\begin{array}{l}
   I(S,u)-\frac{\alpha m}{32}\geq x \\
   \frac{3\alpha m}{32}-I(S,u-1)\geq x
\end{array}
\;\;\;\;\text{ and }\;\;\;\;
\begin{array}{l}
   I(S,w)-\frac{\alpha m}{32}\geq x \\
   \frac{3\alpha m}{32}-I(S,w-1)\geq x
\end{array}.$$
Using the fact that $I(S,\cdot)$ is non-decreasing, we have that $I(S,u)\leq I(S,v)$ and that $I(S,v-1)\leq I(S,w-1)$. Therefore
\begin{eqnarray*}
  I(S,v)-\frac{\alpha m}{32} &\geq& I(S,u)-\frac{\alpha m}{32} \geq x, \\
  \frac{3\alpha m}{32}-I(S,v-1) &\geq& \frac{3\alpha m}{32}-I(S,w-1) \geq x,
\end{eqnarray*}
and $Q(S,v)\geq x$.
\end{proof}

Note that every iteration of algorithm $SanThresholds$ draws at most 2 random samples (on steps~2 and~11) from $\Lap(\frac{1}{\epsilon})$.
We now proceed with the utility analysis by identifying 3 good events that occur with high probability (over the coin tosses of the algorithm).

\begin{claim}\label{claim:san3SmallEvents}
Fix $\alpha,\beta,\epsilon,\delta$. Let $SanThresholds$ be executed with $\calls$ initialized to $c\geq\frac{77}{\alpha}$, and on a database $S$ of
$m\geq 8^{\log^*(d)} \cdot \frac{60 c}{\alpha\epsilon} \log^*(d) \log\big(\frac{12\log^*(d)}{\beta\epsilon\delta}\big)$
elements. With probability at least $(1-3c\beta)$ the following 3 events happen:
\begin{enumerate}[label=$B_{\arabic*}:$]
\item In every random draw of $\Lap(\frac{1}{\epsilon})$ throughout the execution of $SanThresholds$ it holds that $|\Lap(\frac{1}{\epsilon})|\leq \frac{\alpha m}{16c}$.
\item Every interaction with algorithm $RecConcave$ on step~8 succeeds in returning a value $z$ s.t. $Q(S,z)\geq\frac{3\alpha m}{128}$.
\item Every iteration that halts after step~13, defines an interval $[a,b]$ s.t. $\#_S[a,b]\geq\frac{5\alpha m}{128}$.
\end{enumerate}
\end{claim}

\begin{proof}
First note that it suffices to lower bound the terms $\Pr[B_1],\; \Pr[B_2 | B_1]$, and $\Pr[B_3 | B_1 \wedge B_2]$, as by the chain rule of conditional probability it holds that 
$$
\Pr[B_1 \wedge B_2 \wedge B_3] = \Pr[B_1] \cdot \Pr[B_2 | B_1] \cdot \Pr[B_3 | B_1 \wedge B_2].
$$

We now bound each of those terms, starting with $\Pr[B_1]$. In every single draw, the probability of $|\Lap(\frac{1}{\epsilon})|>\frac{\alpha m}{16c}$ is at most $\exp(\frac{-\alpha\epsilon m}{16c})$, which is at most $\frac{\beta}{2}$ for $m\geq\frac{16c}{\alpha\epsilon}\ln(\frac{2}{\beta})$. As $c$ (the initial value of $\calls$) limits the number of iteration, we get that $\Pr[B_1]\geq(1-c\beta)$.\\

For the analysis of $\Pr[B_2 | B_1]$, consider an iteration of algorithm $SanThresholds$ that executes $RecConcave$ on step~8. In particular, this iteration passed step~3 and $\hat{\#}[k,\ell]\geq \frac{\alpha m}{8}$. As event $B_1$ has occurred, we have that $\#_S[k,\ell]\geq\frac{\alpha m}{16}$. 
Recall that (by definition) $I(S,-1)=0$, and so, there exists a $j\in[0,\log(T)]$ s.t. $I(S,j)\geq\frac{\alpha m}{16}$, and 
$I(S,j-1)<\frac{\alpha m}{16}$. Plugging those inequalities in the definition of $Q(\cdot,\cdot)$, for this $j$ we have that
$Q(S,j)\geq\frac{\alpha m}{32}=r$, and the quality promise used to execute algorithm $RecConcave$ is valid. Moreover, the function $Q(S,\cdot)$ defined on step~6 is quasi-concave (by Claim~\ref{claim:san3concQ}).
And so,
for 
$m\geq 8^{\log^*(d)} \cdot \frac{4608}{\alpha\epsilon} \log^*(d) \log\big(\frac{12\log^*(d)}{\beta\delta}\big)$,
algorithm $RecConcave$ ensures that with probability at least $(1-\beta)$, the returned $z$ is s.t. $Q(S,z)\geq(1-\frac{1}{4})r=\frac{3\alpha m}{128}$. As there are at most $c$ iterations, $\Pr[B_2 | B_1]\geq(1-c\beta)$.\\

For the analysis of $\Pr[B_3 | B_1 \wedge B_2]$, consider an iteration of algorithm $SanThresholds$ that halts after step~13, and let $z$ be the value returned by algorithm $RecConcave$ on step~8. As event $B_2$ has occurred, $Q(S,z)\geq\frac{3 \alpha m}{128}$. In particular, $I(S,z)\geq\frac{7\alpha m}{128}$, and there exists an interval $G\subseteq[k,\ell]$ of length (at most) $2^z$ containing at least $\frac{7\alpha m}{128}$ points in $S$. 
Assume $z>0$, and consider the intervals in $\{A_i\}$ and $\{B_i\}$ defined on Step~10a. As those intervals are of length $2\cdot2^z$, and the $\{B_i\}$'s are shifted by $2^z$, there must exist an interval $C\in\{A_i\}\cup\{B_i\}$ s.t. $G\subseteq C$. The quality of $C$ is at least $\#_S[C]\geq\frac{7\alpha m}{128}$. Moreover, this quality function $\#[\cdot]$ over $\{A_i\}\cup\{B_i\}$ is 2-bounded. Therefore, as
$m\geq\frac{1024}{\alpha\epsilon}\ln(\frac{2048}{\alpha\beta\epsilon\delta})$,
with probability at least $(1-\beta)$, the Choosing Mechanism returns an interval $[a,b]$ s.t. $\#_S[a,b]\geq\frac{7\alpha m}{128}-\frac{\alpha m}{64}=\frac{5\alpha m}{128}$.
This also holds when $z=0$ (on Step~9).

So, given that event $(B_1 \wedge B_2)$ has occurred, in every iteration that halts after step~13, the probability of defining $[a,b]$ s.t. $\#_S[a,b]<\frac{5\alpha m}{128}$ is at most $\beta$. As there are at most $c$ iterations, we see that $\Pr[B_3|B_1\wedge B_2]\geq(1-c\beta)$.

All in all, for $c\beta\leq3$ we get that
$$
\Pr[B_1 \wedge B_2 \wedge B_3] \geq (1-c\beta)(1-c\beta)(1-c\beta) \geq 1-3c\beta.
$$
\end{proof}

Every iteration of algorithm $SanThresholds$ that does not halt on step 1 defines an interval $[a,b]$ (on exactly one of the steps 3,9,10b). This interval $[a,b]$ is not part of any range that is given as input to any future recursive call. Moreover, if none of the recursive calls throughout the execution of $SanThresholds$ halts on step 1, these $[a,b]$ intervals form a partition of the initial range. We now proceed with the utility analysis by identifying yet another 3 good events (at a somewhat higher level) that occur whenever $(B_1\wedge B_2\wedge B_3)$ occur.

\begin{claim}\label{claim:san3Events}
Fix $\alpha,\beta,\epsilon,\delta$. Let $SanThresholds$ be executed with $\calls$ initialized to $c\geq\frac{77}{\alpha}$, and on a database $S$ of
$m\geq 8^{\log^*(d)} \cdot \frac{60 c}{\alpha\epsilon} \log^*(d) \log\big(\frac{12\log^*(d)}{\beta\epsilon\delta}\big)$
elements. With probability at least $(1-3c\beta)$ the following 3 events happen:
\begin{enumerate}[label=$E_{\arabic*}:$]
\item There are at most $\frac{77}{\alpha}$ recursive calls, none of them halts on the first step.
\item Every iteration defines $[a,b]$ s.t. $\#_S[a,b-1]\leq\frac{\alpha m}{2}$. That is, every iteration defines $[a,b]$ s.t. the interval $[a,b-1]$ contains at most $\frac{\alpha m}{2}$ points in $S$.
\item In every iteration $\left|\#_S[a,b]-\hat{\#}[a,b]\right|\leq\frac{\alpha m}{4}\frac{\alpha}{77}$.
\end{enumerate}
\end{claim}

\begin{proof}
Consider again events $B_1,B_2,B_3$ defined in Claim~\ref{claim:san3SmallEvents}.
We will show that the event $(E_1 \wedge E_2 \wedge E_3)$ is implied by $(B_1\wedge B_2\wedge B_3)$ (which happens with probability at least $(1-3c\beta)$ by Claim~\ref{claim:san3SmallEvents}). 
We, therefore, continue the proof assuming that $(B_1\wedge B_2\wedge B_3)$ has occurred.

We begin by showing that event $E_1$ occurs. Denote the number of iterations that halts on steps~1-3 as $y_1$, and the number of complete iterations (i.e., that halts after step~13) as $y_2$. Clearly, $y_1\leq 2 y_2$. Now, as event $B_3$ has occurred, we have that every iteration that halts after step~13 defines an interval $[a,b]$ s.t. $\#_S[a,b]\geq\frac{5\alpha m}{128}$. This interval does not intersect any range given as input to future calls, and, therefore, $y_2\leq\frac{128}{5\alpha}$. The total number of iterations is, therefore, bounded by $3 y_2\leq\frac{384}{5\alpha}<\frac{77}{\alpha}$. Thus, whenever $\calls$ is initialized to at least $77/\alpha$,  there are at most $\frac{77}{\alpha}$ iterations, none of them halts on step~1. That is, $E_1$ occurs.\\

We next show that $E_3$ occurs.
As we have seen, event $B_3$ ensures that no iteration halts on step~1. Therefore every iteration defines $\hat{\#}[a,b]$ by adding a random draw of $\Lap(\frac{1}{\epsilon})$ to $\#_S[a,b]$. As event $B_1$ has occurred, it holds that 
$\left|\#_S[a,b]-\hat{\#}[a,b]\right|\leq\frac{\alpha m}{16c}\leq\frac{\alpha m}{4}\frac{\alpha}{77}$. So, $E_3$ occurs.\\

It remains to show that $E_2$ occurs.
As $B_3$ has occurred, no iteration of algorithm $SanThresholds$ halts on step~1. In particular, every iteration defines $[a,b]$ on exactly one of the steps 3,9,10b. Consider an iteration of algorithm $SanThresholds$ that defines $[a,b]$ on step~3. In that iteration, $\hat{\#}[k,\ell]<\frac{\alpha m}{8}$. As event $B_1$ has occur, it holds that $\#_S[k,\ell]\leq\frac{\alpha m}{2}$. Therefore the interval $[a,b-1]=[k,\ell-1]$ contains at most $\frac{\alpha m}{2}$ points.

Consider an iteration of algorithm $SanThresholds$ that defines $[a,b]$ on step~9. In that iteration, $[a,b]$ is defined as $[a,a]$. Trivially, the empty interval $[a,b-1]=[a,a-1]$ contains at most $\frac{\alpha m}{2}$ points.

Consider an iteration of algorithm $SanThresholds$ that defines $[a,b]$ on step~10b (of length at most $2\cdot2^z$).
As event $B_2$ has occurred, $z$ is s.t. $Q(S,z)\geq\frac{3\alpha m}{128}$. In particular $L(S,z-1)\leq\frac{9\alpha m}{128}$, and every interval of length $\frac{1}{2}2^z$ contains at most $\frac{9\alpha m}{128}$ points in $S$.
Therefore $\#_S[a,b-1]\leq4\frac{9\alpha m}{128}\leq\frac{\alpha m}{2}$.
Note that we needed $z$ to be at least 1 (ensured by the If condition on step~10), as otherwise the constraint on intervals of length $\frac{1}{2}2^z$ has no meaning.

At any case, we have that $E_2$ must occur.\\

All in all,
$$
\Pr[E_1 \wedge E_2 \wedge E_3] \geq \Pr[B_1\wedge B_2\wedge B_3]\geq(1-3c\beta).
$$
\end{proof}

We will now complete the utility analysis by showing that the input database $S$ and the sanitized database $\hat{S}$ (at the end of $SanThresholds$' execution) are $\alpha$-close whenever $(E_1 \wedge E_2 \wedge E_3)$ occurs.

\begin{lemma}\label{lem:sanIntervalUtility}
Fix $\alpha,\beta,\epsilon,\delta$. Let $SanThresholds$ be executed on the range $X_d$, a global variable $\calls$ initialize to $c\geq\frac{77}{\alpha}$, and on a database $S$ of
$m\geq 8^{\log^*(d)} \cdot \frac{60 c}{\alpha\epsilon} \log^*(d) \log\big(\frac{12\log^*(d)}{\beta\epsilon\delta}\big)$
elements. With probability at least $(1-3c\beta)$, the sanitized database $\hat{S}$ at the end of the execution is s.t. $|Q_{c_j}(S)-Q_{c_j}(\hat{S})|\leq\alpha$ for every $c_j\in\thresh_d$.
\end{lemma}

\begin{proof}
Denote $S=(x_1,\ldots,x_m)$, and $\hat{S}=(\hat{x_1},\ldots,\hat{x_n})$. Note that $|S|=m$ and that $|\hat{S}|=n$.
By Claim~\ref{claim:san3Events}, the event $E_1 \cap E_2 \cap E_3$ occurs with probability at least $(1-3c\beta)$. We will show that in such a case, the sanitized database $\hat{S}$ is s.t. $|Q_{c_j}(S)-Q_{c_j}(\hat{S})|\leq\alpha$ for every $c_j\in\thresh_d$.

As event $E_1$ has occurred, the intervals $[a,b]$ defined throughout the execution of $SanThresholds$ defines a partition of the domain $X_d$. Denote those intervals as $[a_1,b_1],\;[a_2,b_2],\;\ldots,\;[a_w,b_w]$, where $a_1=0,\; b_w=2^d-1$, and $a_{i+1}=b_i+1$. Now fix some $c_j\in\thresh_d$, and let $t$ be s.t. $j\in[a_t,b_t]$. We have that
$$Q_{c_j}(S)=\frac{1}{m}\#_S[0,j-1]=\frac{1}{m}\left(\#_S[a_t,j-1]+\sum_{i=1}^{t-1}{\#_S[a_i,b_i]}\right).$$
As event $E_2 \cap E_3$ has occurred,
$$Q_{c_j}(S)\leq\frac{1}{m}\left(\frac{\alpha m}{2}+\sum_{i=1}^{t-1}{\left[\hat{\#}[a_i,b_i]+\frac{\alpha m}{4}\frac{\alpha}{77}\right]}\right).$$
As event $E_1$ has occurred, $t\leq\frac{77}{\alpha}$, and
$$Q_{c_j}(S)\leq\frac{\alpha}{2}+\frac{\alpha}{4}+\frac{1}{m}\sum_{i=1}^{t-1}{\hat{\#}[a_i,b_i]}=
\frac{3\alpha}{4}+\frac{1}{m}\#_{\hat{S}}[0,j-1].
$$
Similar arguments show that $Q_{c_j}(S)\geq -\frac{3\alpha}{4}+\frac{1}{m}\#_{\hat{S}}[0,j-1]$, and so $\left| Q_{c_j}(S) - \frac{1}{m}\#_{\hat{S}}[0,j-1] \right| \leq \frac{3\alpha}{4}$.

Recall that the sanitized database $\hat{S}$ is of size $n$, and that $Q_{c_j}(\hat{S})=\frac{1}{n}\#_{\hat{S}}[0,j-1]$. As event $(E_1 \cap E_3)$ has occurred, we have that $n\leq m+ \frac{\alpha m}{4}=(1+\frac{\alpha}{4})m$. Therefore,
$$
\frac{\#_{\hat{S}}[0,j-1]}{m} - \frac{\#_{\hat{S}}[0,j-1]}{n} = \left(\frac{1}{m}-\frac{1}{n}\right)\#_{\hat{S}}[0,j-1]
\leq\frac{\alpha}{4n}\#_{\hat{S}}[0,j-1]\leq\frac{\alpha}{4}.
$$
Similar arguments show that $\left|  \frac{1}{m}\#_{\hat{S}}[0,j-1]-\frac{1}{n}\#_{\hat{S}}[0,j-1] \right|\leq\frac{\alpha}{4} $. By the triangle inequality we have therefore that $\left|  Q_{c_j}(S) -Q_{c_j}(\hat{S}) \right|\leq\frac{3\alpha}{4}+\frac{\alpha}{4}=\alpha$.
\end{proof}

The following theorem is an immediate consequence of Lemma~\ref{lem:sanIntervalUtility} and Lemma~\ref{lem:sanIntervalPrivacy}.

\begin{theorem}\label{thm:sanIntervals}
Fix $\alpha,\beta,\epsilon,\delta$. There exists an efficient $(\alpha,\beta,\epsilon,\delta,m)$-sanitizer for $\thresh_d$, where
$$ m= O\left(
8^{\log^*(d)} \cdot \frac{\log^*(d)}{\alpha^{2.5}\epsilon} \log\Big(\frac{\log^*(d)}{\alpha\beta\epsilon\delta}\Big) \sqrt{\log\Big(\frac{1}{\alpha\delta}\Big)} 
\right).$$
\end{theorem}

\subsection{Sanitization with Pure Privacy}
Here we give a general lower bound on the database size of pure private
sanitizers.
Beimel et al.~\cite{BBKN12} showed that every pure $\epsilon$-private sanitizer for $\point_d$ must operate on databases of $\Omega(d)$ elements.
With slight modifications, their proof technique can yield a much more general result.
\begin{definition}
Given a concept class $C$ over a domain $X$, we denote the {\em effective} size of $X$ w.r.t. $C$ as
$$
X_C = \max\left\{  |\widetilde{X}| \; : \; 
\widetilde{X} \subseteq X \; \text{ s.t. } \;	\forall x_1\neq x_2\in\widetilde{X} \;\; \exists f\in C \; \text{ s.t. } \; f(x_1)\neq f(x_2)
\right\}.
$$
\end{definition}

That is, $X_C$ is the cardinality of the biggest subset $\widetilde{X}\subseteq X$ s.t. every two different elements of $\widetilde{X}$ are labeled differently by at least one concept in $C$.

\begin{lemma} \label{lem:genSanLower}
Let $C$ be a concept class over a domain $X$.
For every $(\alpha,\beta,\epsilon,m)$-sanitizer for $C$ (proper or improper) it holds that $m=\Omega\left(\frac{1}{\eps\alpha}(\log X_C +\log (1/\beta))\right)$. 
\end{lemma}

\begin{proof}
Let $\widetilde{X} \subseteq X$ be s.t. $|\widetilde{X}|=X_C$ and every two different elements of $\widetilde{X}$ are labeled differently by at least one concept in $C$. Fix some $x_1\in\widetilde{X}$, and for every $x_i\in\widetilde{X}$, construct a database $S_i \in \widetilde{X}^m$ by setting
$(1-3\alpha)m$ entries as $x_1$ and the remaining $3\alpha m$ entries as $x_i$
(for $i=1$ all entries of $S_1$ are $x_1$). Note that for all $i\neq j$, databases $S_i$ and
$S_j$ differ on $3\alpha m$ entries.

Let $\mathbb{S}_i$ be the set of all databases that are
$\alpha$-close to $S_i$. That is,
$$
\mathbb{S}_i = \left\{  \widehat{S}\in X^* \; : \; \forall c\in C \text{ it holds that } |Q_c(\widehat{S}) - Q_c(S_i)|\leq\alpha \right\}.
$$

For every $i\neq j$ we have that $\widehat{\mathbb{S}}_i \cap \widehat{\mathbb{S}}_j = \emptyset$. 
To see this, let $f\in C$ be s.t. $f(x_i)\neq f(x_j)$ (such a concept exists, by the definition of $\widetilde{X}$).
For this $f$ it holds that $|Q_f(S_i)-Q_f(S_j)|=3\alpha$. Therefore (by the triangle inequality), there cannot exist 
a database $\widehat{S}$ for which $|Q_f(S_i)-Q_f(\widehat{S})|\leq\alpha$ and $|Q_f(S_j)-Q_f(\widehat{S})|\leq\alpha$.

Let $\AAA$ be an  $(\alpha,\beta,\epsilon,m)$-sanitizer for $C$.
Without loss of generality, we can assume that $\AAA$ is a {\em proper} sanitizer (otherwise, we could transform it into a proper one by replacing $\alpha$ with $2\alpha$). See Remark~\ref{improperProperSanitization}.

For all $i$, on input $S_i$ the mechanism $\AAA$ should pick an
output from $\mathbb{S}_i$ with probability at least
$1-\beta$. 
Hence,
\begin{eqnarray*}
\beta &\geq& \Pr[\AAA(S_1) \not\in \mathbb{S}_1] \\
&\geq& \Pr\left[\AAA(S_1) \in \bigcup_{i\not=1} \mathbb{S}_i\right]\\
&=& \sum_{i\not=1} \Pr[\AAA(S_1) \in \mathbb{S}_i] \quad \quad \quad (\mbox{the sets $\mathbb{S}_i$ are disjoint}) \\
&\geq& \sum_{i\not=1} \exp(-3\eps \alpha m) \Pr[\AAA(S_i) \in \mathbb{S}_i] \quad \quad \quad (\mbox{by the differential privacy of $\AAA$}) \\
& \geq & (X_C-1)\exp(-3\eps \alpha m) \cdot (1-\beta).
\end{eqnarray*}
Solving for $m$, we get that
$m = \Omega(\frac{1}{\eps\alpha}(\log X_C+\log(1/\beta)))$.
\end{proof}

Lemma~\ref{lem:genSanLower}, together with a lower bound from~\cite{BLR08full}, yields the following result:
\begin{theorem}\label{thm:genSanLower}
Let $C$ be a concept class over a domain $X$. If $\AAA$ is an $(\frac{1}{8},\frac{1}{8},\frac{1}{2},m)$-sanitizer for $C$, then $m=\Omega(\log(X_C)+\VC(C))$.
\end{theorem}

\begin{proof}
Immediate from Lemma~\ref{lem:genSanLower} and Theorem~\ref{thm:BlumLow}.
\end{proof}

The above lower bound is the best possible general lower bound in terms of $X_C$ and $\VC(C)$ (up to a factor of $\log\VC(C)$). To see this, let $n<d$, and consider a concept class over $X_d$ containing the following two kinds of concepts. The first kind are $2^n$ concepts shattering the left $n$ points of $X_d$ (and zero everywhere else). The second kind are $(2^d-n)$ ``point concepts'' over the right $(2^d-n)$ points of $X_d$ (and zero on the first $n$). Formally,
for every $j=(j_0,j_1,\ldots,j_{n-1})\in \{0,1\}^n$, let $c_j:X_d\rightarrow\{0,1\}$ be defines as
$c_j(x)=j_x$ if $x<n$ and $c_j(x)=0$ otherwise. Define the concept class $C_L=\{ c_j \}_{j\in X_n}$.
For every $n\leq j<2^d$, define $f_j:X_d\rightarrow\{0,1\}$ as $f_j(x)=1$ if $x=j$ and $f_j(x)=0$ otherwise. Define the concept class
$C_R=\{ f_j \}_{n\leq j<2^d}$. Now define $C=C_L\bigcup C_R$.

We can now construct a sanitizer for $C$ by applying the generic construction of~\cite{BLR08full} separately for $C_L$ and for $C_R$.
Given a database $S$, this will result in two sanitized databases $\widehat{S}_L,\widehat{S}_R$, with which we can answer all queries in the class $C$ --
a query for $c\in C_L$ is answered using $\widehat{S}_L$, and a query for $f\in C_R$ is answered using $\widehat{S}_R$.
The described (improper) sanitizer for $C$ is of sample complexity $O_{\alpha,\beta,\epsilon}(\log(X_C)+\VC(C)\log\VC(C))$.

\section{Sanitization and Proper Private PAC}\label{sec:reduction}

Similar techniques are used for both data sanitization and private learning, suggesting relationships between the two tasks. We now explore one such relationship in proving a lower bound on the sample complexity needed for sanitization (under pure differential privacy). In particular, we show a {\em reduction} from the task of private learning to the task of data sanitization, and then use a lower bound on private learners to derive a lower bound on data sanitization.
A similar reduction was given by Gupta et al.~\cite{GHRU11}, where it is stated in terms of statistical queries. They showed that the existence of a sanitizer that accesses the database using at most $k$ statistical queries, implies the existence of a learner that makes at most $2k$ statistical queries. We complement their proof and add the necessary details in order to show that the existence of an {\em arbitrary} sanitizer (that is not restricted to access its data via statistical queries) implies the existence of a private learner.

\paragraph{Notation.} We will refer to an element of $X_{d+1}$ as $\vec{x}\circ y$, where $\vec{x}\in X_d$, and $y\in\{0,1\}$.

\subsection{Sanitization Implies Proper PPAC}
We show that sanitization of a class $C$ implies private learning of $C$.
Consider an input labeled sample $S=(x_i,y_i)_{i=1}^m\in(X\times\{0,1\})^m$, labeled by some concept $c\in C$.
The key observation is that in order to privately output a good hypothesis it is suffices to first produce a sanitization $\hat{S}$ of $S$ (w.r.t. a slightly different concept class $C^{\rm label}$, to be defined) and then to output a hypothesis $h\in C$ that minimizes the empirical error {\em over the sanitized database $\hat{S}$}.
To complete the proof we then show that sanitization for $C$ implies sanitization for $C^{\rm label}$.

In order for the chosen hypothesis $h$ to have small {\em generalization} error (rather then just small empirical error), our input database $S$ must contain at least
$\frac{\VC(C)}{\alpha^2}\log(\frac{1}{\alpha\beta})$ elements.
We therefore start with the following simple (technical) lemma, handling a case where our initial sanitizer operates only on smaller databases. 

\begin{lemma}\label{lemma:saninitationSample}
If there exists an $(\alpha,\beta,\epsilon,m)$-sanitizer for a class $C$,
then for every $q\in\N$ s.t. $q\geq\frac{18}{\beta}\ln(1/\beta)$ there exists a $((2\alpha+2\beta),\beta,\epsilon,qm)$-sanitizer for $C$.
\end{lemma}

\begin{proof}
Fix $q\in\N$ and let $A$ be an $(\alpha,\beta,\epsilon,m)$-sanitizer for a class $C$ over a domain $X$.
Note that by Theorem~\ref{thm:BlumShrink}, there exists a $(2\alpha,\frac{3}{2}\beta,\epsilon,m)$-sanitizer $A'$ s.t. the sanitized databases returned by $A'$ are always of fixed sized $n=O(\frac{\VC(C)}{\alpha^2}\log(\frac{1}{\alpha\beta}))$. We now construct a $((2\alpha+2\beta),\beta,\epsilon,qm)$-sanitizer $B$ as follows.
$$\boxed{
\begin{array}{l}
\text{Inputs: a database } S=(z_1,z_2,\ldots,z_{qm})\in (X)^{qm}\\
{\begin{array}{ll}
1. & \text{Partition S into } S_1=(z_i)_{i=1}^m,\; S_2=(z_i)_{i=m+1}^{2m},\; \ldots, S_q=(z_i)_{i=qm-m+1}^{qm} .\\
2. & \text{For every }  1\leq i\leq q, \; \hat{S_i} \leftarrow A'(S_i).\\
3. & \text{Output }  \hat{S}=\langle \hat{S_1},\hat{S_2},...,\hat{S}_q\rangle.\\
\end{array}}\\
\end{array}}$$
As $A'$ is $\epsilon$-differentially private, so is $B$.
Denote $\hat{S}=(\hat{z}_1,\hat{z}_1,\ldots,\hat{z}_{qn})\in (X)^{qn}$.
Recall that $q\geq\frac{18}{\beta}\ln(1/\beta)$, and, hence, using the Chernoff bound, with probability at least $(1-\beta)$ it holds that at least $(1-2\beta)q$ of the $\hat{S_i}$'s are $2\alpha$-good for their matching $S_i$'s. In such a case $\hat{S}$ is $(2\alpha+2\beta)$-good for $S$: for every $f \in C$ it holds that
\begin{eqnarray*}
Q_f(S) &=& \frac{1}{qm}\left|\left\{ i \; : \;  \begin{array}{c}1\leq i \leq qm\\f(z_i)=1\end{array} \right\}\right|\\
&=&\frac{1}{qm}\left|\left\{ i \; : \;  \begin{array}{c}1\leq i \leq m\\f(z_i)=1\end{array} \right\}\right|
 + \cdots + \frac{1}{qm}\left|\left\{ i \; : \;  \begin{array}{c}qm-m+1\leq i \leq qm\\f(z_i)=1\end{array} \right\}\right|\\
&=& \frac{1}{q}\left[ Q_f(S_1) + \cdots + Q_f(S_q) \right].
\end{eqnarray*}

As at least $(1-2\beta)q$ of the $\hat{S_i}$'s are $2\alpha$-good for their matching $S_i$'s, and as trivially $Q_f(S_i)\leq 1$ for each database $\hat{S_i}$ that is not $2\alpha$-good,

\begin{eqnarray*}
Q_f(S) &\leq& \frac{1}{q}\left[Q_f(\hat{S}_1)+\ldots + Q_f(\hat{S}_q) + (1-2\beta)q2\alpha + 2\beta q\right]\\
&\leq& \frac{1}{q}\left[Q_f(\hat{S}_1) + \ldots + Q_f(\hat{S}_{t/m})\right] + (2\alpha+2\beta)\\
&=& \frac{1}{q}\left[\;\frac{1}{n}
\left|\left\{ i \; : \;  \begin{array}{c}1\leq i \leq n\\f(\hat{z}_i)=1\end{array} \right\}\right|
 + \ldots + \frac{1}{n}
\left|\left\{ i \; : \;  \begin{array}{c} qn-n+1 \leq i \leq qn\\f(\hat{z}_i)=1\end{array} \right\}\right|
\;\right] + (2\alpha+2\beta)\\
&=& \frac{1}{q n} 
\left|\left\{ i \; : \;  \begin{array}{c} 1 \leq i \leq qn\\f(\hat{z}_i)=1\end{array} \right\}\right|
 + (2\alpha+2\beta)\\
&=& Q_f(\hat{S})+(2\alpha+2\beta).
\end{eqnarray*}
Similar arguments show that $Q_f(S)\geq Q_f(\hat{S})-(2\alpha+2\beta)$.
Algorithm $B$ is, therefore, a $((2\alpha+2\beta),\beta,\epsilon,qm)$-sanitizer for $C$, as required.
\end{proof}

As mentioned above, our first step in showing that sanitization for a class $C$ implies private learning for $C$ is to show that privately learning $C$ is implied by sanitization for the slightly modified class $C^{\rm label}$, defined as follows.
For a given predicate $c$ over $X_d$, we define the predicate $c^{\rm label}$ over $X_{d+1}$ as
$$ c^{\rm label}(\vec{x}\circ y)=\begin{cases}
    1, & c(\vec{x})\neq y.\\
    0, & c(\vec{x})=y.
\end{cases}$$
Note that $c^{\rm label}(\vec{x}\circ\sigma) = \sigma \oplus c(\vec{x})$ for $\sigma\in\{0,1\}$. For a given class of predicates $C$ over $X_d$, we define $C^{\rm label}=\{c^{\rm label} \; : \; c\in C  \}$. 
\begin{claim}
$\VC(C)\leq \VC(C^{\rm label})\leq2\cdot \VC(C)$. 
\end{claim} 
\begin{proof}
For the first inequality notice that if a set $S \subseteq X_d$ is shuttered by $C$ then the set $S\circ 0$ is shuttered by $C^{\rm label}$. For the second inequality, assume $S\subseteq X_{d+1}$ is shattered by $C^{\rm label}$.
Consider the partition of $S$ to $S_0$ and $S_1$, where $S_\sigma=\{\vec{x}\circ y\in S \; : \; y=\sigma\}$. For at least one $\sigma\in\{0,1\}$,
it holds that $|S_\sigma|\geq\frac{|S|}{2}$.  Hence, the set $\hat{S}=\{\vec{x} \; : \; \vec{x}\cdot \sigma\in S_\sigma\}$ is shattered by $C$ and $\VC(C^{\rm label})\leq 2\cdot |\hat{S}| \leq 2\cdot \VC(C)$.
\end{proof}

The next lemma shows that for every concept class $C$, a sanitizer for $C^{\rm label}$ implies a private learner for $C$. In the next lemma, this connection is made under the assumption that the given sanitizer operates on large enough databases. This assumption will be removed in the lemma that follows.

\begin{lemma}\label{lem:sanPac1}
Let $C$ be a class of predicates over $X_d$.
If there exists an $(\alpha,\beta,\epsilon,m)$-sanitizer $A$ for $C^{\rm label}$,
where $m\geq\frac{50 \VC(C)}{\gamma^2}\ln(\frac{1}{\gamma\beta})$ for some $\gamma>0$,
then there exists a proper $((2\alpha+\gamma),2\beta,\epsilon,m)$-PPAC learner for $C$.
\end{lemma}

\begin{proof}
Let $A$ be an $(\alpha,\beta,\epsilon,m)$-sanitizer, and consider the following algorithm $Learn$:
$$\boxed{
\begin{array}{l}
\text{Inputs: a database } S=(x_i,y_i)_{i=1}^m\\
{\begin{array}{ll}
1. & \hat{S} \leftarrow A(S).\\
2. & \text{Output } c\in C \text{ minimizing } \error_{\hat{S}}(c).\\
\end{array}}\\
\end{array}}$$
As $A$ is $\epsilon$-differentially private, so is $Learn$.
For the utility analysis, fix some target concept $c_t\in C$ and a distribution $\DDD$ over $X_d$, and define the following two good events:
\begin{enumerate}[label=$E_{\arabic*}:$]
\item $\forall h\in C, \;\; \big| \error_S(h)-\error_{\hat{S}}(h) \big|\leq\alpha$.
\item $\forall h\in C,\;\; |\error_\DDD(h,c_t)-\error_S(h)|\leq\gamma$.
\end{enumerate}

We first show that if these 2 good events happen, algorithm $Learn$ returns a $(2\alpha+\gamma)$-good hypothesis.
As the target concept satisfies $\error_S(c_t)=0$, event $E_1$ ensures the existence of a concept $f\in C$ s.t. $\error_{\hat{S}}(f)\leq\alpha$. Thus, algorithm $Learn$ chooses a hypothesis $h\in C$ s.t. $\error_{\hat{S}}(h)\leq\alpha$.
Using event $E_1$ again, this $h$ obeys $\error_{S}(h)\leq2\alpha$.
Therefore, event $E_2$ ensures that $h$ satisfies $\error_\DDD(h,c_t)\leq2\alpha+\gamma$.

We will now show that these 2 events happen with high probability.
By the definition of $C^{\rm label}$, for every $c^{\rm label}\in C^{\rm label}$ we have that 
$$Q_{c^{\rm label}}(S)=
\frac{1}{|S|}\left|\left\{ i \; : \; c^{\rm label}(x_i \circ y_i)=1 \right\}\right| =
\frac{1}{|S|}\left|\left\{ i \; : \; c(x_i) \neq y_i \right\}\right| =
\error_S(c).$$
Therefore, as $A$ is an $(\alpha,\beta,\epsilon,m)$-sanitizer for $C^{\rm label}$, event $E_1$ happens with probability at least $(1-\beta)$.
As $m\geq\frac{50 \VC(C)}{\gamma^2}\ln(\frac{1}{\gamma\beta})$, Theorem~\ref{thm:generalization} ensures that event $E_2$ happens with probability at least $(1-\beta)$ as well. 
All in all, $Learn$ is a proper $((2\alpha+\gamma),2\beta,\epsilon,m)$-PPAC learner for $C$.
\end{proof}

The above lemma describes a reduction from the task of privately learning a concept class $C$ to the sanitization task of the slightly different concept class $C^{\rm label}$. 
We next show that given a sanitizer for a class $C$, it is possible to construct a sanitizer for $C^{\rm label}$.
Along the way we will also slightly increase the sample complexity of the starting sanitizer, in order to be able to use Lemma~\ref{lem:sanPac1}.
This results in a reduction from the task of privately learning a concept class $C$ to the sanitization task of the same concept class $C$.

\begin{lemma}\label{lemma:noLabel}
If there exists an $(\alpha,\beta,\epsilon,m)$-sanitizer for a class $C$,
then there exists a $((5\alpha+4\beta),5\beta,6\epsilon,t)$-sanitizer for $C^{\rm label}$, where
$$
\frac{100 m}{\alpha^2}\ln(\frac{1}{\alpha\beta}) \leq t \leq 
\frac{150}{\alpha^2\beta}\ln(\frac{2}{\alpha\beta})\left(m+\frac{1}{\epsilon}\right).
$$
\end{lemma}

\begin{proof}
Let $A'$ be an $(\alpha,\beta,\epsilon,m)$-sanitizer for a class $C$.
By replacing $\alpha$ with $2\alpha$, and $\beta$ with $2\beta$, we can assume that the sanitized databases returned by $A'$ are always of fixed size $n=O(\frac{\VC(C)}{\alpha^2}\log(\frac{1}{\alpha\beta}))$ (see Theorem~\ref{thm:BlumShrink}).
Moreover, we can assume that $A'$ treats its input database as a multiset (as otherwise we could alter $A$ to first randomly shuffle its input database).
Denote $M= m\left\lceil  \frac{18}{\beta}\ln(\frac{2}{\alpha\beta}) \cdot \left( 1 + \frac{1}{m\epsilon} \right)   \right\rceil$.
By Lemma~\ref{lemma:saninitationSample} for every $qM$ (where $q\in\N$) there exists a $((4\alpha+4\beta),2\beta,\epsilon,qM)$ sanitizer $A$ for $C$ (as $qM=q'm$ for an integer $q'$).
Denote $t=\left\lceil \frac{6}{\alpha^2} \right\rceil M$, and consider algorithm $B$ presented in Figure~\ref{fig:algB}

\begin{figure}
\begin{center}
\noindent\fbox{
\parbox{.95\columnwidth}{
{\bf Input:} database $D=(x_i,y_i)_{i=1}^t \in (X_{d+1})^t$.
\begin{enumerate}[rightmargin=10pt]
\item Divide $(x_i)_{i=1}^t$ into $S_0,S_1\in(X_d)^*$, where $S_\sigma$ contains all $x_i$ s.t. $y_i=\sigma$.

\item Sample $\ell_0,\ell_1\leftarrow \left\lfloor \Lap(1/\epsilon) \right\rfloor$; that is, $\ell_0$ and $\ell_1$ are two independent rounded instantiations of a laplacian random variable.

\item Set $m_0= \max\{ 0\;,\;|S_0|+\ell_0\}$ and $m_1=\max\{0\;,\;|S_1|+\ell_1\}$.

\item Set $\hat{m_0}= \left\lfloor \frac{m_0}{M}+\frac{1}{2} \right\rfloor M$, and $\hat{m_1}=\left\lfloor \frac{m_1}{M} +\frac{1}{2} \right\rfloor M$.

\item For $\sigma\in\{0,1\}$, either add copies of the entry $0^d$ to $S_\sigma$, or remove the last entries from $S_\sigma$ until $|S_\sigma|=\hat{m}_{\sigma}$. Denote the resulting multiset as $\hat{S_\sigma}$.

\item Compute $\widetilde{S_0}\leftarrow A(\hat{S_0})$ and $\widetilde{S_1}\leftarrow A(\hat{S_1})$, where if $\hat{S_\sigma}=\emptyset$, then set $\widetilde{S_\sigma}=\emptyset$.

\item Output $\hat{m_0},\hat{m_1},\widetilde{S_0},\widetilde{S_1}$.

\item Construct and output a database $\widetilde{D}\in(X_{d+1})^*$ containing $\hat{m_0}$ copies of $\widetilde{S_0}\circ0$, and $\hat{m_1}$ copies of $\widetilde{S_1}\circ1$. 

\end{enumerate}
}}
\end{center}
\caption{Algorithm $B$ \label{fig:algB}}
\end{figure}

Note that the output on Step~8 is just a post-processing of the 4 outputs on Step~7.
We first show that each of those 4 outputs preserves differential privacy, and, hence, $B$ is private (with slightly bigger privacy parameter, see Theorem~\ref{thm:composition1}).

By the properties of the laplacian mechanism, $\hat{m_0}$ and $\hat{m_1}$ each preserves $\epsilon$-differential privacy. The analysis for $\widetilde{S_0}$ and $\widetilde{S_1}$ is symmetric, and we next give the analysis for $\widetilde{S_0}$. Denote by $B_0$ an algorithm identical to the first 7 steps of $B$, except that the only output of $B_0$ on Step~7 is $\widetilde{S_0}$. We now show that $B_0$ is private. 

\paragraph{Notations.} We use $S_0[D]$ and $\hat{S_0}[D]$ to denote the databases $S_0,\hat{S_0}$ defined on Steps~1 and~5 in the execution of $B_0$ on $D$. Moreover, we use
$m_0[D]$ to denote the value of $|S_0|+\ell_0$ in the execution on $D$, and for every value $m_0[D]=L$, we use $\hat{S_0}[D,L]$ to denote the database $\hat{S_0}$ defined on Step~5, given that $m_0[D]=L$.\\

Fix two neighboring databases $D,D'$, and let $F$ be a set of possible outputs. 
Note that as $D,D'$ are neighboring, it holds that $S_0[D]$ and $S_0[D']$ are identical up to an addition or a change of one entry.
Therefore, whenever $m_0[D]=m_0[D']=L$, we have that $S_0[D,L]$ and $S_0[D',L]$ are neighboring databases.
Moreover, by the properties of the laplacian mechanism, for every value $L$ it holds that $\Pr[m_0[D]=L]\leq e^\epsilon \Pr[m_0[D']=L]$. Hence,

\begin{eqnarray*}
\Pr[B_0(D)\in F] &=& \sum_{L=-\infty}^{\infty}\Pr[m_0[D]=L]\cdot\Pr[B_0(D)\in F | m_0[D]=L]\\
&=&\sum_{L=-\infty}^{\infty}\Pr[m_0[D]=L]\cdot\Pr[A(\hat{S_0}[D,L])\in F]\\
&\leq& \sum_{L=-\infty}^{\infty}e^\epsilon\cdot\Pr[m_0[D']=L]\cdot e^\epsilon\cdot\Pr[A(\hat{S_0}[D',L])\in F]\\
&=& e^{2\epsilon}\cdot\sum_{L=-\infty}^{\infty}\Pr[m_0[D']=L]\cdot\Pr[B_0(D')\in F | m_0[D']=L]\\
&=& e^{2\epsilon}\cdot\Pr[B_0(D')\in F].
\end{eqnarray*}

Overall (since we use two $\epsilon$-private algorithms and two $(2\epsilon)$-private algorithms), algorithm $B$ is $(6\epsilon)$-differentially private.
As for the utility analysis, fix a database $D=(x_i,y_i)_{i=1}^t$ and consider the execution of $B$ on $D$.
We now show that w.h.p. the sanitized database $\widetilde{D}$ is $(5\alpha+4\beta)$-close to $D$.

First note that by the properties of the laplacian mechanism, for $M\geq\frac{2}{\epsilon}\ln(2/\beta)$ we have that with probability at least $(1-\beta)$ it holds that $|\ell_0|,|\ell_1|\leq\frac{M}{2}$. We proceed with the analysis assuming that this is the case. Moreover, note that after the rounding (on Step~4) we have that $|m_\sigma-\hat{m_\sigma}|\leq\frac{M}{2}$. Therefore, for every $\sigma \in \{0,1\}$
$$|S_\sigma|-M\leq|\hat{S_\sigma}|\leq|S_\sigma|+M.$$
Fix a concept $c^{\rm label}\in C^{\rm label}$. It holds that
\begin{eqnarray*}
Q_{c^{\rm label}}(D) &=& \frac{1}{t} |\{ i \; : \; c^{\rm label}(x_i,y_i)=1 \}| \\
&=& \frac{1}{t} \left[\;
\left|\left\{ i \; : \;  \begin{array}{c} y_i=0 \\ c^{\rm label}(x_i,y_i)=1 \end{array} \right\}\right|
+
\left|\left\{ i \; : \;  \begin{array}{c} y_i=1 \\ c^{\rm label}(x_i,y_i)=1 \end{array} \right\}\right|
\;\right]\\
&=& \frac{1}{t} \left[\;
\left|\left\{ i \; : \;  \begin{array}{c} y_i=0 \\ c(x_i)=1 \end{array} \right\}\right|
+
\left|\left\{ i \; : \;  \begin{array}{c} y_i=1 \\ c(x_i)=0 \end{array} \right\}\right|
\;\right]\\
&\leq& \frac{1}{t} \left[\;
\left|\left\{ i \; : \;  \begin{array}{c} x_i\in\hat{S_0} \\ c(x_i)=1 \end{array} \right\}\right| + M
+
\left|\left\{ i \; : \;  \begin{array}{c} x_i\in\hat{S_1} \\ c(x_i)=0 \end{array} \right\}\right| + M
\;\right]\\
&=& \frac{1}{t} \left[\;
\hat{m_0}\cdot Q_c(\hat{S_0}) + \hat{m_1}\left(1 - Q_c(\hat{S_1}) \right) \;\right] + \frac{2M}{t}.
\end{eqnarray*}

By the properties of algorithm $A$, with probability at least $(1-4\beta)$ we have that $\widetilde{S_0}$ and $\widetilde{S_1}$ are $(4\alpha+4\beta)$-close to $\hat{S_0}$ and to $\hat{S_1}$ (respectively). We proceed with the analysis assuming that this is the case. Hence,
\begin{eqnarray*}
Q_{c^{\rm label}}(D)
&\leq& \frac{1}{t} \left[\;
\hat{m_0}\cdot Q_c(\widetilde{S_0}) + (4\alpha+4\beta)\hat{m_0} + \hat{m_1}\left(1 - Q_c(\widetilde{S_1}) \right) + (4\alpha+4\beta)\hat{m_1} \;\right] + \frac{2M}{t}\\
&=& \frac{1}{t} \left[\;
\hat{m_0}\cdot Q_c(\widetilde{S_0}) + \hat{m_1}\left(1 - Q_c(\widetilde{S_1}) \right) \;\right] + (4\alpha+4\beta)\frac{\hat{m_0}+\hat{m_1}}{t}
+\frac{2M}{t}\\
&\leq& \frac{1}{t} \left[\;
\hat{m_0}\cdot Q_c(\widetilde{S_0}) + \hat{m_1}\left(1 - Q_c(\widetilde{S_1}) \right) \;\right] + (4\alpha+4\beta)\frac{t+2M}{t}
+\frac{2M}{t}\\
&\leq& \frac{1}{t} \left[\;
\hat{m_0}\cdot Q_c(\widetilde{S_0}) + \hat{m_1}\left(1 - Q_c(\widetilde{S_1}) \right) \;\right] + (4\alpha+4\beta) + \frac{4M}{t}\\
\end{eqnarray*}

Note that as $c^{\rm label}(x_i\circ0)=c(x_i)$ and as $c^{\rm label}(x_i\circ1)=1-c(x_i)$, it holds that
\begin{eqnarray*}
Q_{c^{\rm label}}(\widetilde{S_0}\circ0)&=&Q_c(\widetilde{S_0});\\
Q_{c^{\rm label}}(\widetilde{S_1}\circ1)&=&1-Q_c(\widetilde{S_1}).
\end{eqnarray*}
Hence,
\begin{eqnarray*}
Q_{c^{\rm label}}(D)
&\leq& \frac{1}{t} \left[\;
\hat{m_0}\cdot Q_{c^{\rm label}}(\widetilde{S_0}\circ0) + \hat{m_1}\cdot Q_{c^{\rm label}}(\widetilde{S_1}\circ1) \;\right] + (4\alpha+4\beta) + \frac{4M}{t}.
\end{eqnarray*}
Denoting $\widetilde{D}=(z_i)_{i=1}^r\in (X_{d+1})^r$ (where $r=n(\hat{m_0}+\hat{m_1})$), we get
\begin{eqnarray*}
&& Q_{c^{\rm label}}(D)\\
&&\;\;\;\;\;\;\leq\;\;\; \frac{1}{nt} \left[\;
\hat{m_0}\cdot\left|\left\{ i \; : \;  \begin{array}{c} z_i\in\widetilde{S_0}\circ0 \\ c^{\rm label}(z_i)=1 \end{array} \right\}\right|
+
\hat{m_1}\cdot\left|\left\{ i \; : \;  \begin{array}{c} z_i\in\widetilde{S_1}\circ1 \\ c^{\rm label}(z_i)=1 \end{array} \right\}\right|
\;\right] + (4\alpha+4\beta) + \frac{4M}{t}\\
&&\;\;\;\;\;\;=\;\;\; \frac{1}{nt} 
\left|\left\{ i \; : \;  c^{\rm label}(z_i)=1  \right\}\right|
 + (4\alpha+4\beta) + \frac{4M}{t}\\
&&\;\;\;\;\;\;=\;\;\; \frac{\hat{m_0}+\hat{m_1}}{t}\cdot Q_{c^{\rm label}}(\widetilde{D})
 + (4\alpha+4\beta) + \frac{4M}{t}\\
&&\;\;\;\;\;\;\leq\;\;\; \frac{t+2M}{t}\cdot Q_{c^{\rm label}}(\widetilde{D})
 + (4\alpha+4\beta) + \frac{4M}{t}\\
&&\;\;\;\;\;\;\leq\;\;\; Q_{c^{\rm label}}(\widetilde{D}) + (4\alpha+4\beta) + \frac{6M}{t}\\
&&\;\;\;\;\;\;\leq\;\;\; Q_{c^{\rm label}}(\widetilde{D}) + (5\alpha+4\beta).
\end{eqnarray*}

Similar arguments show that $Q_{c^{\rm label}}(D)\geq Q_{c^{\rm label}}(\widetilde{D}) - (5\alpha+4\beta)$. Algorithm $B$ is, therefore, a $(5\alpha+4\beta),5\beta,6\epsilon,t)$-sanitizer for $C^{\rm label}$, where
$$
t= \left\lceil \frac{6}{\alpha^2} \right\rceil M = 
\left\lceil \frac{6}{\alpha^2} \right\rceil \cdot 
\left\lceil \frac{18}{\beta}\ln(\frac{2}{\alpha\beta})\left(1+\frac{1}{m\epsilon}\right) \right\rceil \cdot m = O_{\alpha,\beta,\epsilon}(m).
$$

\end{proof}

\begin{theorem}\label{thm:sanPac}
Let $\alpha,\epsilon\leq\frac{1}{8}$, and let $C$ be a class of predicates.
If there exists an $(\alpha,\beta,\epsilon,m)$-sanitizer $A$ for $C$,
then there exists a proper $((15\alpha+12\beta),10\beta,6\epsilon,t)$-PPAC learner for $C$,
where $t=O_{\alpha,\beta,\epsilon}(m)$.
\end{theorem}

\begin{proof}
Let $A$ be an $(\alpha,\beta,\epsilon,m)$-sanitizer for $C$.
Note that by Theorem~\ref{thm:BlumLow}, it must be that $m\geq\frac{\VC(C)}{2}$.
By Lemma~\ref{lemma:noLabel}, there exists a $((5\alpha+4\beta),5\beta,6\epsilon,t)$-sanitizer for $C^{\rm label}$, where $t=O_{\alpha,\beta,\epsilon}(m)$ and $t\geq\frac{100 m}{\alpha^2}\ln(\frac{1}{\alpha\beta})\geq\frac{50 \VC(C)}{\alpha^2}\ln(\frac{1}{\alpha\beta})$.
By Lemma~\ref{lem:sanPac1}, there exists a proper $((15\alpha+12\beta),10\beta,6\epsilon,t)$-PPAC learner for $C$.
\end{proof}

\begin{remark}
Given an efficient proper-sanitizer for $C$ and assuming the existence of an efficient {\em non-private} learner for $C$, this reduction results in an efficient {\em private} learner for $C$.
\end{remark}

\subsection{A Lower Bound for $\kpoint_d$}
Next we prove a lower bound on the database size of every sanitizer for $\kpoint_d$ that preserves pure differential privacy.

Consider the following concept class over $X_d$.
For every $A \subseteq X_d$ s.t. $|A|=k$, the concept class $\kpoint_d$ contains the concept $c_A:X_d \rightarrow\{0,1\}$, defined as
$c_A(x)=1$ if $x\in A$ and $c_A(x)=0$ otherwise.
The VC dimension of $\kpoint_d$ is $k$ (assuming $2^d\geq2k$).

To prove a lower bound on the sample complexity of sanitization, we first prove a lower bound on the sample complexity of the related learning problem and then use the reduction (Theorem~\ref{thm:sanPac}). Thus, we start by showing that every private proper learner for $\kpoint_d$ requires $\Omega(\frac{k d}{\alpha\epsilon})$ labeled examples.
A similar version of this lemma appeared in Beimel et al.~\cite{BBKN12}, where it is shown that every private proper learner for $\point_d$ requires $\Omega(\frac{d}{\alpha\epsilon})$ labeled examples.

\begin{lemma}\label{lem:kPointBig}
Let $\alpha<\frac{1}{5}$, and let $k,d$ be s.t. $2^d\geq k^{1.1}$.
If $L$ is a proper $(\alpha,\frac{1}{2},\epsilon,m)$-PPAC learner for $\kpoint_d$, then $m=\Omega(\frac{kd}{\alpha\epsilon})$.
\end{lemma}

\begin{proof}
Let $L$ be a proper $(\alpha,\frac{1}{2},\epsilon,m)$-PPAC learner for $\kpoint_d$.
Without loss of generality, we can assume that $m\geq\frac{5\ln(4)}{3\alpha}$ (since $L$ can ignore part of the sample).

Consider a maximal cardinality subset $B\subseteq\kpoint_d$ s.t. for every $c_A\in B$ it holds that $0^d\notin A$, and moreover, for every $c_{A_1}\neq c_{A_2}\in B$ it holds that $|A_1 \cap A_2|\leq\frac{k}{2}$.
We have that $|B|\geq\left(\frac{2^d-1}{4 e^2 k}\right)^{k/2}$. To see this, we could construct such a set using the following greedy algorithm.
Initiate $\hat{B}=\emptyset$, and $C=\kpoint_d \setminus \{ c_I\in\kpoint_d \; : \; 0^d\in I  \} $. While $C\neq\emptyset$, arbitrarily choose a concept $c_A \in C$, add $c_A$ to $\hat{B}$, and remove from $C$ every concept $c_I$ s.t. $|A \cap I|\leq\frac{k}{2}$.

Clearly, for every two $c_{A_1}\neq c_{A_2}\in\hat{B}$ it holds that $|A_1 \cap A_2|\leq\frac{k}{2}$. Moreover, at every step, the number of concepts that are removed from $C$ is at most
$$
\sum_{j=k/2}^k{{k\choose j}\cdot{2^d-1-k \choose k-j}}\leq{k\choose k/2}\cdot{2^d-1 \choose k/2},
$$
and, therefore,
$$
\hat{B}\geq\frac{{2^d-1 \choose k}}{{k\choose k/2}\cdot{2^d-1 \choose k/2}} \geq 
\left(\frac{2^d-1}{4 e^2 k}\right)^{k/2}.
$$

For every $c_A\in B$ we will now define a distribution $\DDD_A$, a set of hypotheses $G(A)$, and a database $S_A$.
The distribution $\DDD_A$ is defined as
$$
\DDD_A(x)=
\begin{cases}
    1-5\alpha, & x=0^d.\\
    \frac{5\alpha}{k}, & x \in A.\\
    0, & {\rm else}.
  \end{cases}
$$

Define the set $G(A)\subseteq\kpoint_d$ as all $\alpha$-good hypothesis for $(c_A,\DDD_A)$ in $\kpoint_d$.
\remove{As $L$ is a proper PAC learner for $\kpoint_d$, it must return a hypothesis from $G(A)$ with probability at least $1/2$.}
Note that for every $h_I\in\kpoint_d$ s.t. $\error_{\DDD_A}(h_I,c_A)\leq\alpha$ it holds that $|I\cap A|\geq\frac{4k}{5}$. Therefore,
for every $c_{A_1}\neq c_{A_2}\in B$ it holds that $G(A_1)\cap G(A_2)=\emptyset$ (as $|A_1 \cap A_2|\leq\frac{k}{2}$, and as $|A_1|=|A_2|=|I|=k$).

By the utility properties of $L$, we have that $\Pr_{L,\DDD_A}[L(S)\in G(A)]\geq\frac{1}{2}$.
We say that a database $S$ of $m$ labeled examples is {\em good} if the unlabeled example $0^d$ appears in $\db$ at least $(1-8\alpha)m$ times.
Let $S$ be a database constructed by taking $m$ i.i.d. samples from $\DDD_A$, labeled by $c_A$.
By the Chernoff bound, $S$ is good with probability at least $1-\exp(-3\alpha m /5)$. Hence,

$$\Pr_{\DDD_A , L}\left[   ( L(S) \in G(A) ) \wedge ( S {\rm \;is\;good } )    \right]\geq \frac{1}{2}-\exp(-3\alpha m /5)\geq\frac{1}{4}.$$

Note that, as $0^d\notin A$, every appearance of the example $0^d$ in $S$ is labeled by $0$.
Therefore, there exists a good database $S$ of $m$ samples that contains the entry $0^d \circ 0$ at least $(1-8\alpha)m$ times, and
$\Pr_L\left[ L(S) \in G(A)  \right]\geq\frac{1}{4}$,
where the probability is only over the randomness of $L$. We define $S_A$ as such a database.

\medskip
Note that all of the databases $S_{A_i}$ defined here are of distance at most $8\alpha m$ from one another. The privacy of $L$ ensures, therefore, that for any two such $S_{A_i},S_{A_j}$ it holds that $\Pr_L[L(S_{A_i})\in G(A_j)]\geq\frac{1}{4}\exp(-8\alpha\epsilon m)$.

Now,
\begin{eqnarray}\label{eqn:bigKpoints}
1-\frac{1}{4} &\geq& \Pr_L[L(S_{A_i})\notin G(A_i)] \nonumber \\
&\geq& \Pr_L[L(S_{A_i})\in \bigcup_{A_j \neq A_i} G(A_j)] \nonumber \\
&\geq& \sum_{A_j \neq A_i} {\Pr_L[L(S_{A_i})\in G(A_j)]} \nonumber \\
&\geq& (|B|-1) \frac{1}{4}\exp(-8\alpha\epsilon m) \nonumber \\
&\geq& \left(\left(\frac{2^d-1}{4 e^2 k}\right)^{k/2}-1 \right) \frac{1}{4}\exp(-8\alpha\epsilon m).
\end{eqnarray}
Solving for $m$ yields $m=\Omega(\frac{k}{\alpha\epsilon}(d-\ln(k)))$.
Recall that $2^d\geq k^{1.1}$, and, hence, $m=\Omega(\frac{kd}{\alpha\epsilon})$
\end{proof}

\begin{remark}
The constant $1.1$ in the above lemma could be replaced with any constant strictly bigger than $1$. Moreover, whenever $2^d=O(k)$ we have that $|\kpoint_d|={2^d \choose k}=2^{O(2^d)}$ and, hence, the generic construction of Kasiviswanathan et al.~\cite{KLNRS08} yields a proper $\epsilon$-private learner for this class with sample complexity $O_{\alpha,\beta,\epsilon}(2^d)=O_{\alpha,\beta,\epsilon}(k)$.
\end{remark}

In the next lemma we will use the last lower bound on the sample complexity of private learners, together with the reduction of Theorem~\ref{thm:sanPac}, and derive a lower bound on the database size necessary for pure private sanitizers for $\kpoint_d$.

\begin{theorem}\label{thm:SanPacReduction}
Let $\epsilon\leq\frac{1}{8}$, and let $k$ and $d$ be s.t. $2^d\geq k^{1.1}$. Every $(\frac{1}{150},\frac{1}{150},\epsilon,m)$-sanitizer for $\kpoint_d$ requires databases of size
$$m=\Omega\left(\frac{1}{\epsilon}\VC(\kpoint_d)\cdot \log|X_d|\right).$$
\end{theorem}

\begin{proof}
Let $A$ be a $(\frac{1}{150},\frac{1}{150},\epsilon,m)$-sanitizer for $\kpoint_d$.
By Theorem~\ref{thm:sanPac}, there exists a proper $(\frac{9}{50},\frac{1}{15},6\epsilon,t)$-PPAC learner for $\kpoint_d$, where $t=O\left( m \right)$.
By Lemma~\ref{lem:kPointBig}, $t=\Omega\left( \frac{kd}{\epsilon} \right)$, and hence
$m=\Omega\left( \frac{kd}{\epsilon} \right)$.
\end{proof}

Recall that in the proof of Theorem~\ref{thm:sanPac}, we increased the sample complexity in order to use Lemma~\ref{lem:sanPac1}. This causes a slackness of $\alpha^2$ in the database size of the resulting learner, which, in turn, 
eliminates the dependency in $\alpha$ in the above lower bound.
For the class $\kpoint_d^{\rm label}$ it is possible to obtain a better lower bound, by using the reduction of Lemma~\ref{lem:sanPac1} twice.

\begin{theorem}\label{thm:kpointSanLower}
Let $\alpha\leq\frac{1}{50}$ and $\epsilon\leq\frac{1}{8}$.
There exist a $d_0=d_0(\alpha,\epsilon)$ s.t. for every $k$ and $d$ s.t. $2^d\geq\max\{ k^{1.1} \;,\; 2^{d_0}\}$, it holds that
every $(\alpha,\frac{1}{50},\epsilon,m)$-sanitizer for $\kpoint_d^{\rm label}$ must operate on databases of size
$$m=\Omega\left(\frac{1}{\alpha\epsilon}\VC(\kpoint_d^{\rm label})\cdot \log|X_d|\right).$$
\end{theorem}

\begin{proof}

Let $A$ be a $(\frac{1}{50},\frac{1}{50},\epsilon,m)$-sanitizer for a class $\kpoint_d^{\rm label}$, where $\epsilon\leq\frac{1}{8}$.
Note that by Theorem~\ref{thm:BlumLow}, it must be that $m\geq\frac{\VC(\kpoint_d^{\rm label})}{2}\geq\frac{\VC(\kpoint_d)}{2}$.
In order to use Lemma~\ref{lem:sanPac1}, we need a slightly stronger guarantee, and therefore use Lemma~\ref{lemma:saninitationSample} to increase the input database size as follows.

Denote $q = \left\lceil 100\cdot 50^3\ln(50^2) \right\rceil$.
By Lemma~\ref{lemma:saninitationSample}, there exists a $(\frac{2}{25},\frac{1}{50},\epsilon,t)$-sanitizer $B$ for $\kpoint_d^{\rm label}$, where
$$
t=qm=m\left\lceil 100\cdot 50^3\ln(50^2) \right\rceil
\geq 50\cdot50^3\VC(\kpoint_d)\ln(50^2).
$$

By Lemma~\ref{lem:sanPac1}, there exists a proper $(\frac{9}{50},\frac{1}{25},\epsilon,t)$-PPAC learner for $\kpoint_d$.
By Lemma~\ref{lem:kPointBig}, $t=\Omega\left( \frac{kd}{\epsilon} \right)$, and hence
\begin{eqnarray}
m=\Omega\left( \frac{kd}{\epsilon} \right).
\label{eq:kpointSanFirstUse}
\end{eqnarray}\\

Let $\alpha\leq\frac{1}{50}$ and $\epsilon\leq\frac{1}{8}$, and let $B$ be an $(\alpha,\frac{1}{50},\epsilon,m)$-sanitizer for $\kpoint_d^{\rm label}$.
As $B$ is, in particular, a $(\frac{1}{50},\frac{1}{50},\epsilon,m)$-sanitizer for $\kpoint_d^{\rm label}$, where $\epsilon\leq\frac{1}{8}$, Equation~\ref{eq:kpointSanFirstUse} states that there exists a constant $\lambda$ s.t.
$m\geq\lambda\frac{ k d }{\epsilon}$.
Asserting that $d\geq d_0 \triangleq \frac{50\epsilon}{\lambda \alpha^2}\ln(\frac{50}{\alpha})$, we ensure that $m\geq\frac{50k}{\alpha^2}\ln(\frac{50}{\alpha})$. By reusing Lemme~\ref{lem:sanPac1}, we now get that there exists a proper $(3\alpha,\frac{1}{25},\epsilon,m)$-PPAC learner for $\kpoint_d$.
Lemma~\ref{lem:kPointBig} now states that
$$m=\Omega\left(\frac{k d}{\alpha\epsilon}\right)=\Omega\left(\frac{1}{\alpha\epsilon}\VC(\kpoint_d^{\rm label})\cdot \log|X_d|\right).$$
\end{proof}

\section{Label-Private Learners}
\label{sec:semiPrivate}
\subsection{Generic Label-Private Learner}
In this section we consider relaxed definitions of private learners preserving pure privacy (i.e., $\delta=0$).
We start with the model of label privacy (see~\cite{CH11} and references therein). In this model, privacy must only be preserved for the {\em labels} of the elements in the database, and not necessarily for their identity.
This is a reasonable privacy requirement when the identity of individuals in a population are known publicly but not their labels. In general, this is not a reasonable assumption.

We consider a database $S=(x_i,y_i)_{i=1}^m$ containing labeled points from some domain $X$, and denote $S_x=(x_i)_{i=1}^m\in X^m$, and $S_y=(y_i)_{i=1}^m\in\{0,1\}^m$.

\begin{definition}[Label-Private Learner]\label{def:labelPrivate}
Let $A$ be an algorithm that gets as input a database $S_x\in X^m$ and its labels $S_y\in \{0,1\}^m$. Algorithm $A$ is an {\em $(\alpha,\beta,\epsilon,m)$-Label Private PAC Learner} for a concept
class $C$ over $X$ if
\begin{description}
\item{\sc Privacy.} $\forall S_x\in X^m$, algorithm $A(S_x,\cdot)=A_{S_x}(\cdot)$ is $\epsilon$-differentially private (as in \defref{dp});
\item{\sc Utility.} Algorithm $A$ is an {\em $(\alpha,\beta,m)$-PAC learner} for $C$ (as in \defref{PAC}).
\end{description}
\end{definition}

Chaudhuri et al.~\cite{CH11} proved lower bounds on the sample complexity of label-private learners for a class $C$ in terms of its doubling dimension. As we will now see, the correct measure for characterizing the sample complexity of such learners is the VC dimension, and the sample complexity of label-private learners is actually of the same order as that of non-private learners (assuming $\alpha,\beta$, and $\epsilon$ are constants).

\begin{theorem}
Let $C$ be a concept class over a domain $X$.
For every $\alpha,\beta,\epsilon$, there exists an $(\alpha,\beta,\epsilon,m)$-Label Private PAC learner for $C$, where $m=O_{\alpha,\beta,\epsilon}(\VC(C))$. The learner might not be efficient.
\end{theorem}

\begin{proof}

For a concept class $C$ over a domain $X$, and for a subset $B=\{b_1,\ldots,b_\ell\}\subseteq X$, the projection of $C$ on $B$ is denoted as $$\Pi_C(B)=\{(c(b_1),\ldots,c(b_\ell)):c\in C\}.$$

\begin{figure}
\begin{center}
\noindent\fbox{
\parbox{.95\columnwidth}{
{\bf Generic Label-Private Learner}\\
{\bf Input:} parameter $\epsilon$, and a labeled database $S=(x_i,y_i)_{i=1}^m$, where $m\geq\frac{768}{\alpha^2 \epsilon}(\VC(C)\ln(\frac{64}{\alpha})+2\ln(\frac{8}{\beta}))$.
\begin{enumerate}[rightmargin=10pt]
\item Initialize $H=\emptyset$.
\item Set $n=\frac{32}{\alpha}(\VC(C)\ln(\frac{64}{\alpha})+\ln(\frac{8}{\beta}))$, and denote $S^1=(x_i,y_i)_{i=1}^n$, and $S^2=(x_i,y_i)_{i=n+1}^m$.
\item Let $B=\{b_1,\ldots,b_\ell\}$ be the set of all the (unlabeled) points appearing at least once in $S^1$.
\item For every $(z_1,\ldots,z_\ell)\in \Pi_C(B)$,
\begin{enumerate}
\item Let $c\in C$ be an arbitrary concept in $C$ s.t. $c(b_i)=z_i$ for every $1\leq i\leq\ell$.
\item Add $c$ to $H$.
\end{enumerate}
\item Choose and return $h\in H$ using the exponential mechanism with inputs $\epsilon,H,S^2$ and the quality function $q(S^2,h)=|\{i:h(x_i)=y_i\}|$.
\end{enumerate}
}}
\end{center}
\caption{A generic label-private learner. \label{fig:genericLabelPrivate}}
\end{figure}

In Figure~\ref{fig:genericLabelPrivate} we describe a label-private algorithm $A$.
Algorithm $A$ constructs a set of hypotheses $H$ as follows:
It samples an unlabeled sample $S_1$, and defines $B$ as the set of points in $S_1$.
For every labeling of the points in $B$ realized by $C$, add to $H$ an arbitrary concept consistent with this labeling. Afterwards, algorithm $A$ uses the exponential mechanism to choose a hypothesis out of $H$.\\

Note that steps~1-4 of algorithm $A$ are independent of the labeling vector $S_y$.
By the properties of the exponential mechanism (which is used to access $S_y$ on Step~5), for every set of elements $S_x$, algorithm $A(S_x,\cdot)$ is $\epsilon$-differentially private.

For the utility analysis, fix a target concept $c\in C$ and a distribution $\DDD$ over $X$, and define the following 3 good events:
\begin{enumerate}[label=$E_{\arabic*}$]
\item The constructed set $H$ contains at least one hypothesis $f$ s.t. $\error_{S^2}(f)\leq\frac{\alpha}{4}$.
\item For every $h\in H$ s.t. $\error_{S^2}(h)\leq\frac{\alpha}{2}$, it holds that $\error_{\DDD}(c,h)\leq\alpha$.
\item The exponential mechanism chooses an $h$ such that $\error_{S^2}(h) \leq \frac{\alpha}{4} + \min_{f\in H}\left\{\error_{S^2}(f)\right\}$.
\end{enumerate}

We first show that if these 3 good events happen, then algorithm $A$ returns an $\alpha$-good hypothesis.
Event $E_1$ ensures the existence of a hypothesis $f\in H$ s.t. $\error_{S^2}(f)\leq\frac{\alpha}{4}$. Thus, event $E_1 \cap E_3$ ensures algorithm $A$ chooses (using the exponential mechanism) a hypothesis $h\in H$ s.t. $\error_{S^2}(h)\leq\frac{\alpha}{2}$. Event $E_2$ ensures, therefore, that this $h$ obeys $\error_{\DDD}(c,h)\leq\alpha$.

We will now show that those 3 events happen with high probability.
For every $(y_1,\ldots,y_\ell)\in \Pi_C(B)$, algorithm $A$ adds to $H$ a hypothesis $f$ s.t. $\forall 1\leq i \leq \ell,\;f(b_i)=y_i$.
In particular, $H$ contains a hypothesis $h^*$ s.t. $h^*(x)=c(x)$ for every $x\in B$. That is, a hypothesis $h^*$ s.t. $\error_{S^1}(h^*)=0$.
Therefore, by setting $n\geq\frac{32}{\alpha}(\VC(C)\ln(\frac{64}{\alpha})+\ln(\frac{8}{\beta}))$, Theorem~\ref{thm:VCconsistant} ensures that $\error_{\DDD}(c,h^*)\leq\frac{\alpha}{8}$ with probability at least $(1-\frac{\beta}{4})$. In such a case, using the Chernoff bound, we get that with probability at least $1-\exp(-(m-n)\alpha/24)$ this hypothesis $h^*$ satisfies $\error_{S^2}(h^*)\leq\frac{\alpha}{4}$. Event $E_1$ happens, therefore, with probability at least $(1-\frac{\beta}{4})(1-\exp(-(m-n)\alpha/24))$, which is at least $(1-\frac{\beta}{2})$ for
$m\geq n+\frac{24}{\alpha}\ln(4/\beta)$.

Fix a hypothesis $h$ s.t. $\error_{\DDD}(c,h)>\alpha$.
Using the Chernoff bound, the probability that $\error_{S^2}(h)\leq\frac{\alpha}{2}$ is less than $\exp(-(m-n)\alpha/8)$.
As $|H|=2^{|B|}\leq2^n$, the probability that there is such a hypothesis in $H$ is at most $2^n\cdot\exp(-(m-n)\alpha/8)$.
For $m\geq\frac{8}{\alpha}(n+\ln(\frac{4}{\beta}))$, this probability is at most $\frac{\beta}{4}$, and event $E_2$ happens with probability at least $(1-\frac{\beta}{4})$.

The exponential mechanism ensures that the probability of event $E_3$ is at least $1-|H| \cdot \exp(-\epsilon \alpha m /8)$ (see Proposition \ref{prop:expMech}), which is at least $(1-\frac{\beta}{4})$ for $m \geq \frac{8}{\alpha \epsilon}( n+\ln(\frac{4}{\beta}))$.

All in all, by setting
$n=\frac{32}{\alpha}(\VC(C)\ln(\frac{64}{\alpha})+\ln(\frac{8}{\beta}))$
and
$m\geq\frac{768}{\alpha^2 \epsilon}(\VC(C)\ln(\frac{64}{\alpha})+2\ln(\frac{8}{\beta}))$,
we ensure that the probability of $A$ failing to output an $\alpha$-good hypothesis is at most $\beta$.
\end{proof}

\subsection{Label Privacy Extension}
We consider a slight generalization of the label privacy model.
Recall that given a labeled sample, a private learner is required to preserve the privacy of the entire sample, while a label-private learner is only required to preserve privacy for the labels of each entry.

Consider a scenario where there is no need in preserving the privacy of the distribution $\DDD$ (for example, $\DDD$ might be publicly known), but we still want to preserve the privacy of the entire sample $S$. We can model this scenario as a learning algorithm $A$ which is given as input 2 databases -- a labeled database $S$, and an unlabeled database $D$. For every database $D$, algorithm $A(D,\cdot)=A_D(\cdot)$ must preserve differential privacy. We will refer to such a learner as a {\em Semi-Private} learner.

Clearly, $\Omega(\VC(C))$ samples are necessary in order to semi-privately learn a concept class $C$, as this is the case for non-private learners.\footnote{The lower bound of $\Omega(\VC(C))$ is worst case over choices of distributions $\DDD$. For a specific distribution, less samples may suffice.}
This lower bound is tight, as the above generic learner could easily be adjusted for the semi-privacy model, and result in a generic semi-private learner with sample complexity $O_{\alpha,\beta,\epsilon}(\VC(C))$. To see this, recall that in the above algorithm, the input sample $S$ is divided into $S_1$ and $S_2$. Note that the labels in $S_1$ are ignored, and, hence, $S_1$ could be replaced with an unlabeled database. Moreover, note that $S_2$ is only accessed using the exponential mechanism  (on Step~5), which preserves the privacy both for the labels and for the examples in $S_2$.

\begin{example} Consider the task of learning a concept class $C$, and suppose that the relevant distribution over the population is publicly known.
Now, given a labeled database $S$, we can use a semi-private learner and guarantee privacy both for the labellings and for the mere existence of an individual in the database. That is, in such a case, the privacy guarantee of a semi-private learner is the same as that of a private learner.
Moreover, the necessary sample complexity is $O_{\alpha,\beta,\epsilon}(\VC(C))$, which should be contrasted with $O_{\alpha,\beta,\epsilon}(\log|C|)$ which is the sample complexity that would result from the general construction of Kasiviswanathan et al.~\cite{KLNRS08}.
\end{example}

\paragraph{{\bf Acknowledgments.}} We thank Salil Vadhan and Jon Ullman for helpful discussions of ideas in this work.

\bibliographystyle{plain}

\begin{thebibliography}{10}

\newcommand{\bibhead}[1]{}

\bibitem{Anthony93}\bibhead{Anthony93}
{\sc Martin Anthony and John Shawe-Taylor}: A result of {V}apnik with
  applications.
\newblock {\em Discrete Applied Mathematics}, 47(3):207--217, 1993.

\bibitem{Anthony2009}\bibhead{Anthony2009}
{\sc Matin Anthony and Peter~L. Bartlett}: {\em Neural Network Learning:
  Theoretical Foundations}.
\newblock Cambridge University Press, 2009.

\bibitem{BBKN12}\bibhead{BBKN12}
{\sc Amos Beimel, Hai Brenner, Shiva~Prasad Kasiviswanathan, and Kobbi Nissim}:
  Bounds on the sample complexity for private learning and private data
  release.
\newblock {\em Machine Learning}, 94(3):401--437, 2014.

\bibitem{BNS13}\bibhead{BNS13}
{\sc Amos Beimel, Kobbi Nissim, and Uri Stemmer}: Characterizing the sample
  complexity of private learners.
\newblock In {\sc Robert~D. Kleinberg}, editor, {\em ITCS}, pp. 97--110. ACM,
  2013.

\bibitem{BNS13b}\bibhead{BNS13b}
{\sc Amos Beimel, Kobbi Nissim, and Uri Stemmer}: Private learning and
  sanitization: Pure vs. approximate differential privacy.
\newblock In {\sc Prasad Raghavendra, Sofya Raskhodnikova, Klaus Jansen, and
  Jos{\'e} D.~P. Rolim}, editors, {\em APPROX-RANDOM}, volume 8096 of {\em
  Lecture Notes in Computer Science}, pp. 363--378. Springer, 2013.

\bibitem{BDMN05}\bibhead{BDMN05}
{\sc Avrim Blum, Cynthia Dwork, Frank McSherry, and Kobbi Nissim}: Practical
  privacy: The {SuLQ} framework.
\newblock In {\sc Chen Li}, editor, {\em PODS}, pp. 128--138. ACM, 2005.

\bibitem{BLR08full}\bibhead{BLR08full}
{\sc Avrim Blum, Katrina Ligett, and Aaron Roth}: A learning theory approach to
  noninteractive database privacy.
\newblock {\em J. ACM}, 60(2):12, 2013.

\bibitem{BEHW}\bibhead{BEHW}
{\sc Anselm Blumer, Andrzej Ehrenfeucht, David Haussler, and Manfred~K.
  Warmuth}: Learnability and the vapnik-chervonenkis dimension.
\newblock {\em J. ACM}, 36(4):929--965, 1989.

\bibitem{CH11}\bibhead{CH11}
{\sc Kamalika Chaudhuri and Daniel Hsu}: Sample complexity bounds for
  differentially private learning.
\newblock In {\sc Sham~M. Kakade and Ulrike von Luxburg}, editors, {\em COLT},
  volume~19 of {\em JMLR Proceedings}, pp. 155--186. JMLR.org, 2011.

\bibitem{De12}\bibhead{De12}
{\sc Anindya De}: Lower bounds in differential privacy.
\newblock In {\sc Ronald Cramer}, editor, {\em TCC}, volume 7194 of {\em
  Lecture Notes in Computer Science}, pp. 321--338. Springer, 2012.

\bibitem{DKMMN06}\bibhead{DKMMN06}
{\sc Cynthia Dwork, Krishnaram Kenthapadi, Frank McSherry, Ilya Mironov, and
  Moni Naor}: Our data, ourselves: Privacy via distributed noise generation.
\newblock In {\sc Serge Vaudenay}, editor, {\em EUROCRYPT}, volume 4004 of {\em
  Lecture Notes in Computer Science}, pp. 486--503. Springer, 2006.

\bibitem{DworkLei}\bibhead{DworkLei}
{\sc Cynthia Dwork and Jing Lei}: Differential privacy and robust statistics.
\newblock In {\sc Michael Mitzenmacher}, editor, {\em STOC}, pp. 371--380. ACM,
  2009.

\bibitem{DMNS06}\bibhead{DMNS06}
{\sc Cynthia Dwork, Frank McSherry, Kobbi Nissim, and Adam Smith}: Calibrating
  noise to sensitivity in private data analysis.
\newblock In {\sc Shai Halevi and Tal Rabin}, editors, {\em TCC}, volume 3876
  of {\em Lecture Notes in Computer Science}, pp. 265--284. Springer, 2006.

\bibitem{DNRRV09}\bibhead{DNRRV09}
{\sc Cynthia Dwork, Moni Naor, Omer Reingold, Guy~N. Rothblum, and Salil~P.
  Vadhan}: On the complexity of differentially private data release: efficient
  algorithms and hardness results.
\newblock In {\sc Michael Mitzenmacher}, editor, {\em STOC}, pp. 381--390. ACM,
  2009.

\bibitem{DRV10}\bibhead{DRV10}
{\sc Cynthia Dwork, Guy~N. Rothblum, and Salil~P. Vadhan}: Boosting and
  differential privacy.
\newblock In {\em FOCS}, pp. 51--60. IEEE Computer Society, 2010.

\bibitem{EHKV}\bibhead{EHKV}
{\sc Andrzej Ehrenfeucht, David Haussler, Michael~J. Kearns, and Leslie~G.
  Valiant}: A general lower bound on the number of examples needed for
  learning.
\newblock {\em Inf. Comput.}, 82(3):247--261, 1989.

\bibitem{FX14}\bibhead{FX14}
{\sc Vitaly Feldman and David Xiao}: Sample complexity bounds on differentially
  private learning via communication complexity.
\newblock {\em CoRR}, abs/1402.6278, 2014.

\bibitem{GHRU11}\bibhead{GHRU11}
{\sc Anupam Gupta, Moritz Hardt, Aaron Roth, and Jonathan Ullman}: Privately
  releasing conjunctions and the statistical query barrier.
\newblock {\em SIAM J. Comput.}, 42(4):1494--1520, 2013.

\bibitem{PMW_HR10}\bibhead{PMW_HR10}
{\sc Moritz Hardt and Guy~N. Rothblum}: A multiplicative weights mechanism for
  privacy-preserving data analysis.
\newblock In {\em FOCS}, pp. 61--70. IEEE Computer Society, 2010.

\bibitem{HT10}\bibhead{HT10}
{\sc Moritz Hardt and Kunal Talwar}: On the geometry of differential privacy.
\newblock In {\sc Leonard~J. Schulman}, editor, {\em STOC}, pp. 705--714. ACM,
  2010.

\bibitem{Moritz}\bibhead{Moritz}
{\sc Moritz~A.W. Hardt}: {\em A Study of Privacy and Fairness in Sensitive Data
  Analysis}.
\newblock Ph.\,D.\ thesis, Princeton University, 2011.

\bibitem{KLNRS08}\bibhead{KLNRS08}
{\sc Shiva~Prasad Kasiviswanathan, Homin~K. Lee, Kobbi Nissim, Sofya
  Raskhodnikova, and Adam Smith}: What can we learn privately?
\newblock {\em SIAM J. Comput.}, 40(3):793--826, 2011.

\bibitem{Kearns98}\bibhead{Kearns98}
{\sc Michael~J. Kearns}: Efficient noise-tolerant learning from statistical
  queries.
\newblock {\em J. ACM}, 45(6):983--1006, 1998.

\bibitem{MT07}\bibhead{MT07}
{\sc Frank McSherry and Kunal Talwar}: Mechanism design via differential
  privacy.
\newblock In {\em FOCS}, pp. 94--103. IEEE Computer Society, 2007.

\bibitem{Roth10}\bibhead{Roth10}
{\sc Aaron Roth}: Differential privacy and the fat-shattering dimension of
  linear queries.
\newblock In {\sc Maria~J. Serna, Ronen Shaltiel, Klaus Jansen, and Jos{\'e}
  D.~P. Rolim}, editors, {\em APPROX-RANDOM}, volume 6302 of {\em Lecture Notes
  in Computer Science}, pp. 683--695. Springer, 2010.

\bibitem{ADist}\bibhead{ADist}
{\sc Abhradeep Thakurta and Adam Smith}: Differentially private feature
  selection via stability arguments, and the robustness of the lasso.
\newblock In {\sc Shai Shalev-Shwartz and Ingo Steinwart}, editors, {\em COLT},
  volume~30 of {\em JMLR Proceedings}, pp. 819--850. JMLR.org, 2013.

\bibitem{Ullman12}\bibhead{Ullman12}
{\sc Jonathan Ullman}: Answering $n^{2+o(1)}$ counting queries with
  differential privacy is hard.
\newblock In {\sc Dan Boneh, Tim Roughgarden, and Joan Feigenbaum}, editors,
  {\em STOC}, pp. 361--370. ACM, 2013.

\bibitem{UV11}\bibhead{UV11}
{\sc Jonathan Ullman and Salil~P. Vadhan}: {PCP}s and the hardness of
  generating private synthetic data.
\newblock In {\sc Yuval Ishai}, editor, {\em TCC}, volume 6597 of {\em Lecture
  Notes in Computer Science}, pp. 400--416. Springer, 2011.

\bibitem{Valiant84}\bibhead{Valiant84}
{\sc Leslie~G. Valiant}: A theory of the learnable.
\newblock {\em Commun. ACM}, 27(11):1134--1142, 1984.

\bibitem{VC}\bibhead{VC}
{\sc Vladimir~N. Vapnik and Alexey~Y. Chervonenkis}: On the uniform convergence
  of relative frequencies of events to their probabilities.
\newblock {\em Theory of Probability and its Applications}, 16(2):264--280,
  1971.

\end{thebibliography}

\end{document}